%% file: main.tex
\documentclass[11pt]{article}
\usepackage[margin=1in]{geometry}
\usepackage{verbatim}
\usepackage[usenames,dvipsnames]{xcolor}
\definecolor{Gred}{RGB}{219, 50, 54}
\definecolor{Ggreen}{RGB}{60, 186, 84}
\definecolor{Gblue}{RGB}{72, 133, 237}
\definecolor{Gyellow}{RGB}{247, 178, 16}
\definecolor{ToCgreen}{RGB}{0, 128, 0}
\definecolor{myGold}{RGB}{231,141,20}
\definecolor{myBlue}{rgb}{0.19,0.41,.65}
\definecolor{myPurple}{RGB}{175,0,124}

\usepackage{subfig}
\usepackage{hyperref}
\hypersetup{
			colorlinks=true,
	citecolor=ToCgreen,
	linkcolor=Sepia,
	filecolor=Gred,
	urlcolor=Gred
	}

\usepackage{preamble}
\title{Online and Distribution-Free Robustness: \\ Regression and Contextual Bandits with Huber Contamination}
\author{Sitan Chen\thanks{Email: \texttt{sitanc@mit.edu} This work was supported in part by NSF CAREER Award CCF-1453261, NSF Large CCF-1565235 and Ankur Moitra's ONR Young Investigator Award.} \\
MIT
 \and 
Frederic Koehler\thanks{Email: \texttt{fkoehler@mit.edu}. This work was supported in part by NSF CAREER Award CCF-1453261, NSF Large CCF-1565235, Ankur Moitra's ONR Young Investigator Award, and E. Mossel's Vannevar Bush Fellowship ONR-N00014-20-1-2826.} \\
MIT
 \\\and
Ankur Moitra\thanks{Email: \texttt{moitra@mit.edu} This work was
supported in part by a Microsoft Trustworthy AI Grant, NSF CAREER Award CCF-1453261, NSF Large CCF1565235, a David and Lucile Packard Fellowship and an ONR Young Investigator
Award.}\\
MIT
\and 
Morris Yau\thanks{Email: \texttt{morrisyau@berkeley.edu}}\\
UC Berkeley}

\newcommand{\argmin}{\mathop{\text{argmin}}}

\newcommand{\RegHCB}{\mathop{\textup{Reg}_{\mathsf{HCB}}}}
\newcommand{\RegHSQ}{\mathop{\textup{Reg}_{\mathsf{HSq}}}}
\newcommand{\psRegHCB}{\widetilde{\mathop{\textup{Reg}_{\mathsf{HCB}}}}}
\newcommand{\N}{\mathcal{N}}

\newcommand{\td}{\widetilde}
\newcommand{\Ber}{\mathop{\textup{Ber}}}
\newcommand{\bareta}{\overline{\eta}}
\newcommand{\calDx}{\calD_x}
\newcommand{\frakF}{\mathfrak{F}}
\newcommand{\graderr}{\epsilon_{grad}}

\newcommand{\todo}[1]{{[\color{red} TODO: #1]}}

\newtheorem{innercustomassume}{Assumption}
\newenvironment{customassume}[1]
  {\renewcommand\theinnercustomassume{#1}\innercustomassume}
  {\endinnercustomassume}

\begin{document}

\maketitle

\begin{abstract}
    \normalsize
    In this work we revisit two classic high-dimensional online learning problems, namely linear regression and contextual bandits, from the perspective of adversarial robustness. Existing works in algorithmic robust statistics make strong distributional assumptions that ensure that the input data is evenly spread out or comes from a nice generative model. {\em Is it possible to achieve strong robustness guarantees even without distributional assumptions altogether, where the sequence of tasks we are asked to solve is adaptively and adversarially chosen?} 

    We answer this question in the affirmative for both linear regression and contextual bandits. In fact our algorithms succeed where conventional methods fail. In particular we show strong lower bounds against Huber regression and more generally any convex $M$-estimator. Our approach is based on a novel alternating minimization scheme that interleaves ordinary least-squares with a 
    simple convex program that finds the optimal reweighting of the distribution
    under a spectral constraint.
    %semidefinite program for finding appropriate reweightings of the distribution.
    Our results obtain essentially optimal dependence on the contamination level $\eta$, reach the optimal breakdown point, and naturally apply to infinite dimensional settings where the feature vectors are represented implicitly via a kernel map. 
    
    % We answer this question in the affirmative for both regression and linear contextual bandits. In fact our algorithms succeed where convex surrogates fail in the sense that we show strong lower bounds categorically for the existing approaches. Our approach is based on a novel way to use the sum-of-squares hierarchy in online learning and in the absence of distributional assumptions. Moreover we give extensions of our main results to infinite dimensional settings where the feature vectors are represented implicitly via a kernel map. 
\end{abstract}

\thispagestyle{empty}

\newpage

\tableofcontents

\newpage

\setcounter{page}{1}

\input{intro}

\subsection{Roadmap}

In Section~\ref{sec:overview}, we give an overview of the main techniques in our approach. In Section~\ref{sec:related}, we discuss related work in more detail.
In Section~\ref{sec:prelims} we record some useful technical facts we use from the literature and state slightly more general versions of the models which we consider. In Section~\ref{sec:nosos}, we give an alternating minimization algorithm for solving the offline case of Huber-contaminated linear regression.
In Section~\ref{sec:sos_strikes_back}, we give a sum-of-squares algorithm to handle the case of high contamination rate; combined with the result of the previous section, we obtain Theorem~\ref{thm:nonsos-intro}.
%In Section~\ref{sec:hd}, we extend this to the high-dimensional setting. 
In Section~\ref{sec:online}, we give a generic recipe for converting our fixed-design guarantees into online ones, thereby proving Theorem~\ref{thm:online_informal}. In Section~\ref{sec:apply} we apply the reduction of \cite{foster2020beyond} to our regression results to obtain our main result for contextual bandits, Theorem~\ref{thm:main_bandits_informal}. Lastly, in Section~\ref{sec:lowerbound}, we prove our lower bound, Theorem~\ref{thm:lowerbound_informal}. In Appendix~\ref{app:oracle} we verify that the reduction in \cite{foster2020beyond} applies to our Huber-contaminated setting.

\input{overview}
\section{Related Work}
\input{related_work}

\input{prelims}
\input{no_sos}

\input{sos_strikes_back}

%\input{sos}

%\input{high_dimensional} %obsolete

% \input{ellipsoid}

\input{new_ellipsoid}

\input{apply}

\input{lowerbound}

\paragraph{Acknowledgments} We thank Ainesh Bakshi and Dylan Foster for useful discussions related to their papers, \cite{bakshi2020robust} and \cite{foster2020beyond}, respectively.

\bibliographystyle{alpha}
\bibliography{biblio}

\appendix

\input{reduction}

\input{azuma} %fixed

\end{document}

%% file: intro.tex
\section{Introduction}
\label{sec:intro}

\subsection{Background}

The field of robust statistics was founded over five decades ago by John Tukey \cite{tukey1960survey, tukey1975mathematics}, Peter Huber \cite{huber1964robust} and others and seeks to design estimators that are provably robust to some fraction of their data being adversarially corrupted. However these estimators are generally not efficiently computable in high-dimensional settings \cite{bernholt2006robust, hardt2013algorithms}. After a decades long lull we have recently seen considerable progress in algorithmic robust statistics \cite{diakonikolas2019robust, lai2016agnostic, diakonikolas2017being, charikar2017learning, klivans2018efficient, diakonikolas2019sever, hopkins2018mixture, kothari2018robust,bakshi2020outlier,kane2020robust,diakonikolas2020robustly}. The first works \cite{diakonikolas2019robust, lai2016agnostic} focused on robust parameter estimation tasks, like robust mean estimation. The key insight from these works is that uncorrupted data often enjoys various spectral regularity properties, and this makes it possible to efficiently search for low-dimensional projections that can be used to identify corrupted data. 

Since then many of these ideas have found a number of exciting further applications, such as performing robust regression \cite{klivans2018efficient,bakshi2020robust,zhu2020robust,cherapanamjeri2020optimal} or minimizing a strongly convex function when your gradients can be adversarially corrupted \cite{diakonikolas2019sever}. However, what these works all share in common is that they are based on assumptions that the uncorrupted data is somehow evenly spread out. These assumptions can either come about by explicitly assuming a generative model, like a Gaussian \cite{diakonikolas2019robust} or a mixture of Gaussians \cite{bakshi2020outlier,kane2020robust,diakonikolas2020robustly}, or through a deterministic condition like hypercontractivity \cite{klivans2018efficient} or certifiable sub-Guassianity \cite{hopkins2018mixture, kothari2018robust}. 

Still, there is a widespread need for provably robust learning algorithms even in settings where these types of ``evenly spread out'' assumptions are just not appropriate. This is particularly the case in the context of online prediction \cite{cesa2006prediction} which operates in a setting where the input data is ever-changing and potentially even adversarially chosen. This flexibility allows it to capture challenging dynamic settings, as arise in reinforcement learning, where our learning algorithm interacts with the world around it and its decisions may in turn influence the next prediction task it is expected to solve. 
In this work we take an important first step towards answering a much broader question:

\begin{quote}
{\em Are there provably robust learning algorithms that can tolerate adversarial corruptions even for challenging high-dimensional and distribution-free online prediction tasks? }
\end{quote}

\noindent We will work in the Huber contamination model \cite{huber1964robust}. We will study two classic online learning problems: online linear regression with squared loss and linear contextual bandits. In unsupervised learning settings, the Huber contamination model posits that each random sample we get has an $\eta$ probability of coming from an arbitrary noise distribution chosen by an adversary instead of from our model. In our setting we will allow the feedback in each round to be arbitrarily corrupted with $\eta$ probability, and otherwise is subject to the usual stochastic noise. 

It turns out that for our problems the key challenge is to disentangle the effect of the \emph{dynamic range} of predictions vs. the effect of the \emph{noise level} on the overall regret guarantee.  
In particular, consider the basic linear regression problem where $(x_t)_{t = 1}^T$ is the input sequence of covariate vectors\footnote{In this paper, we will study the general case where these vectors are chosen adversarially and adaptively and the predictions are made online, but the importance of distinguishing dynamic range vs. noise level we discuss is relevant already in the basic (offline) setting.} and our goal is to robustly predict the response $y_t$.
Without adversarial corruptions, we assume the responses are generated according to the following well-specified model:
\[ y_t = \langle w^*, x_t \rangle + \xi_t \]
where $w^*$ is unknown  and $\xi_t$ is the noise, and our goal is to predict the clean, noiseless response $\langle w^*, x_t \rangle$ accurately. This problem is straightforward to solve with variants of Ordinary Least Squares \cite{azoury2001relative,vovk2001competitive} even in the online setting.
Now, consider what happens when we allow a random $\eta$ fraction of the responses $y_t$ to be adversarially corrupted, %(see Section~\ref{sec:model} for the formal model),
and our goal is to predict the \emph{clean/uncorrupted responses} $\langle w^*, x_t \rangle$ accurately. Let $R$ be the dynamic range of the true optimal predictions, so $| \langle w^*, x_t \rangle | \leq R$, and let $\sigma^2$ be the variance of $\xi_t$. When $\sigma^2$ is comparable to $R^2$, then the problem is relatively easy as there is (information-theoretically) not much that can be learned about $w^*$ in the first place. See the left panel of Figure~\ref{fig:noise} for an illustration.% and Corollary~\ref{cor:lb} for a formal lower bound.

\begin{figure}
    \centering
    \subfloat[][When $\sigma^2$ is comparable to $R^2$, many lines (range depicted in green), including the one found by ordinary least squares (orange), fit the data equally well (although the fit is not that good to begin with).]{\includegraphics[width=0.45\textwidth]{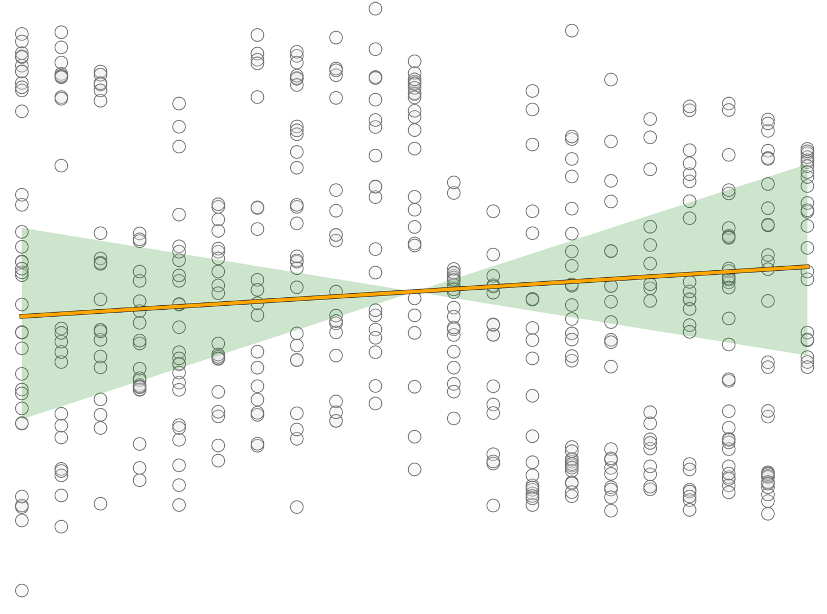}\label{fig:lownoise}}
    \qquad
    \subfloat[][When $\sigma^2$ is much smaller than $R^2$, then ordinary least squares (orange) fails, but in principle it should be possible to do much better even in high-dimensions.]{\includegraphics[width=0.45\textwidth]{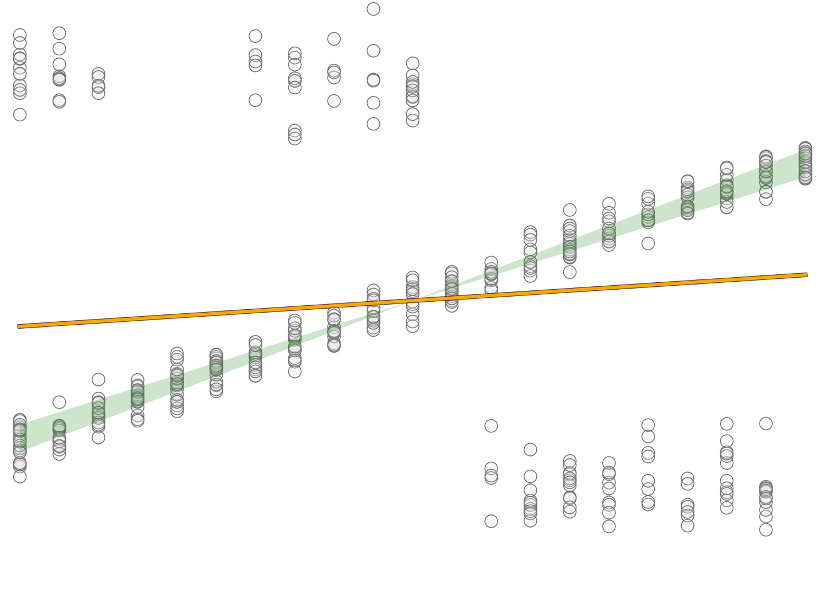}\label{fig:highnoise}}
    \caption{Datasets with equal contamination rates but different levels of noise $\sigma$. The corruptions are located in the upper left and bottom right parts of both figures. The goal in robust regression is to achieve low square loss on the \emph{uncorrupted points}. We depict in orange the ordinary least squares estimator and in green the range of linear predictors that would perform comparably to what our algorithms can achieve.}
    \label{fig:noise}
\end{figure}

In contrast we will be interested in the setting where $\sigma^2$ is much smaller than $R^2$, depicted in the right panel of Figure~\ref{fig:noise}. %, as is natural in many regression settings. 
It turns out that existing approaches break down in the sense that they pay an extra factor of $R$ or $R^2$ in the clean prediction error (resp. clean regret). Moreover getting around this dependence is a serious obstacle for the usual techniques: we show that regression using any convex surrogate (including Huber loss and $L_1$ loss) must pay this price (see Theorem~\ref{thm:lowerbound}).  Thus our main question is:

\begin{quote}
    {\em Is it algorithmically possible, in the presence of adversarial corruptions, to achieve average clean prediction error (resp. average clean regret) that is independent of $R$?}
\end{quote}

\noindent We answer this question in the affirmative for both online regression with squared loss and linear contextual bandits. Our algorithms succeed where convex surrogates fail, and are based on a novel alternating minimization scheme that interleaves OLS with carefully designed reweighting schemes found through SDPs. %Again, we emphasize that our results are new even in the basic/easier \emph{offline} regression setting, i.e. the conventional setting for Ordinary Least Squares, Huber Regression, and Least Absolute Deviations ($\ell_1$-regression), and our lower bounds apply even in this simple offline setting --- showing that these classical algorithms are suboptimal.

Finally we emphasize that the issue of $R^2$ vs. $\sigma^2$ dependence is quite relevant in modern reinforcement learning. In particular, there are many sequential tasks where at each step the variance in the losses/rewards is much smaller than the dynamic range. This can happen naturally when there are some catastrophic states that we must avoid, but at no point is the outcome of playing an action in a given state all that uncertain \--- e.g. when manipulating a robotic arm, some actions can require the application of orders of magnitude more torque. Thus our work may be viewed as a stepping stone towards achieving stronger and more meaningful robustness guarantees in reinforcement learning more broadly. %we ask: Is it possible to achieve strong robustness guarantees in model-based reinforcement learning, particularly in settings with large action/state spaces and function approximation? Since contextual bandits are often viewed as a natural ``halfway point'' between supervised learning and RL (see, for example, the discussion in \cite{abe2003reinforcement,dudik2011efficient}), this work serves as a first step in this direction.

\subsection{Our Results}
In this section, we present our main results for both linear regression and contextual bandits in the Huber contamination model. We go on to discuss related work (e.g. robust linear regression under distributional assumptions) in Section~\ref{sec:related} below. 

\paragraph{Distribution-free offline linear regression with Huber contamination.} We begin by discussing our results in the simplest setting we consider, which is the classical {\em offline} linear regression model with a Huber contamination adversary. In the clean version of this model, an arbitrary set of covariates $x_1,\ldots,x_n$ is fixed and clean responses are generated by 
\begin{equation}\label{eqn:clean-model}
y_t = \langle w^*, x_t \rangle + \xi_t 
\end{equation}
for some mean zero noise $\xi_t$; for example, if $\xi_t \sim N(0,\sigma^2)$ then $y_t \sim N(\langle w^*, x_t \rangle, \sigma^2)$. In the Huber contamination model, we relax the assumptions to a total variation distance ball around the generative model. In particular, using the coupling interpretation of total variation distance, this translates into the assumption that with probability $1 - \eta$ the response $y_t$ is generated by \eqref{eqn:clean-model} above, and with probability $\eta$ the response $y_t$ is sampled from an adversarially chosen noise distribution, which we allow to depend on all other randomness in the problem. In this setting, we obtain the following strong result (and for a fairly simple algorithm, see Technical Overview):
\begin{theorem}[Informal version of Theorem~\ref{thm:nonsos-subgaussian} and Theorem~\ref{thm:sos}]\label{thm:nonsos-intro}
%\TODO{FIXME: figure out whats up with $n$ vs $T$ notation}
Suppose that $\eta < 0.499$ is an upper bound on the contamination level,
and suppose for some $\sigma \ge 0$ that 
for all $1 \le t \le n$, $\|x_t\| \le 1$ and the noise $\xi_t$ is conditionally mean-zero and $\sigma^2$-subgaussian.
%and $\lambda \in \mathbb{R}$
%\begin{equation}\label{eqn:subgaussian-noise}
%     \log \E{\exp(\lambda \xi_t)} \le \lambda^2 \sigma^2/2
%\end{equation}
%almost surely. 
Suppose also that $\|w^*\| \le R$.
Then if $\eta = 0$ or $n \gtrsim \log(\min(n,d))/\eta$,
there exists a polynomial time algorithm outputting $w$ satisfying
the clean squared loss guarantee
%$\alpha = \Theta\left(\sqrt{\frac{\eta \log(d/\delta)}{n}}\right)$ and $\overline{\eta} = \eta$, the output $w$
%of \textsc{AltMin} with $poly(R/\sigma,\log(2/\delta),d,n)$ many steps satisfies for oblivious covariates the bound
\begin{align*} 
\sqrt{\frac{1}{n} \sum_{t = 1}^n \langle w^* - w, x_t \rangle^2}
    &\lesssim \eta \sigma\sqrt{\log(1/\eta)} + \eta^{1/8} R^{1/2} \sigma^{1/2} (\eta \sqrt{\log(1/\eta)})^{1/4}\sqrt[8]{\frac{\log(\min(n,d))}{n}} \\
    &\quad+  \eta^{1/4} R\sqrt[4]{\frac{\log(\min(n,d))}{n}} 
     + \min \left\{\sigma\sqrt{d/n}, (R\sigma)^{1/2}\sqrt[4]{1/n}\right\}
\end{align*}
with high probability.
\end{theorem}
Note that all but the first term are $o(1)$ as $n \to \infty$.
 On the other hand, when $\eta = 0$ only the last term remains and our result simplifies to standard (minimax optimal) guarantees for Ordinary Least Squares and Ridge regression, see e.g. \cite{keener2010theoretical,rigollethigh,shalev2014understanding}. Our result obtains the optimal dependence on $\eta$ up to the $\sqrt{\log(1/\eta)}$ factor, because the information-theoretic lower bound is $\Omega(\eta \sigma)$:
\begin{proposition}
\label{prop:lb}
    For any $0\le \eta < 1/2$, any algorithm for Huber-contaminated regression
    with Gaussian noise must incur clean square loss $\frac{1}{n} \sum_{t = 1}^n \langle w^* - w, x_t \rangle^2$ at least $\Omega(\eta^{2} \sigma^2)$.
\end{proposition}
This follows by embedding the 1-dimensional robust mean estimation problem in a straightforward way --- see Example~\ref{obs:gaussianmean}.

We also show the other aspects of the bound (lower bound on $n$, and the presence of additional ``middle terms'') are required --- see Example~\ref{example:extra}.
Our results generalize naturally to the setting with heavy-tailed noise, even without second moments, and achieve the optimal dependence on $\eta$ in those settings too. We defer the detailed statement of these variants to Section~\ref{sec:nosos}.

\paragraph{Impossibility of strengthening the adversary.} Before proceeding to the more sophisticated online settings we consider, we emphasize the impossibility of strengthening the adversary even in the basic model above. First, we consider the version of this problem where the adversary is allowed to corrupt an \emph{arbitrary} $\eta$ fraction of responses, as opposed to corrupting responses in random locations. In this case, the problem is trivially impossible even in $1$-dimension. If $1 - \eta$ fraction of $x_i$ are zero and $\eta$ fraction are $1$, $w^* = \pm R$, and the adversary corrupts an arbitrary $\eta$ fraction of responses, it's information-theoretically impossible to tell if $w^* = R$ or $w^* = -R$. Thus, we have the following lower bound:
\begin{proposition}[Impossibility with adversarial corruption locations]
In the linear regression model where an adversary corrupts an \emph{arbitrary} $\eta$ fraction of responses $y_t$, any algorithm must suffer clean squared loss $\frac{1}{n} \sum_{t = 1}^n \langle w^* - w, x_t \rangle^2$
at least $\Omega(\eta R^2)$.
\end{proposition}
We note variants of this example have already appeared previously in the literature, see e.g. Lemma 6.1 in~\cite{klivans2018efficient} or Theorem D.1 in~\cite{cherapanamjeri2020optimal}.
Similarly, we can consider a strengthened adversary which still corrupts in random locations, but is allowed to change the covariate $x_t$ as well as the response $y_t$. For essentially the same reason (the adversary can change covariates $x_t$ from 0 to 1 and label them with negated responses $y_t = \mp R$), it again becomes impossible to tell whether $w^* = R$ or $w^* = -R$ and so we have a strong impossibility result: 
\begin{proposition}[Impossibility with corrupted covariates]
In the linear regression model where an adversary corrupts an \emph{random} $\eta$ fraction of covariate and response pairs $(x_t,y_t)$, any algorithm must suffer clean squared loss $\frac{1}{n} \sum_{t = 1}^n \langle w^* - w, x_t \rangle^2$
at least $\Omega(\eta R^2)$.
\end{proposition}
Finally, we consider the ``breakdown point'' assumption $\eta < 1/2$. (We wrote $\eta < 0.499$ above only to simplify the statement.) If $\eta = 1/2$, a special case of our model is a balanced mixture of linear regressions where half of the responses are generated according to linear model $\langle w_1, x_t \rangle + \xi_t$ and the other half are generated according to a different linear model $\langle w_2, x_t \rangle + \xi_t$. By symmetry, it's impossible to know which of $w_1,w_2$ is the ground truth linear model, so a clean loss guarantee as in Theorem~\ref{thm:nonsos-intro} is information-theoretically impossible. In fact, in this setting even list recovery, i.e. outputting both $w_1$ and $w_2$, is computationally hard \cite{yi2014alternating} and this holds even if $\sigma = 0$.

\paragraph{Online linear regression with Huber contamination.} Next, we consider an \emph{online} version of the linear regression model from before. In this case, the algorithm faces two additional complications compared to before: 
\begin{enumerate}
    \item (Online prediction.) The algorithm is forced to output a prediction $\hat{y}_t$ given only $x_t$ and the information from previous rounds $(x_1,y_1),\ldots,(x_{t - 1},y_{t - 1})$, instead of being able to predict based on all of the data.
    \item (Adaptive covariates.) Instead of having the covariates $x_1,\ldots,x_T$ fixed in advance, i.e. chosen obliviously, the covariate $x_t$ is chosen \emph{adaptively} by the adversary, based on all information from rounds $1$ to $t - 1$. In particular, the algorithm's choices may affect the future inputs it receives.
\end{enumerate}
Nevertheless, we are able to give a version of our algorithm which deals with both of these issues. The statement below is for the finite-dimensional setting, but we also give a version of the result with no dependence on $d$ (Theorem~\ref{thm:sos-gd}), appropriate for the setting of kernel regression. As above, it has an optimal dependence on $\eta$ up to the log factor. In all online settings, we use $T$ for the total number of rounds/covariates to distinguish from the offline setting where we use $n$.
\begin{theorem}[Robust online regression, informal version of Theorem~\ref{thm:sosandcut}] \label{thm:online_informal}% and \ref{thm:sosandcut-hd}]
	In the setting of Huber-Contaminated Online Regression (see Definition~\ref{defn:huberreg}) with subgaussian noise, $\|x_t\| \le 1$ for all $t$ and $\|w^*\| \le R$, for any fixed $\eta < 0.499$, there exists an algorithm which runs in time $\poly(n,d)$ and outputs online predictions $\hat{y}_t$ which satisfy the following clean square loss regret bound with high probability:
	\begin{equation}
		\RegHSQ(T) = \sum_{t = 1}^T (\langle w^*, x_t \rangle - \hat{y}_t)^2 \lesssim \sigma^2 \eta^{2}\log(1/\eta) T +  \poly(R,\sigma,d,\eta) \cdot o(T).\label{eq:reghsq_informal}
	\end{equation}
\end{theorem}

\paragraph{Online contextual bandits with Huber contamination.} Finally, by combining our online linear regression result with a recent reduction from the contextual bandits literature (\cite{foster2020beyond}, see Appendix~\ref{app:oracle}), we obtain a result for contextual bandits with adaptive contexts and Huber-contaminated losses/rewards. We note that other reductions can probably be applied in the special case of stochastic contexts, e.g. \cite{simchi2020bypassing}, but for simplicity we only state a result in the more general setting with adaptive contexts.
First, we describe the interaction model for each round $t$: %\TODO{this description probably needs some tweaking.}
\begin{enumerate}
	\item Nature chooses context $z_t = (z_{ta})_{a \in \calA}$, possibly adversarially based on the transcript from previous rounds. Here $\calA$ with $K \triangleq |\calA|$ is the space of possible actions.
	\item Learner chooses action $a_t$ from $\calA$.
	\item A $\Ber(\eta)$ coin $\gamma_t$ is flipped to decide whether this round is corrupted.
	\item If $\gamma_t = 0$, i.e. the round is not corrupted, the learner sees loss $\ell^*_t(a_t) \triangleq \langle z_{ta}, w^* \rangle + \xi_t$ where $\xi_t$ is mean-zero noise.
	%, where $\ell^*_t$ is drawn independently from the distribution $\Pr[\ell^*_t]{\cdot | z_t}$ satisfying $\E{\ell^*_t(a_t) | z_t} = \langle z_{ta}, w^* \rangle$.
	\item If $\gamma_t = 1$, i.e. the round is corrupted, the learner sees an arbitrary loss $\ell_t(a_t)$ chosen by an adversary based on $z_t, a_t$, and the transcript from the previous rounds.
\end{enumerate}
In this model, the goal is to minimize the \emph{clean regret}, that is, to compete with the best policy $\pi$ in hindsight as measured by the \emph{true uncorrupted losses}. We obtain the following guarantee.
\begin{theorem}[Robust contextual bandits, informal version of Theorems~\ref{thm:main_cb_formal} and \ref{thm:main_cb_formal-hd}]\label{thm:main_bandits_informal}
%\TODO{probably needs some updates}
	In the setting of Huber-Contaminated Contextual Bandits (see Definition~\ref{defn:huber_bandits}) with $\sigma^2$-subgussian noise $\xi_t$, for any fixed $\eta < 0.499$, there is an algorithm which runs in polynomial time and selects actions $a_t$ which satisfy the following clean regret bound with high probability:
	\begin{equation}
		\RegHCB(T) = \sup_{\pi}\E*{\sumt (\ell^*_t(a_t) - \ell^*_t(\pi(z_t)))} \lesssim \left(\sigma \eta\sqrt{K\log(1/\eta)}\right) T + \poly(R,K,\eta,\sigma) \cdot o(T) \label{eq:psregret_informal}
	\end{equation}
	 where the supremum ranges over all (non-adaptive) policies $\pi$, see Preliminaries.
%	In the special case where $\epsilon = 0$, the clean \emph{regret} $\RegHCB(T)$ is upper bounded by the quantity on the right-hand side of \eqref{eq:psregret_informal} \emph{with high probability}.
\end{theorem}

\paragraph{An impossibility result: failure of convex $M$-estimators.} 
It may appear surprising that our algorithms for dealing with Huber contamination, even in the simplest linear regression setting, do not use an established approach like Huber regression or $L_1$/LAD (Least Absolute Deviation) regression --- classical approaches which have been studied for decades, and in the case of LAD, even as far back as the 1700s \cite{boscovich1757litteraria}. This is because there are fundamental reasons that \emph{neither} of these approaches can match our strong guarantees in the distribution-free setting. In fact, we prove a lower bound showing the failure of any $M$-estimator based on a convex loss function:
\begin{theorem}[Lower bound against convex $M$-estimators, informal version of Theorem~\ref{thm:lowerbound}]\label{thm:lowerbound_informal}
    There is an instance of Huber-contaminated linear regression where the covariates $x_t$ are drawn i.i.d. from a distribution, for which no vector $w$ obtained by minimizing a convex loss with respect to the Huber-contaminated distribution over $(x,y)$'s can achieve square loss better than $\Omega(\eta^3 R\sigma)$ on the true distribution.
\end{theorem}

%% file: overview.tex
%!TEX root = ./main.tex

\section{Technical Overview}
\label{sec:overview}

By a slight modification of the proof of Theorem 5 in \cite{foster2020beyond}, we can reduce the problem of achieving low clean regret in the contextual bandits setting of Definition~\ref{defn:huber_bandits} to that of producing an oracle for Hubert-contaminated online regression which gets low clean square loss regret. In this section, we overview the main ingredients for producing such an oracle. 
%For simplicity, we will focus on the special case where there is \emph{zero misspecification}, i.e. $\epsilon_t = 0$ at all time steps $t$.

There are two main steps: 1) designing an algorithm for fixed-design Huber-contaminated regression that achieves low square loss, and 2) a generic online-to-offline reduction based on cutting plane methods/online gradient descent.

\subsection{Huber-Contaminated Fixed-Design Regression}
\iffalse
In the fixed-design setting, we are simply given a collection of pairs $(x_1,y_1),\dots,(x_n,y_n)\in\R^d\times\R$, where a random subset of roughly $\eta \cdot n$ of the responses $y_t$ are corrupted adversarially. We are promised that for the indices $t$ for which $y_t$ was not corrupted, $y_t = \iprod{w^*,x_t} + \xi_t$ for some independent noise $\xi_t\sim\calD$,
%satisfying Assumption~\ref{assume:hypernoise}, 
and the goal is to output $\td{w}\in\R^d$ for which $\E[x\sim\calD]{\iprod{w^* - \td{w},x}^2}^{1/2}$ is small.
\fi
We start with the offline/fixed-design setting, where we are given an arbitrary fixed set of covariates $x_1,\ldots,x_n \in \mathbb{R}^d$ and for the indices $t$ for which $y_t$ was not corrupted, $y_t = \iprod{w^*,x_t} + \xi_t$ for some independent noise $\xi_t\sim\calD$. The exact assumption on the noise is not so important for the argument, since our algorithm is robust to Huber contamination: given an analysis for bounded noise, all the other versions of the results follow more or less by a straightforward truncation argument, treating heavy-tail events as outliers. 

%% seems like unecessary detail here?
%As mentioned in Section~\ref{sec:related}, a consequence of existing techniques \cite{chinot2020erm} is that M-estimation based on the Huber loss suffices to achieve \emph{$L_1$ error} $\eta\sigma$ in this setting, but as we show in Theorem~\ref{thm:lowerbound}, this approach would fail to achieve low square loss without additional assumptions that we do not make.
% % \TODO{as this result is not stated explicitly in Chinot's paper, we should add a little more discussion (e.g. say he proved $L^2$ under distributional assumptions, and the proof is easily modified to show this $L^1$ bound. also cite classical works on $L^1$ regression by Pollard etc? and add example showing Huber fails without distributional assumption in appendix.}
%Indeed, as we will now see, achieving low \emph{square loss} in this same setting, which is essential for the reduction in \cite{foster2020beyond}, is far more challenging.

%\paragraph{A Nonconvex Optimization Problem}
\paragraph{Spectrally Regularized Alternating Minimization.}
Similar to existing approaches in the robust statistics literature, our starting point is to formulate an optimization problem that searches for a regressor $w$ and a ``structured'' subset $S\subset[n]$ of size $(1 - O(\eta))n$ over which the clean square loss of $w$ is minimized, i.e.
\begin{equation}
	w,S = \argmin_{\substack{w,S: \\ S \ \text{large and ``structured''}}}\frac{1}{n}\sum_{t\in S}(y_t - \iprod{w,x_t})^2.\label{eq:opt}
\end{equation}
The subset $S$ should satisfy certain structural properties that the set of uncorrupted points $S^*\subseteq[n]$ would collectively satisfy and that can be used to \emph{certify} that the regressor we use is close to $w^*$.
Before we describe how the structural property that we use fundamentally differs from the ones exploited in prior works on robust regression, we first discuss our approach to optimizing the nonconvex objective \eqref{eq:opt}.
%walk through how to deal with the fact that \eqref{eq:opt} is nonconvex.
What we do is use a version of a standard heuristic, alternating minimization:
\begin{itemize}
    \item Given a candidate regressor $w$, we consider the optimization problem \begin{equation}\min_S \frac{1}{n}\sum_{t\in S}(y_t - \iprod{w,x_t})^2.\end{equation} We relax the set of $(1- O(\eta))n$-sized ``structured'' subsets $S$ to the set of $[0,1]$-valued ``structured'' weights $\{a_t\}_{t\in[n]}$ over the dataset satisfying $\sum_t a_t = 1 - O(\eta)$, and it will be apparent from our definition of ``structured'' below that this can be formulated as a basic SDP.
    \item Given a candidate set of weights $\{a_t\}_{t\in[n]}$, we solve the \emph{convex} optimization problem \begin{equation}\min_w \frac{1}{n}\sum_{t\in S}a_t (y_t - \iprod{w,x_t})^2.\end{equation}
\end{itemize}
By repeatedly alternating between these two steps, we arrive at an approximate first-order stationary point $(w,\brc{a_t})$: more precisely, one for which $\brc{a_t}$ is optimal given $w$ and for which 
\begin{equation}
    \frac{1}{n}\sum_{t\in[n]} a_t(y_t - \iprod{w,x_t})\iprod{x_t,v - w} \le o(1) \label{eq:overview_grad_zero}
\end{equation} for all $v$ of bounded norm (Lemma~\ref{lem:getfirstorder}). Of course, this stationary point does not have to be a global optimum of the objective function. %The bulk of the analysis goes into showing that stationarity in fact already suffices
Nevertheless, our analysis shows that any stationary point of our objective has strong statistical guarantees %for our purposes
(Section~\ref{subsec:allstationarygood}). 
%To show this, we will take $v$ in \eqref{eq:overview_grad_zero} to be the best linear predictor for the data in the absence of corruptions (see \eqref{eq:ustardef}). 
To show this, we can decompose the left-hand side of \eqref{eq:overview_grad_zero} for the choice $v = w^*$ into two quantities: 1) the contribution from the uncorrupted points, indexed by some subset $T\subset[n]$, and 2) the contribution from the corrupted ones, indexed by $[n]\backslash T$.%\footnote{The subset $T$ we work with in the analysis is defined slightly differently; we defer the details to Section~\ref{sec:nosos}.}

In 1), we can pull out the contribution from the quantity $\frac{1}{n}\sum_{t\in T}\iprod{x_t, w^* - w}^2$, %(see \eqref{eq:circle1}), 
which corresponds to the clean square loss achieved by the regressor $w$ we have found and turns out to be the dominant term. To upper bound the rest of 1) and 2), the key technical challenge is respectively to control the error incurred from failing to place nonzero weight $a_t$ on some of the points $t\in T$, and from placing nonzero weight $a_t$ on some of the points $t\not\in T$. To bound both sources of error, we end up needing to control the quantity
\begin{equation}
    \frac{1}{n}\sum_{t\in T}(1 - a_t)\iprod{x_t, w^* - w}^2.\label{eq:missedpoints}
\end{equation}
The way in which we do so marks the key distinction between our approach and that of previous works on robust regression. 

In prior works (see Section~\ref{sec:related} below), this is the place where one could insist that the weights $\{a_t\}$ are structured in the sense that along every univariate projection, the empirical moments of the dataset reweighted by $\{a_t\}$ are $k$-hypercontractive for some $k\ge 4$, in which case we could use Holder's to upper bound \eqref{eq:missedpoints}. This is not applicable %to our fixed-design setting where we make no assumptions on the process by which  
in the general case, where $x_1,\ldots,x_n$ are arbitrary bounded vectors, so a reweighting with hypercontractive empirical moments may not even exist.
Instead, our approach is to insist that $\brc{a_t}$ must \emph{sub-sample the empirical covariance}, i.e. that \begin{equation}
	\frac{1}{n}\sum_{t\in [n]} a_t x_tx_t^{\top} \succeq (1 - \eta) \frac{1}{n}\sum_{t\in[n]} x_tx_t^{\top} - o(1)\cdot\Id\label{eq:subsample_informal}
\end{equation}

The intuition for this constraint is that because the points that get corrupted in the Huber contamination setting form a \emph{random} subset of the data, the ideal reweighting $\brc{a^*_t}$ given by placing uniform mass on the true set of uncorrupted points would satisfy this constraint with high probability by standard matrix concentration. So for any $\brc{a_t}$ which sub-samples the empirical covariance, ignoring the low-order term in \eqref{eq:subsample_informal}, we can thus upper bound the quantity \eqref{eq:missedpoints} by $\eta\sum_{t\in [n]}\iprod{w^* - w,x_t}^2$. This is negligible compared to the aforementioned dominant term, allowing us to complete the proof that \eqref{eq:overview_grad_zero} suffices to ensure that $w$ incurs low clean square loss.

\paragraph{Optimal breakdown point via Sum of Squares.} It turns out that the above approach fails for $\eta$ larger than 1/3. Consider a scenario where $1/3$ of the data has been corrupted to come from a different linear model; in this case, there is a spurious local minima in which one takes $w$ in \eqref{eq:opt} to be the linear model generating the corrupted data and $S$ to consist of the corrupted data and a random half of the uncorrupted data (see Remark~\ref{rmk:breakdown-1/3} for further details).

To circumvent this issue, we appeal to a different algorithm when $1/3 \le \eta < 1/2$. Our starting point is the observation that another way of circumventing the nonconvexity of \eqref{eq:opt} is by considering the natural degree-4 sum-of-squares (SoS) relaxation of \eqref{eq:opt}. It turns out that an analysis similar to the one for our alternating minimization algorithm suffices to show that the pseudoexpectation one gets out of solving this relaxation achieves low clean square loss. At a high level, the reason is that one can extract from the former analysis a simple proof in the degree-4 SoS proof system that for $w$ and $S$ satisfying the constraints imposed by the SoS program and optimizing the objective of \eqref{eq:opt}, $w$ achieves low clean square loss. The key difference that allows us to circumvent the bad loss landscape of \eqref{eq:opt} when $\eta$ is large is that the SoS relaxation is guaranteed to produce a lower bound on the original (unrelaxed) problem \eqref{eq:opt}, whereas the objective value achieved by an arbitrary stationary point need not.

\paragraph{Other extensions.} Using existing generalization bounds \cite{srebro2010optimistic}, we give natural and fairly sharp versions of our results for the stochastic/random-design setting. The analysis we outlined works with heavy-tailed noise in $L_q$ for any $q > 1$ and achieves the optimal dependence on $\eta$ in this setting. If we only use the estimator described above, the sample complexity of our estimator with small confidence parameter $\delta$ is not as good with heavy-tailed noise as with subgaussian noise; we show how to improve the sample complexity when $q \ge 2$ by combining our estimator with a simple median-of-means approach from the heavy-tailed regression literature \cite{hsu2016loss,minsker2015geometric}.

\subsection{Online-to-Offline Reduction}

We now explain how to use the guarantee of the previous section to get an algorithm for online regression. At a high level, the idea is to use the fixed-design guarantee above to design a \emph{separation oracle} between whatever bad predictor we might be using at a particular time step, and the small ball $\calB$ of good predictors $w$ around $w^*$, any of which would incur sufficiently low regret over any possible sequence of samples. This reduction has a similar spirit to the ``halving'' algorithm from online learning \cite{shalev2011online}, and efficient variants for halfspace learning based on the ellipsoid algorithm \cite{yang2009online,tewarinotes}.

Concretely, suppose inductively we have seen samples $(x_1,y_1),\ldots,(x_n,y_n)$ thus far and have used some vector $w$ to predict in the last $m$ steps where we were given $(x_{n-m+1},y_{n-m+1}),\ldots,(x_n,y_n)$. Let $\Sig$ be the average of $x_i x_i^T$ over the last $m$ steps. One of two things could be true.
% There are two possibilities: either $w$ already lies in $\calB$, in which case we have no reason to ever update our predictor in future steps because we are guaranteed to achieve low regret, or $w$ lies outside of $\calB$. How do we detect whether we are in the latter situation?

It could be that in these last $m$ steps, $w$ actually performed well, that is, $\norm{w - w^*}^2_{\Sig}$ is small, either because $w\in\calB$ or because $x_{n-m+1},\dots,x_n$ mostly lie in the slab of space where $w$ and $w^*$ yield similar predictions. Either way, because the prediction error under $w$ has been small so far, there is no need to update to a new predictor just yet.

Alternatively, if $\norm{w - w^*}^2_{\Sig}$ is large, then the gradient of the function $w\mapsto \norm{w - w^*}^2_{\Sig}$ would give a separating hyperplane between $w$ and $\calB$. Of course, the issue with this is that we don't know $w^*$.
To get around this, recall from the fixed-design guarantee that if we ran the alternating minimization algorithm above on the data $(x_{n-m+1},y_{n-m+1}),\dots,(x_n,y_n)$ (assuming $m$ is large enough that things concentrate sufficiently well), then the resulting vector $\td{w}$ is close to $w^*$ under $\norm{\cdot}_{\Sig}$. So to check whether $\norm{w - w^*}^2_{\Sig}$ is large, by triangle inequality we can simply check whether $\norm{w - \td{w}}^2_{\Sig}$ is large! If so, the gradient of $w\mapsto \norm{w - \td{w}}^2_{\Sig}$ gives us a separating hyperplane that we can actually compute.

To summarize, the contrapositive of this tells us that if we don't form a separating hyperplane in a given step, then we know $\norm{w - w^*}^2_{\Sig}$ is small and we are content to continue using $w$. Conversely, if we do form a separating hyperplane, we know we won't cut $\calB$. This is because every point in $\calB$ is, by design, close to $w^*$ under any norm $\norm{\cdot}_{\Sig}$ defined by the empirical covariance $\Sig$ of a sequence of samples.

With these two facts in hand, we can safely run
%, e.g. ellipsoid or Vaidya's 
a cutting plane algorithm like ellipsoid or Vaidya's method to update our predictor every time we find a separating hyperplane and ensure that after a bounded number of updates, we find a predictor that will achieve low regret on subsequent steps.

\paragraph{Handling the high-dimensional case.}

The above approach does not work when the dimension is unbounded, e.g. in kernelized settings, because the guarantees of cutting plane methods are inherently dimension-dependent. We now describe an alternative approach based on wrapping online gradient descent around our guarantee for Huber-contaminated fixed-design regression.

Instead of using Vaidya's algorithm to update the vector $w$ that we predict with whenever the separation oracle returns $\nabla \varphi_t(w)$, we can imagine updating $w$ by simply stepping in the direction of $-\nabla\varphi_t(w)$. The key challenge is to bound the number of times $V$ we get a hyperplane from the separation oracle and have to make such a step, because as long as we don't receive any new hyperplanes, the predictions we make will incur low square loss.
For this, we can appeal to the the fundamental regret bound for online gradient descent \cite{zinkevich2003online}. Specifically, if we receive a sequence of convex losses $\varphi_1,\ldots,\varphi_V$ and play a sequence of inputs $w_1,...,w_V$ where $w_{t+1}$ is given by taking a gradient step with respect to $\varphi_t$ from $w_t$, then the cumulative loss $\sum \varphi_t(w_t)$ incurred only exceeds $\sum \varphi_t(w^*)$ for any single move $w^*$ by an $O(\sqrt{V})$ term (see Theorem~\ref{thm:gd}). But because the separation oracle is called only when $\varphi_t(w_t) \approx \varphi_t - \varphi_t(w^*)$ is large, this immediately implies that $V$ is bounded.

\subsection{Lower Bound for Convex Losses}
At a high level, the intuition for why convex losses fails is this: in order for the algorithm to be robust to outliers, the loss needs to look roughly like the $L_1$ loss (e.g. the Huber loss looks like the $L_1$ loss except in a ball near the origin). However, the $L_1$ loss $\E{|Y - \langle w, X \rangle|}$ is much less sensitive to making errors for $X$ lying in rare areas of the space than the usual $L_2$/squared loss $\E{(Y - \langle w, X \rangle)^2}$. In order to take advantage of this, we construct a 1-dimensional example with $1 - \Theta(\eta/R)$ fraction of the covariate distribution a delta mass at $1/R$ and the remainder a delta mass at $-1$. For simplicity, we take the noise variance $\sigma = 1$. By having the adversary corrupt the response for the much more common portion of the data at $1/R$, the $L_1$ regression is tricked into making an order $\eta R$ error on the rare portion of the data, which causes a squared loss of $\Omega((\eta/R) \eta^2 R^2) = \Omega(\eta^3 R)$. By appropriately generalizing this argument, we rule out the success of all convex losses.

%% file: related_work.tex
%\subsection{Related Work}
\label{sec:related}

\paragraph{Robust regression, when both the covariates and responses are corrupted}

As discussed in Section~\ref{sec:intro}, our work is closely tied to the long line of recent work on designing efficient algorithms for robust statistics in high dimensions. We refer to~\cite{li2018principled,steinhardt2018robust,diakonikolas2019recent} for comprehensive surveys of this literature and focus here on the results related to regression \cite{klivans2018efficient,bakshi2020robust,zhu2020robust,pensia2020robust,diakonikolas2019sever,prasad2020robust,diakonikolas2019efficient,cherapanamjeri2020optimal}. These works are for the stochastic setting where the covariates are drawn i.i.d. from some distribution $\calD_x$ but work in a corruption model where the adversary can arbitrarily alter any $\eta$ fraction of the responses \emph{and} the corresponding covariates. All of these results operate under the assumption that the underlying distribution $\calD_x$ is either Gaussian or at least 4-hypercontractive. This is not merely an issue of convenience: in the absence of such assumptions, it is impossible to do anything even in one dimension under this corruption model. We recall the following example from the Results section above:
\begin{example}
	Let $d = 1$ and $\epsilon = 0$, and suppose $w^* = R$. Suppose the distribution over covariates is $Ber(\eta)$, i.e. it has $1 - \eta$ mass at 0 and $\eta$ mass at 1.  Suppose the adversary corrupts an $\eta$ fraction of the pairs $(0,0)$ to be $(1,-R)$. Then it is impossible for the learner to distinguish whether $w^* = R$ or $w^* = -R$. 
\end{example}
%\TODO{state propositions more explicitly: no adversarial locations, no contaminatino in covariates}
We note that variants of this example have already appeared previously in the literature, see e.g. Lemma 6.1 in~\cite{klivans2018efficient} or Theorem D.1 in~\cite{cherapanamjeri2020optimal}. This does not contradict prior results which make distributional assumptions, because they consider the case where $\eta$ is small: when $\eta = o(1)$, $Ber(\eta)$ is no longer $O(1)$-hypercontractive as its fourth moment is $\eta R^4$ while the square of its second moment is $\eta^2 R^4$.
In summary: when there exist rare features in the data, or when the corruption fraction $\eta$ is large, it is simply not information-theoretically possible to handle corruption in the covariates. 

We also note that the work of \cite{pensia2020robust} shows that, at least in some cases, the covariate corruption can be handled separately from the response corruption by first running a standard filtering method on the covariates, and second running a method robust to response outliers (in their case, Huber regression) on the remaining data. This suggests that handling covariate corruption (when it is possible) and response corruption may be largely orthogonal problems. Finally, one commonality with our work and much of the previous literature is the use of Sum of Squares programming (for us, only needed near the breakdown point $1/2$); however, we use a fairly simple degree-4 SoS program, as opposed to prior work (e.g. \cite{klivans2018efficient,bakshi2020robust}) where the SoS degree and sample complexity need to be large in order to take advantage of stronger regularity assumptions.

\paragraph{Robust regression, when just the responses are corrupted}

A milder corruption model which has received significant attention in the statistics literature is the setting where a fraction, either randomly or adversarily chosen, of the \emph{responses} are corrupted, while the covariates are left intact.
One popular approach for regression in this setting is M-estimation \cite{loh2017statistical,zhou2018new}, originally introduced by Huber \cite{huber1964robust}, in which one minimizes a loss function with suitable robustness properties. Common choices of loss function include the $L_1$ loss and the Huber loss. In addition to the earlier asymptotic results for this approach \cite{bassett1978asymptotic,huber1973robust,pollard1991asymptotics}, by now numerous works have obtained non-asymptotic guarantees for M-estimation under a variety of models for how the responses are corrupted, but predominantly under the assumption that the design is sub-Gaussian or similarly structured \cite{karmalkar2018compressed,dalalyan2019outlier,Sasai2020RobustEW,dorsi2020regress}. Notably, in \cite{dalalyan2019outlier,Sasai2020RobustEW} it was shown that in the setting of sparse linear regression with Huber-contaminated responses, M-estimation with ($\ell_1$-regularized) Huber loss is nearly minimax-optimal when the noise distribution $\calD$ and the covariates are i.i.d. Gaussian.

One exception, and perhaps the result closest in spirit to our results for regression, is that of \cite{chinot2020erm}. One consequence of the results in this work is that in the random-design setting of Definition~\ref{defn:huberreg}, that is when the covariates are drawn i.i.d. from some distribution $\calD_x$, then if the function class (equivalently, covariate distribution) is hypercontractive in the sense that for any $w\in\calW$, $\E[\calD_x]{\iprod{w - w^*,x}^p}^{2/p} \le \E[\calD_x]{\iprod{w - w^*,x}^2}$ for some $p > 2$, and if the noise distribution $\calD$ satisfies suitable conditions, then M-estimation with Huber loss achieves the information-theoretically optimal error of $\Theta(\sigma^2\eta^2)$ in squared loss. It is also possible to modify their proof to show that the same algorithm would yield the information-theoretically optimal error of $\Theta(\sigma\eta)$ in a different metric, the \emph{$L_1$ loss}, without the hypercontractivity condition. An $L_1$ guarantee is much weaker than the usual $L_2$ (i.e. squared loss) guarantee: for example, it is too weak to give anything interesting for the contextual bandits application.

In fact, as we show in Theorem~\ref{thm:lowerbound}, M-estimation with Huber loss, and more generally minimization of \emph{any} convex surrogate loss, will not achieve squared loss $\Theta(\sigma^2\eta^2)$ in general when the function class/covariate distribution fails to satisfy this hypercontractivity condition. Instead, we show such estimators must pay squared loss at least $\Omega(\sigma R \eta^3)$. 
We also mention that to our knowledge, the only work that has explicitly considered \emph{online} regression with corruptions is \cite{pesme2020online}, where they considered Gaussian covariates and a random fraction of responses are corrupted by an \emph{oblivious} shift. Additionally, another notable line of work to mention in the literature on regression with contaminated responses stems from using hard thresholding \cite{bhatia2015robust,bhatia2017consistent,suggala2019adaptive}, though these works work also make strong regularity assumptions on the covariates.

Lastly, we mention that in the context of \emph{classification}, there have been a number of recent works giving new algorithmic results for corruption models where the binary labels are corrupted by some process that is halfway between purely stochastic and purely adversarial. For instance, \cite{diakonikolas2019distribution,chen2020classification,diakonikolas2020learning} focus on the \emph{Massart noise model} which can essentially be viewed as a setting where an adversary can only control a random fraction of the labels, but can change them in an arbitrary way. This can be thought of as the classification version of the Huber-contaminated regression problem that we consider in the present work, and the former two results work in the setting without distributional assumptions. We also note that the recent work of \cite{diakonikolas2020polynomial} considers the stronger model of \emph{Tsybakov noise} and obtains polynomial-time algorithms under distributional assumptions.

\paragraph{Robustness for bandits}

There have been a number of notions of robustness proposed in the bandits literature. A classic notion is that of adversarial bandits, a setting where one would like to prove regret bounds even when the rewards are chosen adversarially \cite{auer2002nonstochastic}. Many papers have worked to identify ways of interpolating between fully adversarial rewards and stochastically generated ones, including the line of work on ``best of both worlds'' results \cite{bubeck2012best,seldin2014one,auer2016algorithm,seldin2017improved} as well as an interesting model of bandits with adversarial corruptions introduced by \cite{lykouris2018stochastic} and subsequently studied by \cite{gupta2019better}. The latter is a setting of multi-armed bandits where rewards are generated stochastically but then perturbed by an adaptive adversary with a fixed budget of how much he can move the rewards in any given sample path. {\em We stress that the setting of adversarial bandits is orthogonal to the thrust of the present work, where the goal is to get small clean regret.} For example, while the adversarial nature of the rewards makes the former quite challenging, it is still possible to achieve sublinear regret for adversarial bandits, whereas in our setting, one cannot do better than $\Omega(\eta^2\sigma^2 T)$.

Other notions of robustness that have been considered include the standard notion of misspecification (e.g. \cite{foster2020beyond,neu2020efficient}) as in Definition~\ref{defn:huber_bandits}, as well as the notion of heavy-tailed reward distributions \cite{bubeck2013bandits}.
The setting of Huber-contaminated rewards that we study was previously studied in the multi-armed case by \cite{kapoor2019corruption,altschuler2019best}. \cite{kapoor2019corruption} also studied Huber-contaminated linear contextual bandits when the contexts are Gaussian or collectively satisfy some RSC-like condition. Even in this distribution-specific setting, their analysis loses a factor of $R$. A recent work \cite{awasthi2021online} also studied the Gaussian context case of Huber-contaminated linear contextual bandits and improved over \cite{kapoor2019corruption}; however their result also suffers from a dependence on $R$.
Lastly, we mention the work of \cite{seldin2014one,zimmert2019optimal} who considered a different corruption model for the multi-armed case where the contaminations cannot reduce the ``gap,'' i.e. the difference between the reward of the best arm and that of any other arm, by more than a constant factor in any time step.

%% file: prelims.tex
%!TEX root = ./main.tex

\section{Preliminaries}\label{sec:prelims}
\subsection{Formal Description of Models}
 For technical reasons which will appear naturally in the analysis, it is useful for us to consider the general \emph{misspecified} model where $\epsilon \ge 0$ is a misspecification parameter that accommodates deviation between the true prediction rule and the best linear model.
 However, the reader should feel free to consider the usual well-specified setting $\epsilon = 0$ when reading the results.
 
 \paragraph{Robust Offline Regression.}
Our analysis in the oblivious setting allows the corruption adversary to depend arbitrarily on the randomness in the problem, as in e.g. \cite{chinot2020erm}. This is different from in the online setting, where it's important that all of the randomness respects the filtration corresponding to time. To be clear, we define the offline model explicitly here.
%The guarantee at the end for Theorem~\ref{thm:nonsos} and Corollary~\ref{cor:random-design} is exactly the same. This is the setup:
\begin{enumerate}
    \item Covariates $x_1,\ldots,x_n$ are arbitrary fixed vectors in the unit ball of $\mathbb{R}^d$, i.e. they are chosen obliviously.
    \item For every $t$ from $1$ to $n$, a $Ber(\eta)$ coin is flipped to determine if round $t$ is corrupted or not. Let $a^*_t$ be the indicator for whether round $t$ was uncorrupted, i.e. $a^*_t = 1$ when the round is \emph{not} corrupted and this occurs with probability $1 - \eta$.
    \item For every uncorrupted round, we observe $y_t$ given by \begin{equation}
	y_t = y^*_t + \xi_t, \qquad y^*_t = \iprod{w^*,x_t} + \epsilon_t
\end{equation} where $w^*$ is the true regressor and $\|w^*\| \le R$, and $\xi_t$ is independently sampled from the noise distribution $\calD$ and $\abs{\epsilon_t} \le \epsilon$ is the misspecification. The misspecification $\epsilon_t$ can be chosen in a completely adversarial fashion: formally, it is a random variable depending arbitrarily on all other randomness in the setup (e.g. it can depend arbitrarily on the noise and the coin flips from all rounds).
    \item For every corrupted round, $y_t$ is chosen freely by the adversary. Again, we assume nothing about $y_t$ -- it can depend arbitrarily on all other randomness in the problem.
\end{enumerate}
 
 \paragraph{Robust Online Regression.}
We begin by introducing the setup for the online linear regression problem, which is closely related to the linear contextual bandits problem we introduce later. 
%This problem and our results are interesting even in the less general \emph{offline} setting, where covariates are chosen obliviously and the algorithm gets to see all data before making a prediction, a classical model which has been intensely studied in robust statistics (see Related Work below). 
Online regression itself is one of the fundamental problems in online learning that has been extensively studied in the uncontaminated setting, see e.g. \cite{vovk2001competitive,azoury2001relative,cesa2006prediction}.

\begin{definition}[Huber-Contaminated Online Regression]\label{defn:huberreg}
	Fix Huber contamination rate $\eta \in (0,1/2)$, misspecification bound $\epsilon$, %maximum loss $R$, %\Sitan{is this the right terminology?}
	noise distribution $\mathcal{D}$,
	and unknown weight vector $w^*$.
% 	Flip $T$ independent $\Ber(\eta)$ coins to determine which rounds get \emph{corrupted}, and sample $\xi_1,\dots,\xi_T$ independently from a fixed noise distribution $\calD$.
	In each round $t\in[T]$:
	\begin{enumerate}
		\item Nature chooses input $x_t \in \mathbb{R}^d$, possibly adversarially based on the transcript from previous rounds.
		\item Learner chooses prediction $\wh{y}_t$.
		\item A $\Ber(\eta)$ coin is flipped to decide whether this round is corrupted.
		\item If the round is not corrupted, sample $\xi_t$ independently from $\calD$. The learner sees $y_t \triangleq y^*_t + \xi_t$, where $y^*_t \triangleq \langle w^*, x_t \rangle + \epsilon_t$ for some quantity $\epsilon_t(x_t)$ satisfying $\abs{\epsilon_t(x_t)} \le \epsilon$.
		\item If the round is corrupted, the learner sees an arbitrary $y_t$ chosen by an adversary based on $x_t$ and the transcript from the previous rounds.
	\end{enumerate}

	The goal of the learner, given any $x_t$ in round $t$ (and the transcript from the previous rounds), is to choose a prediction $\wh{y}_t$ such that with high probability over the choice of $\Ber(\eta)$ coins, and for any (possibly adaptively chosen) sequence of feature vectors $\brc{x_1,\ldots,x_T}$ in the above model, the quantity
	\begin{equation}
		\RegHSQ(T) = \sum^T_{t=1} (\wh{y}_t - y^*_t)^2.\label{eq:linreg}
	\end{equation} is small. We say that $A$ achieves \emph{clean square loss regret $\RegHSQ(T)$}. Note that $\RegHSQ$ is a random variable depending on the randomness of the $\Ber(\eta)$ coins, the randomness of the noise $\xi_t$, any stochasticity in the choice of the inputs $x_t$, and the randomness of the learner and adversary. We will establish high-probability bounds on this random variable.
\end{definition}
\begin{remark}[Clean vs Dirty Loss]
It is very important to note that the goal for robust statistics is to minimize the \emph{clean square loss} $\sum_{t = 1}^T (\hat{y}_t - y^*_t)^2$ and not the ``dirty'' square loss $\sum_{t = 1}^T (\hat{y}_t - y_t)^2$ where $y_t$ is potentially corrupted. If our goal was to try to fit the corruptions, as in agnostic learning, then using Ordinary Least Squares would be a good approach for this regression problem.

On the other hand, there is no importance difference between optimizing the \emph{noisy clean square loss} $\sum_{t = 1}^T (\hat{y}_t - (y^*_t + \xi_t))^2$ and the clean square loss as defined above. Because the noise is by definition independent of $\hat{y}_t,y^*_t$, we know that in expectation $\E{\sum_{t = 1}^T (\hat{y}_t - (y^*_t + \xi_t))^2} = \E{\sum_{t = 1}^T (\hat{y}_t - y^*_t)^2} + \sigma^2 T$ and so the additive term coming from the noise doesn't depend on the prediction sequence $\hat{y}_t$.
\end{remark}
\iffalse % unneeded discussion
\begin{remark}[Online Regression generalizes Offline Regression]
The online regression setup is very general and, as a special case, captures the stochastic and fixed-design \emph{offline} regression problems which are more classically studied in robust statistics. In the offline \emph{fixed design} setting, the sequence $x_1,\ldots,x_T$ is chosen obliviously and the predictions $\hat{y}_1,\ldots,\hat{y}_T$ are allowed to depend on all of the responses $y_1,\ldots,y_T$. In this case, the objective \eqref{eq:linreg} is conventionally normalized by $1/T$ and referred to as the (clean) \emph{Mean-Squared Error} (MSE). In the offline \emph{random design}/stochastic setting, the sequence $x_1,\ldots,x_T$ is sampled i.i.d. from a fixed distribution $\mathcal{D}_x$ over covariate vectors, and the goal is to minimize the population loss $\E{(\hat{y}_{T + 1} - y^*_{T + 1})^2}$ over the randomness of test data $x_{T + 1}, y^*_{T + 1}$. A standard \emph{online-to-batch} conversion argument \cite{shalev2011online} shows that our online algorithms imply solutions to these offline regression problems, but we also give direct and explicit algorithms (see Corollary~\ref{cor:random-design}) for the offline tasks.
\end{remark}
\fi

\paragraph{Connection to robust mean estimation} Note that regression with Huber contaminations is at least as hard as the problem of mean estimation under Huber contaminations, implying that achieving sublinear regret for Huber-contaminated online regression is impossible:

\begin{example}\label{obs:gaussianmean}
	Let $d = 1$ and $\epsilon = 0$, and suppose $w^* = R$ and $\calD = \N(0,\sigma^2)$. Suppose we only ever see $x_t = 1$, so that we always have $y^*_t = R$. Then each uncorrupted $y_t$ is simply an independent draw from $\N(R,\sigma^2)$, so the question of producing a good predictor $\wh{y}$ in this special case is \emph{equivalent} to that of estimating the mean of a univariate Gaussian with variance $\sigma^2$ under the Huber contamination model. It is known that one cannot do this to error better than $\Omega(\eta\sigma)$ (see \cite{diakonikolas2018robustly}).
	More generally, if we only assume $\calD$ has hypercontractive moments up to degree $k$, one can devise distributions $\calD$ for which one cannot do better than error $\Omega(\eta^{1 - 1/k}\sigma)$ (see e.g. Fact 2 from \cite{hopkins2019hard}).
\end{example}
\paragraph{Robust Contextual Bandits.}
We study the following robust version of contextual bandits, first introduced in \cite{kapoor2019corruption}. We first state the general form of the contextual bandits model (for an abstract regression function $f$), then specialize to the linear case.
\begin{definition}[Huber-Contaminated Contextual Bandits]\label{defn:huber_bandits}
	Let $\calZ$ be an arbitrary state space, and let $\calA$ be an action space of size $K$. Fix Huber contamination rate $\eta \in (0,1/2)$, misspecification rate $\epsilon$, %maximum loss $R$, %\Sitan{is this the right terminology?}
	and unknown function $f:\calZ\times\calA\to\R$. Ahead of time, an oblivious adversary chooses distributions $\Pr[\ell^*_t]{\cdot|z_t}$ over loss functions $\ell^*_t: \calA\to\brk{0,R}$ for all possible contexts $z_t$ and all time steps $t\in[T]$. We assume the conditional means of the loss distributions are \emph{realized up to misspecification $\epsilon$} by $f$, i.e. for all $t,z,a$, \begin{equation}
		\E[\ell^*_t]{\ell^*_t(a) | z_t  = z} = f(z,a) + \epsilon_t(z,a), \qquad \abs{\epsilon_t(z,a)} \le \epsilon. \label{eq:realizable}
	\end{equation}
	Let $\xi_t$ be the random variable which, conditioned on $z_t = z$, takes on the value \begin{equation}
		\xi_t \triangleq \ell^*_t(a) - f(z,a) - \epsilon_t(z,a),
	\end{equation} and define \emph{noise parameter} $\sigma$ by $\sigma^2 \triangleq \sup_{z,t}\E{\xi_t^2 | z_t = z}$.
	In each round $t\in[T]$:
	\begin{enumerate}
		\item Nature chooses $z_t$, possibly adversarially based on the transcript from previous rounds.
		\item Learner chooses action $a_t\in\calA$.
		\item A $\Ber(\eta)$ coin $\gamma_t$ is flipped to decide whether this round is corrupted.
		\item If $\gamma_t = 0$, i.e. the round is not corrupted, the learner sees loss $\ell^*_t(a_t)$, where $\ell^*_t$ is drawn independently from the distribution $\Pr[\ell^*_t]{\cdot | z_t}$.
% 		\TODO{If this is the bandit setting, shouldnt the learner only observe the loss for the chosen
		\item If $\gamma_t = 1$, i.e. the round is corrupted, the learner sees an arbitrary loss $\ell_t(a_t)$ chosen by an adversary based on $z_t, a_t$, and the transcript from the previous rounds.
	\end{enumerate}

	The goal of the learner in the adversarial setting is to compete with the best policy in hindsight as measured by the \emph{clean losses} $\ell^*_t$ incurred in every round, that is to select a sequence of actions $a_1,\dots,a_T$ for which 
% 	\TODO{I think expectation is not part of definition of regret? instead expectation is called the ``expected regret''. see \url{https://cims.nyu.edu/~mohri/amls/aml_bandit.pdf}}
	\begin{equation}
		\psRegHCB(T) = \sup_{\pi}\E*{\sum^T_{t=1}\left(\ell^*_t(a_t) - \ell^*_t(\pi(z_t))\right)},\label{eq:psreg}
	\end{equation} is small, where the supremum ranges over all (non-adaptive) policies $\pi:\calX\to\calA$ and the expectation is over the randomness of the $\Ber(\eta)$ coins, the randomness of the rewards, any stochasticity in the choice of contexts, and the randomness of the learner. We say that such a learner achieves \emph{clean pseudo-regret $\psRegHCB(T)$}.
	
	In the special case where $\epsilon = 0$, we will consider the quantity \begin{equation}
		\RegHCB(T) = \sum^T_{t=1}\left(\ell^*_t(a_t) - \ell^*_t(\pi^*(z_t))\right)
	\end{equation}
	where $\pi^*(z) \triangleq \arg\max_a f(z,a)$. Note that this is a random variable in the same things defining the expectation in \eqref{eq:psreg}. We say that a learner achieves \emph{clean regret $\RegHCB(T)$}. We will establish high-probability bounds on $\RegHCB$.
\end{definition}
\begin{definition}[Huber-Contaminated Linear Contextual Bandits]\label{def:lin_bandits}
This is the special case of Definition~\ref{defn:huber_bandits} where the regression function $f : \mathcal{X} \times \mathcal{A} \to \mathbb{R}$ is linear in the following sense. The context space $\mathcal{X}$ is a Hilbert space and each context is of the form $z_t = (z_{t1},\ldots,z_{tK})$, i.e. there is a separate context vector for each arm. Then we assume that
\[ f(z,a) = \langle z_{ta}, w^* \rangle \]
for some vector $w^* \in \mathbb{R}^d$.
\end{definition}

Without adversarial corruptions this is the familiar linear contextual bandits problem, which has a wide range of applications precisely because in many settings the context is an important component of the prediction task. For example, in online advertising the choice of which ad to display ought to depend on information about the webpage that the ad will be displayed on as well as any information we have about the user we are displaying it to, which can be encoded as a high-dimensional vector. In healthcare, when we want to choose between various treatment options again we want to adapt to the relevant context such as the patient history. For additional applications, see the survey \cite{bouneffouf2019survey}. 

However in many of these settings it is natural to imagine that some of the feedback we receive departs in arbitrary ways from the model. This could happen in online advertising due to clickfraud, particularly when malware takes over a user's account. It could happen in healthcare in the context of drug trials, particularly ones that measure some real valued variable, when there are testing errors or confounding variables that are difficult to model.  For all these and many more reasons it is natural to wonder if there could be algorithms for contextual bandits with stronger robustness guarantees. 

\begin{remark}
We note that in some papers on contextual bandits, the range of the loss functions is normalized to $\brk{0,1}$ for convenience.
%, but in the Huber-contaminated setting we focus on, we crucially want to avoid any $R$ dependence in the dominant term of our pseudo-regret/regret bounds. Equivalently,
The scale-invariant quantity which we want to avoid dependence on is the ratio $R/\sigma$.
\end{remark}

% \TODO{Should $\RegHCB(T)$ show up in the equation above?}

% By a slight modification of the argument from \cite{foster2020beyond} (see Appendix~\ref{app:oracle}), we can exhibit a black-box reduction from  an oracle for the regression analogue of the setup in Definition~\ref{defn:huber_bandits}, which we now define.

\begin{remark}
    As we will rely on a formal connection between contextual bandits and online regression illuminated in \cite{foster2020beyond}, it will be helpful to situate our definitions in their context. In particular, when $\eta= 0$, Definition~\ref{defn:huber_bandits} specializes to Assumption 4 of \cite{foster2020beyond}, and an algorithm for Definition~\ref{defn:huberreg} achieving clean square loss regret at most $\RegHSQ(T)$ would satisfy Assumption 2b of \cite{foster2020beyond} in the realizable case with $\epsilon$-misspecification.
\end{remark}

%As the bulk of the technical contributions of this work is focused on obtaining guarantees for Huber-contaminated online regression, we will simplify notation somewhat by referring to $(z_t,a_t)$ in Definition~\ref{defn:huberreg} simply as $x_t$, and $\epsilon_t(z_t,a_t)$ as $\epsilon_t$,\footnote{Our techniques can even handle the case where $\epsilon_t$ depends adversarially not just on $z_t,a_t$, but on the transcript from the preceding rounds, so the simplified notation of $\epsilon_t$ is also meant to be suggestive of this.} As $x_t\in\calX$, the linear function $f:\calX\to\R$ is given by $f(x) = \iprod{w^*,x}$ for some unknown regressor $w^*\in\calX$, and for any round $t$ which is not corrupted, we have that \begin{equation}
%	y_t \triangleq y^*_t + \xi_t, \qquad y^*_t \triangleq \iprod{w^*,x_t} + \epsilon_t
%\end{equation}
%for $\xi_t$ sampled independently from some distribution $\calD$ over $\R$ and $\abs{\epsilon_t} \le \epsilon$. We will make the following assumption on $\calD$:

\paragraph{Model Assumptions.} We adopt the following standard normalization convention for the covariates and weight vector.

\begin{assumption}\label{assume:scaling}
	In the regression setting (Definition~\ref{defn:huberreg}), for any round $t$, $\norm{x_t} \le 1$ almost surely, $\norm{w^*} \le R$. Correspondingly, in the contextual bandits setting (Definition~\ref{def:lin_bandits}) we assume $\norm{z_{ta}} \le 1$ for all $a$ and $\norm{w^*} \le R$. 
\end{assumption}
To simplify the statement of bounds we assume in all statements that $\epsilon, \sigma = O(R)$. The last assumption can be removed at the cost of longer Theorem statements (e.g. writing $R + \sigma$ instead of $R$); this scaling captures the interesting setting for the bounds, because if $\epsilon \gg R$ then the responses are arbitrary, and if $\sigma \gg R$ then no interesting robustness guarantee is possible, as explained earlier --- the trivial guarantee of Ordinary Least Squares in this setting is already close to optimal.

\iffalse
We will \emph{not} make any kind of hypercontractive assumption on the covariates $x_t$. What we will do is (optionally) make a hypercontractive assumption on the (one-dimensional) noise distribution $\mathcal{D}$. First we state the assumption in the regression setting (Definition~\ref{defn:huberreg}).

\todo{should actually be fine to take $k \ge 1$ here, if we don't require $\sigma$ is the standard deviation. Also we can weaken the assumption to the tail bound version (this is called weak $L_p$ norm). Wrote it below}

\begin{assumption}\label{assume:hypernoise}
	Let $\sigma\triangleq \E[\xi\sim\calD]{\xi^2}^{1/2}$. For some absolute constant $c > 0$ and any real number $k \ge 2$, the noise distribution $\calD$ is mean zero and \emph{$(c,k)$-hypercontractive}, that is, \begin{equation}
		\E[\xi\sim\calD]{|\xi|^{\ell}}^{1/\ell} \le c\sqrt{\ell}\cdot\sigma \qquad \forall \ 2\le \ell \le k.
	\end{equation}
\end{assumption}
We emphasize in this assumption and our results that $k$ does not need to be integer. Furthermore, $k$ can be taken as small as $2$ in which case our result is still meaningful and in which case Assumption~\ref{assume:hypernoise} is a tautology; the reason for allowing larger values of $k$ is to get improved quantitative bounds.
\fi
%\todo{weak $L_q$ may be overkill, change it back to $L_q$? We need $L_q$ for some of the results anyway.}
%, i.e. it can be as taken as small as $2 + \varepsilon$ for $\varepsilon > 0$. This condition is sometimes called $L_{2 + \varepsilon}/L_2$ norm equivalence.
We now formally describe the (weak) assumptions on the noise under which we can perform our analysis. 
% We will consider three different possible assumptions of increasing generality on the noise.

% The most basic type of assumption we can make is that the noise is bounded:

% \begin{customassume}{2a}
%     We assume the noise $\xi\sim\calD$ is mean zero and satisfies $|\xi| \le \sigma$ almost surely.
% \end{customassume}

% We can also extend our results to the more general subgaussian case:

% \begin{customassume}{2b}
%     We assume the noise $\xi\sim\calD$ is mean zero and $\sigma^2$-subgaussian.
% \end{customassume}

% Our most general results consider hypercontractive noise in the following sense.

\begin{definition}[Weak $L_q$ Space]
Suppose $X$ is a real-valued random variable and $q \ge 1$. We define the \emph{weak $L_q$} or \emph{$L_{q,\infty}$ quasinorm} of $\xi$ to be
\[ \|X\|_{q,\infty} \triangleq \sup_{\lambda > 0} \lambda \cdot \left|\Pr{|X| > \lambda}^{1/q}\right| \]
so that $\Pr{|X| > \lambda} \le \|X\|_{q,\infty}^q/\lambda^q$. When $q = \infty$, we define $\|X\|_{\infty,\infty} = \inf \{ \lambda > 0 : \Pr{|X| \ge \lambda} = 0 \}$ to be the same as the $L_{\infty}$ norm.
We say that $X$ is in weak $L_q$ or $L_{q,\infty}$ space if $\|X\|_{q,\infty} < \infty$.
\end{definition}
From Markov's inequality, one has that $\Pr{|X| > \lambda} \le \E{|X|^q}/\lambda^q$ which shows that $\|X\|_{q,\infty} \le \|X\|_q$.
\begin{customassume}{2}\label{assume:lq-noise}
We assume the noise $\xi \sim \mathcal{D}$ is mean zero and that for some $q > 1$,
\[ \sigma_q \triangleq \|\xi\|_{q,\infty} < \infty.  \]
\end{customassume}

\iffalse
We will make the analogous assumption on the noise $\xi_t$ in Definition~\ref{defn:huber_bandits}:

\begin{assumption}\label{assume:hypernoise_bandits}
	Let $\xi_t$ be the random variable defined in Definition~\ref{defn:huber_bandits}. For some absolute constant $c > 0$, any real number $k > 0$, and any round $t$ and context $z\in\calX$, conditioned on $z_t = z$, the distribution over $\xi_t$ is $(c,k)$-hypercontractive.
\end{assumption}
\fi
\subsection{Technical Preliminaries}
Here we collect miscellaneous technical facts that will be useful in later sections. Throughout this paper we use standard notation for inequalities up to constants; for example, $a \lesssim b$ and $a = O(b)$ both denote an inequality true up to an absolute constant, and occasionally we use $C > 0$ to denote a universal constant which can change from line to line. Given a matrix $M$, we let $\norm{M}$ denote the operator norm of $M$. Given a positive semidefinite matrix $\Sig$, we define the \emph{Mahalanobis} norm by
\begin{equation}
\|x\|_{\Sig}^2 := \|\Sig^{1/2} x\|^2 = \langle x, \Sig x \rangle.
\end{equation}

\paragraph{Concentration of measure.} We use some concentration inequalities which we state here. We use standard martingale terminology, see e.g. \cite{durrett2019probability}; in particular, we say that a sequence of random variables $X_1,\ldots,X_t$ adapted to a filtration $\mathcal{F}_t$ form a \emph{martingale difference sequence} if $\E{X_t | \mathcal{F}_{t - 1}} = 0$ for all $t$. 
We say a mean-zero random variable $X$ is $\sigma^2$-subgaussian if $\log \E{e^{\lambda X}} \le \lambda^2 \sigma^2/2$ for all $\lambda \in \mathbb{R}$; recall that if $|X| \le K$ then $X$ is $O(K^2)$-subgaussian \cite{vershynin2018high}.
%\todo{change vector-valued azuma to use subgaussian steps based on Kallenberg, use this to shave log factor for subgaussian.} %done
\begin{fact}[Azuma-Hoeffding inequality] \label{fact:hoeffding}
Suppose that $X_1,\ldots,X_n$ is a martingale difference sequence and $|X_i| \le M_i$ almost surely. Then
    \begin{equation}
		\Pr*{\frac{1}{n}\sum^n_{i=1} X_i \ge t} \le \exp\left(-\Omega\left(\frac{nt^2}{\frac{1}{n}\sum_i M_i^2}\right)\right)
	\end{equation}
\end{fact}
\begin{fact}[Bernstein's inequality]\label{fact:bernstein}
	For $X_1,...,X_n$ independent and mean-zero, if $\abs{X_i} \le M$ for all $i$, then for all $t > 0$, \begin{equation}
		\Pr*{\frac{1}{n}\sum^n_{i=1}X_i \ge t} \le \exp\left(-\Omega\left(\frac{nt^2}{\frac{1}{n}\sum\E{X_i^2} + Mt}\right)\right)
	\end{equation}
\end{fact}
We will use the following general version of the Azuma-Hoeffding inequality, which applies to martingales in Euclidean space of arbitrary dimension with subgaussian step sizes. (Note: this result is false if we consider martingales with steps that are general subgaussian vectors, which can make steps of size $\sqrt{d}$ in dimension $d$.) This result follows from the same proof as Equation 5.18 in \cite{kallenberg1991some}, with some small differences: there they consider bounded variation processes instead of discrete-time martingales. In the bounded step size case, optimal constants were  obtained in \cite{pinelis1994optimum}. For completeness, we prove Theorem~\ref{thm:azuma-vector} in the Appendix. 
\begin{theorem}[Subgaussian-step vector Azuma-Hoeffding, cf. Equation 5.18 in \cite{kallenberg1991some}]\label{thm:azuma-vector}
Suppose that $X_1,\ldots,X_n$ are random vectors in Euclidean space with $\|X_t\| \le 1$ almost surely for all $t$, and $\xi_1,\ldots,\xi_n$ are random variables such that almost surely, the law of $\xi_t$ conditional on $X_1,\ldots,X_{t},\xi_1,\ldots,\xi_{t - 1}$ is mean-zero and $\sigma^2$-subgaussian. Then
    \begin{equation}
		\Pr*{\left\|\frac{1}{n} \sum^n_{i=1} \xi_i X_i\right\| \ge u} \le 2\exp\left(-\Omega\left(\frac{nu^2}{\sigma^2}\right)\right).
	\end{equation}
\end{theorem}
\begin{comment}
\begin{theorem}[Vector Bernstein Inequality, Corollary 4.1 of \cite{minsker2017some}]\label{thm:bernstein-vector}
Suppose that $X_1,\ldots,X_n$ are independent mean-zero vectors with $\|X_i\| \le 1$ almost surely. Let $\sigma_n^2 = \sum_{i = 1}^n \E{\|X_i\|^2}$, then for all $t \ge \frac{1}{6}(1 + \sqrt{1 + 36\sigma_n^2})$ we have
\[ \Pr{\|\sum_{i = 1}^n X_i\| > t} \le 28 \exp\left(\frac{-t^2}{\sigma_n^2 + t/3}\right)\]
\end{theorem}
\end{comment}

\paragraph{Matrix concentration.}
A key ingredient in our argument is concentration for matrix martingales. See \cite{tropp2012user,tropp2011user} for background on matrix concentration; for infinite dimensional settings we use a version of matrix concentration which depends on effective dimension \cite{hsu2012tail,minsker2017some}.  To briefly recall, a matrix martingale $\vY_1,\ldots,\vY_n$ adapted to a filtration $\mathcal{F}_t$ with difference sequence $\vX_t$ is an $\mathcal{F}_t$-adapted process satisfying $\vY_t = \sum_{s = 1}^t \vX_s$ and $\E{\vX_t | \mathcal{F}_{t - 1}} = 0$. We also recall that for a function $f : \mathbb{R} \to \mathbb{R}$ and symmetric matrix $M$ with eigendecomposition $M = \sum_i \lambda_i \rho_i \rho_i^T$, the notation $f(M)$ corresponds to applying $f$ to the spectrum, i.e. $f(M) = \sum_i f(\lambda_i) \rho_i \rho_i^T$.
\begin{comment}
\begin{theorem}[Matrix Freedman Inequality, Corollary 4.2 of \cite{tropp2011user}]\label{thm:matrix-freedman}
Suppose $\vY_1,...,\vY_n\in\R^{d\times d}$ is a symmetric matrix martingale adapted to filtration $\mathcal{F}_t$, whose associated difference sequence $\brc{\vX_t}$ satisfies $\norm{\vX_t} \le 1$ almost surely for all $t$. Let $\vW = \sum_t \E{\vX_t^2 | \mathcal{F}_{t- 1}}$, then
\[ \Pr{\|\vY_n\| \ge t \text{ and } \|\vW\| \le \sigma_n^2} \le 2d \cdot \exp\left(\frac{-t^2/2}{\sigma_n^2 + t/3}\right). \]
\end{theorem}
\begin{corollary}\label{corr:matrix-bernstein}
In the same setting as Theorem~\ref{thm:matrix-freedman}, suppose that for some $\sigma \le 1$, $\E{\vX_t^2 | \mathcal{F}_{t - 1}} \preceq \sigma^2$ almost surely. Then
\[ \Pr{\|(1/n) \cdot \vY_n\| \ge u} \le 2d \cdot \exp\left(-\frac{n u^2/2}{\sigma^2 + u/3}\right) \]
\end{corollary}
\begin{proof}
Apply Theorem~\ref{thm:matrix-freedman} with $t = nu$ and $\sigma_n^2 = n \sigma^2$.
\end{proof}
\begin{remark}[Effective Dimension]
There is a version of Matrix Freedman's inequality which depends on a standard measure of effective dimension, strengthening the above results --- see \cite{minsker2017some}. Using this result, it's straightforward to generalize our results to depend on effective dimension as well.
\end{remark}
\end{comment}
\begin{theorem}[Matrix Freedman Inequality, \cite{minsker2017some}]\label{thm:matrix-freedman-eff}
Suppose $\vY_1,...,\vY_n\in\R^{d\times d}$ is a symmetric matrix martingale adapted to filtration $\mathcal{F}_t$, whose associated difference sequence $\brc{\vX_t}$ satisfies $\norm{\vX_t} \le 1$ almost surely for all $t$. Let $\vW = \sum_t \E{\vX_t^2 | \mathcal{F}_{t- 1}}$, then for any $t \ge \frac{1}{6}(1 + \sqrt{1 + 36 \sigma_n^2})$ 
\[ \Pr{\|\vY_n\| \ge t \text{ and } \|\vW\| \le \sigma_n^2} \le 50 d_1(t) \cdot \exp\left(\frac{-t^2/2}{\sigma_n^2 + t/3}\right) \]
where
\[ d_1(t) = \Tr f(t \E \vW/\sigma_n^2) \]
and $f(x) = \min(1,x)$.
\end{theorem}
\begin{corollary}\label{corr:matrix-bernstein-eff}
In the same setting as Theorem~\ref{thm:matrix-freedman-eff}, suppose that for some $\sigma \le 1$, $\E{\vX_t^2 | \mathcal{F}_{t - 1}} \preceq \sigma^2$ almost surely. Then for any $u \ge 1/18n + \sigma \sqrt{1/n}$
\[ \Pr{\|(1/n) \cdot \vY_n\| \ge u} \le 50 d_2(u) \cdot \exp\left(\frac{-n u^2/2}{\sigma^2 + u/3}\right) \]
where
\[ d_2(u) = \Tr f(u \E \vW/\sigma^2) \]
and $f(x) = \min(1,x)$ as in Theorem~\ref{thm:matrix-freedman-eff}.
\end{corollary}
\begin{proof}
Apply Theorem~\ref{thm:matrix-freedman-eff} with $t = nu$ and $\sigma_n^2 = n \sigma^2$, noting that $d_2(u) = d_1(nu)$; in this statement, we only strengthened the assumed lower bound on $t$. %\todo{In the adversarial contexts setting, $\E \vW$ cannot really be understood except via norm constraint and dimension. State a corollary with that version?}
\end{proof}
\paragraph{Truncation Lemma.} In our algorithm and analysis, we handle heavy-tailed noise using a truncation argument; this somewhat parallels the use of truncation arguments in large deviation theory, see e.g. \cite{fuk1971probability}.
The following Lemma shows that random variables with tail bounds behave reasonably under truncation, in the sense that their means do not move drastically. 
\begin{lemma}\label{lem:truncation}
Suppose that $X$ is a mean-zero random variable and $\sigma_q \triangleq \|X\|_{q,\infty} < \infty$.
Then for any $s > 0$,
\[ |\E{X \bone{|X| < s}}| \le  \frac{q}{q - 1} \cdot \frac{\sigma_q^q}{s^{k - 1}}. \]
\end{lemma}
\begin{proof}
We know
\[ 0 = \E{X} = \E{X \bone{|X| < s}} + \E{X \bone{|X| \ge s}} \]
so using the identity $\E{Z} = \int_0^{\infty} \Pr{Z > y} dy$ for nonnegative random variable $Z$ (Lemma 1.2.1 of \cite{vershynin2018high}), we have
\begin{align*}
|\E{X \bone{|X| < s}}| 
= |\E{X \bone{|X| \ge s}}|
&\le \E{|X| \bone{|X| \ge s)}} \\
&= \int_0^{\infty} \Pr{|X| \bone{|X| \ge s} > y} dy \\
&= s \Pr{|X| \ge s} + \int_s^{\infty} \Pr{|X| > y} dy \\
&\le \frac{\sigma_q^q}{s^{q - 1}} + \int_s^{\infty} \frac{\sigma_q^q}{y^k} dy \\
&= \sigma_q^q \left(\frac{1}{s^{q - 1}} + \frac{1}{(q - 1) s^{q - 1}}\right) = \frac{q}{q - 1} \cdot \frac{\sigma_q^q}{s^{k - 1}}
\end{align*}
where in the last inequality, we used the definition of $L_{q,\infty}$.
\end{proof}

%% file: no_sos.tex
\section{Alternating Minimization for Offline Regression}
\label{sec:nosos}
In this section, we prove our main results for regression in the usual offline setting. After giving some setup and stating the main offline result in Section~\ref{sec:setup}, in Section~\ref{subsec:scram_def} we give a full description of our alternating minimization-based algorithm. In Section~\ref{subsec:optimization} we show that it converges to an approximate stationary point. In Section~\ref{subsec:allstationarygood} we show that this suffices to obtain our claimed error guarantees, and also give improved rates in the case of subgaussian noise. In Section~\ref{subsec:stochastic} we show how our fixed-design guarantee can yield strong results in the stochastic setting often considered in statistical learning. Finally, in Section~\ref{sec:geometric-median}, we give improved rates when the noise is in $L_q$ for $q \ge 2$ by boosting via a high-dimensional median.

\subsection{Setup and Main Result}
\label{sec:setup}

We will state and prove results for two closely related settings: (1) the usual setting in linear regression where the covariates $x_t$ are fixed arbitrary vectors (i.e. chosen obliviously), and (2) the model which is relevant for our online applications, where the covariates $x_t$ are generated sequentially and adaptively, so they can depend on e.g. the realization of the noise in previous rounds. 
%\todo{explain oblivious vs adaptive somewhere here, and also maybe in the introduction.}
The second setting is the proper offline version of the Huber-Contaminated Online Regression Problem as defined in Definition~\ref{defn:huberreg}.
%Compared to the fully online setup, the step where the algorithm creates a prediction is deferred to the very end. In other words, the algorithm gets to see all of the data points $x_1,\ldots,x_n$ and (possibly corrupted) labels $y_1,\ldots,y_n$ before it has to pick predictions $\hat{y}_1,\ldots,\hat{y}_n$. We use the notation $n$ for the number of data points instead of $T$ to emphasize that we are in an offline setting. %Since we are following the generative process from Definition~\ref{defn:huberreg}, the adversary still gets to choose the covariate $x_t$ adaptively based off of the noise and coin flips at previous rounds. In Remark~\ref{rmk:alternative-setup} we describe a slightly more conventional, non-adaptive setup where our result also holds and with an even more powerful adversary. The key technical element which is shared by both settings is the martingale structure of the covariance matrix of uncorrupted data (see Lemma~\ref{lem:regularity}).

%In this setting, we give a robust algorithm for the following fixed-design regression problem: \TODO{undo this edit.}
%Let $x_1,...,x_n$ be the set of covariates seen after the first $n$ rounds. 
We briefly recall some of the relevant notation.
Let $a^*_t$ be the indicator for whether round $t$ was uncorrupted, i.e. $a^*_t = 1$ when the round is \emph{not} corrupted and this occurs with probability $1 - \eta$.
Recall from \eqref{eq:linreg} that for every $t\in[n]$ corresponding to a round which is not corrupted, we observe $y_t$ given by \begin{equation}
	y_t = y^*_t + \xi_t, \qquad y^*_t = \iprod{w^*,x_t} + \epsilon_t
\end{equation} where $w^*$ is the true regressor and $\|w^*\| \le R$, and $\xi_t$ is independently sampled from the noise distribution $\calD$, and $\abs{\epsilon_t} \le \epsilon$ is the misspecification.
On the other hand, on corrupted rounds $y_t$ is chosen freely by the adversary.  
For convenience, define \begin{equation}
	\Sig_n \triangleq \frac{1}{n}\sum^n_{t=1}x_tx_t^{\top}
\end{equation}
Let $u^*$ be the best norm $R$ linear predictor of the uncorrupted and unnoised data, that is, \begin{equation}
	u^* \triangleq \arg\min_{u : \|u\| \le R} \frac{1}{n}\sum_t (y^*_t - \iprod{u,x_t})^2 \label{eq:ustardef}
\end{equation}
and let $\delta_t \triangleq y^*_t - \langle u^*, x_t \rangle$.
By definition of $u^*$, we have that 
\begin{equation}
	\frac{1}{n}\sum_t \delta^2_t \le \frac{1}{n}\sum_t \epsilon^2_t \le \epsilon^2 \label{eq:delta_bound}
\end{equation}    
almost surely; in fact, the conclusion of \eqref{eq:delta_bound} is all we need about the misspecification model and $w^*,\epsilon_t$ play no further role in this section.

Our goal will be to output $\wh{w}$ such that the MSE (Mean Squared Error) with respect to the true responses is as small as possible; since $u^*$ is the optimal linear predictor, this is the same (by the Pythagorean Theorem) as asking for $\|\wh{w} - u^*\|_{\Sig_n}^2$ is small. When there is no misspecification, this is equivalent to recovering $w^*$ up to small error in $\Sig_n$ norm. When there is misspecification, it is easy to see that if $\norm{\wh{w} - u^*}$ is small, then $\norm{\wh{w} - w^*}_{\Sig_n}$ is also small, up to an extra $O(\epsilon)$ term from the triangle inequality. The algorithm achieving our goal is {\sc SCRAM} (\textbf{S}pe\textbf{C}trally \textbf{R}egularized \textbf{A}lternating \textbf{M}inimization, defined in Algorithm~\ref{alg:SCRAM} and analyzed in Theorem~\ref{thm:nonsos}).% except close to the breakdown point $\eta = 1/2$ where we use a more sophisticated algorithm based on SOS.

% When there is no misspecification, this is the same as the definition of the usual MSE (Mean Squared Error) objective $\frac{1}{n} \sum_{t = 1}^n (y^*_t - \langle \tilde w, x \rangle)^2$. When there is misspecification, it differs by a $O(\epsilon^2)$ error, so we will ultimately get the same guarantees for both upper bounding $\|\td{w} - w^*\|_{\Sig_n}^2$ and the MSE. 

 In the following Theorem, the constants in the guarantee must deteriorate slightly as we approach the breakdown point $\eta = 1/3$ of this estimator, so we introduce a parameter $\beta$ which tracks the distance to $1/3$; as long as we are strictly bounded away from this point, $\beta$ is a $\Theta(1)$ quantity and can be ignored. As explained in Remark~\ref{rmk:breakdown-1/3}, this breakdown point is optimal for \textsc{SCRAM}, but in Section~\ref{sec:sos} we will give a more powerful version of this estimator based on sum-of-squares programming which achieves optimal breakdown point $1/2$.
 %We also note that in the low contamination regime, e.g. when $\eta = 0$, the bound below is dominated by the last term, which up to constants is the same (minimax optimal) error achieved by Ordinary Least Squares (see e.g. \cite{rigollethigh}).

\begin{theorem}\label{thm:nonsos}
Suppose that $\eta < 1/3$ is an upper bound on the contamination level, define
\begin{equation}\label{eqn:rho-defn} 
\beta \triangleq (1/3 - \eta)^2
%\beta = \begin{cases}
%1/4 & \text{if $\eta < 1/4$} \\
%\text{solution of $\eta = \frac{1}{2 + 2\sqrt{2\beta}}$} & \text{otherwise}
%\end{cases}
\end{equation}
and suppose for some $q \in (1,\infty],\sigma_q \ge 0$ and all $t$ that
\begin{equation}\label{eqn:q-noise}
    \|\xi_t\|_{q,\infty} \le \sigma_q
\end{equation} in the sense of Assumption~\ref{assume:lq-noise}.
Then provided
\[ \eta \cdot n \gtrsim \log(\min(n,d)/\delta), \]
we can take
$\alpha = \Theta\left(\sqrt{\frac{\eta \log(d/\delta)}{n}}\right)$ and $\overline{\eta} = \eta + \Theta(\eta\sqrt{\beta})$ such that the output $w$
of \textsc{SCRAM} with $poly(R/\sigma,\log(2/\delta),d,n)$ many steps satisfies for oblivious covariates the bound
\begin{align*} 
\beta^{1 + 1/q} \norm{u^* - w}_{\Sig_n}
    &\lesssim \frac{q}{q - 1} \eta^{1 - 1/q} \sigma_q + \eta^{1/2}\epsilon + \eta^{1/8} R^{1/2} (\epsilon + \frac{q}{q - 1} \eta^{1/2 - 1/q} \sigma_q)^{1/2}\sqrt[8]{\frac{\log(\min(n,d)/\delta)}{n}} \\
    &\quad+  \eta^{1/4} R\sqrt[4]{\frac{\log(\min(n,d)/\delta)}{n}} 
     + \eta^{-1/q} \min \left\{\sigma\sqrt{\frac{d + \log(2/\delta)}{n}}, (R\sigma)^{1/2}\sqrt[4]{\frac{\log(2/\delta)}{n}}\right\}
\end{align*}
with probability at least $1 - \delta$. In the more general case of adaptive covariates, it satisfies the bound
\begin{align*} 
\beta^{1 + 1/q} \norm{u^* - w}_{\Sig_n}
    &\lesssim \frac{q}{q - 1} \eta^{1 - 1/q} \sigma_q + \eta^{1/2}\epsilon + \eta^{1/8} R^{1/2} (\epsilon + \frac{q}{q - 1} \eta^{1/2 - 1/q} \sigma_q)^{1/2}\sqrt[8]{\frac{\log(\min(n,d)/\delta)}{n}} \\
    &\quad+  \eta^{1/4} R\sqrt[4]{\frac{\log(\min(n,d)/\delta)}{n}} 
     + \eta^{-1/q}(R\sigma)^{1/2}\sqrt[4]{\frac{\log(2/\delta)}{n}}
\end{align*}
i.e. the same bound except the last term was changed.
\end{theorem}
%done
%\todo{should we add a version with lower probability guarantees? i.e. improving the last term when $q \ge 2$ and $\delta$ is not too small by using Chebyshev/Khinchin. Seems like we should, because if we do this and then the geometric median trick, seems like we get a much better sample complexity.}
%\todo{should be able to optimize last term slightly when noise is subgaussian.} %done
\begin{remark}[Oracle Inequality Interpretation]
As mentioned before, the only guarantee on the misspecification we need is \eqref{eq:delta_bound}. This means that for any $\epsilon^2 \ge \frac{1}{n} \sum_t \delta_t^2$, i.e. any $\epsilon > 0$ such that \eqref{eq:delta_bound} is true almost surely, we have
\[ \frac{1}{n} \sum_t (y^*_t - \langle \hat{w}, x_t \rangle)^2 \lesssim \epsilon^2 + \|u^* - \hat{w}\|_{\Sig_n}^2 \]
which combined with Theorem~\ref{thm:nonsos} makes formal that $\langle \hat{w}, x_t \rangle$ is the best linear model of $y^*_t$ up to a small error term. This kind of bound for an estimator in the presence of misspecification is known as an \emph{oracle inequality} \cite{tsybakov2008introduction}, since $\hat{w}$ competes with the oracle fit $u^*$.
\end{remark}
\begin{remark}[Breakdown point and landscape]\label{rmk:breakdown-1/3}
The breakdown point of $\eta = 1/3$ is optimal for this estimator based on local search. This breakdown point is optimal even if $X \sim N(0,I)$ and the true generative model is a noiseless mixture of two linear regressions $w_1 \ne w_2$ with corresponding weights $1/3,2/3$, so we view $w_1$ as contamination. In this setting $\sigma_q = 0$ so an estimator achieving the optimal $O(\sigma_q)$ rate gets error $o(1)$. However, the pair $(w_1,a_1)$ is a bad local minimum where the weight vector $a_1$ keeps all of the data points from $w_1$ and keeps each point labeled by $w_2$ with probability $1/2$. In Section~\ref{sec:sos} we show how to overcome the bad landscape for $\eta \in [1/3,1/2)$, achieving the optimal $O(\sigma_q)$ error guarantee, using more powerful optimization tools (the Sum of Squares hierarchy) and a new analysis.
\end{remark}
\begin{remark}[Small $\eta$ regime]
If the true contamination level is very small, e.g. $\eta = 0$, then applying Theorem~\ref{thm:nonsos} with a larger value of $\eta$ will optimize the upper bound. %\todo{work out some examples.}
\end{remark}
% NOTE: not needed?
%Let $\alpha \triangleq \Theta(\sqrt{\log(d/\delta)/n})$. By taking $\bareta \triangleq \eta + \Theta([d\log(1/\delta)/n]^{1/k})$ and $n$ sufficiently large, we can ensure that $\alpha \le \bareta$. \TODO{move this to the right place}

When the noise is $L_q$ for $q \ge 2$, we show how to improve the last term on the right-hand side of Theorem~\ref{thm:nonsos} to avoid an $\eta^{-1/q}$ dependence in the last term on the right-hand side, see Theorem~\ref{thm:geometricmedian}.

\subsection{Algorithm Specification}
\label{subsec:scram_def}

The algorithm used in Theorem~\ref{thm:nonsos} is based upon finding first-order stationary points of the following nonconvex problem.
\begin{program}\label{program:biconvex}
	Define variables $w,a_1,\ldots,a_n$ and consider the optimization problem with parameters $\overline{\eta},\alpha,R \ge 0$ given by
	\begin{equation}
		\begin{aligned}
			& \min_w \underset{a_1,\ldots,a_n}{\text{min}} & & \frac{1}{n} \sum_{t = 1}^n a_t(y_t - \iprod{w,x_t})^2 \\ 
			& \text{s.t.} & & 0 \le a_t \le 1 \quad \forall t\in[n] \\
			& & & \sum_t a_t \ge (1 - \bareta - \alpha)n  \\
			& & & \frac{1}{n}\sum_t (1 - a_t) x_t x_t^{\top} \preceq \bareta \Sig_n + \alpha\cdot \Id \\
			& & & \|w\| \le R
		\end{aligned}
	\end{equation}
	where $\|w\|$ denotes the Euclidean norm of $w$. 
	%  \begin{align*}
	% 	0 \le a_t &\le 1 \quad \forall i \\
	% 	\sum^n_{t=1} a_t &\ge (1 - \bareta - \alpha)n  \\
	% 	\sum^n_{t=1} (1 - a_t) x_t x_t^{\top} &\preceq \bareta \Sigma + \alpha\cdot \Id.
	% \end{align*}
\end{program}
The overall objective
\[ L(w,a) := \frac{1}{n} \sum_t a_t(y_t - \iprod{w,x_t})^2 \]
is \emph{biconvex}, i.e. convex individually in the variables $a$ and the variables $w$, but not jointly convex. Since it is a nonconvex problem, we cannot guarantee to find the true global minimum of this optimization problem. One of the most common heuristics for biconvex problems is to perform alternating minimization, which will output an approximate first order stationary point.
%Instead, we will do a landscape analysis and show that any approximate first order stationary point satisfies the desired statistical guarantee.
%We first show how to efficiently obtain an approximate first order stationary point for this problem using an alternating minimization procedure. 
Fortunately, we prove in our setting that this suffices and all approximate first order stationary points satisfy the desired statistical guarantee. 
As one half of the alternating minimization procedure, we observe that minimizing $a$ for fixed $w$ is a simple SDP (semidefinite program):
%We will use the following basic SDP:
\begin{program}\label{program:basic_sdp}
	For fixed vector $w$, define variables $a_1,\ldots,a_n$ and define the optimization problem \textbf{SDP$_w$} with additional parameters $\overline{\eta},\alpha \ge 0$ given by
	\begin{equation}
		\begin{aligned}
			& \underset{a_1,\ldots,a_n}{\text{min}} & & \frac{1}{n} \sum_{t = 1}^n a_t(y_t - \iprod{w,x_t})^2 \\ 
			& \text{s.t.} & & 0 \le a_t \le 1 \quad \forall t\in[n] \\
			& & & \sum_t a_t \ge (1 - \bareta - \alpha)n  \\
			& & & \frac{1}{n} \sum_t (1 - a_t) x_t x_t^{\top} \preceq \bareta \Sig_n + \alpha\cdot \Id.
		\end{aligned}
	\end{equation}
    Note that this corresponds to Program~\ref{program:biconvex} for a fixed choice of $w$. %We also note that some specialized solvers apply to SDPs of this form, see e.g. \cite{jambulapati2020positive}.%% Already mentioned elsewhere
	%  \begin{align*}
	% 	0 \le a_t &\le 1 \quad \forall i \\
	% 	\sum^n_{t=1} a_t &\ge (1 - \bareta - \alpha)n  \\
	% 	\sum^n_{t=1} (1 - a_t) x_t x_t^{\top} &\preceq \bareta \Sigma + \alpha\cdot \Id.
	% \end{align*}
\end{program}
\begin{algorithm2e}
\DontPrintSemicolon
\caption{\textsc{SCRAM}($D,\epsilon_{\mathsf{OPT}}$)}
\label{alg:SCRAM}
	\KwIn{Dataset $D = \brc{(x_1,y_1),\ldots,(x_n,y_n)}$}
	\KwOut{Approximate first-order critical point of Program~\ref{program:biconvex} (see Lemma~\ref{lem:getfirstorder})}
	Let $w^{(1)} = 0$.\;
	\For{$s = 1$ to $\infty$} {
	    Let $a^{(s)}$ be the minimizer of Program~\ref{program:basic_sdp} with $w = w^{(s)}$.\;
	    Let $w^{(s + 1)}$ be the minimizer of $L(w,a^{(s)}) = \sum_t a^{(s)}_t (y_t - \langle w, x_t \rangle)^2$ over all $w$ with $\|w\| \le R$.\;
	    \If{$L(w^{(s + 1)},a^{(s)}) > L(w^{(s)},a^{(s)}) + \epsilon_{\mathsf{OPT}}$}{
	        Return $w^{(s)},a^{(s)}$.
	    }
	 }
\end{algorithm2e}

\subsection{Optimization Analysis}
\label{subsec:optimization}
For the analysis we need the following simple Taylor expansion inequality used to analyze gradient descent on smooth functions:
\begin{lemma}[Standard, see e.g. \cite{bubeck2014convex}]\label{lem:smoothness}
    Suppose that $f : \mathbb{R}^d \to \mathbb{R}$ is $L$-smooth in the sense
    that $\|\nabla^2 f\|_{OP} \le 2L$. Then
    \[ f(y) \le f(x) + \langle \nabla f(x), y - x \rangle + L\|y - x\|^2. \]
\end{lemma}
From this we get the following Descent Lemma on the ball:
\begin{lemma}\label{lem:descent}
Suppose that $f$ is $L$-smooth and $x,y$ are vectors in $\mathbb{R}^d$ such that $\|x\|,\|y\| \le R$ and $\langle \nabla f(x), x - y \rangle \ge \Delta > 0$. Then there exists a point $z$ which is a convex combination of $x,y$ such that
\[ f(z) \le f(x) - \frac{\Delta^2}{16 L R^2}.\]
\end{lemma}
\begin{proof}
    We consider points of the form $z_{\lambda} := (1 - \lambda) x + \lambda y$ which by convexity lie in the radius $R$ ball. Observe by Lemma~\ref{lem:smoothness} that
    \[ f(z_{\lambda}) \le f(x) - \lambda \Delta + 4LR^2\lambda^2\]
    since $\|x - x_{\lambda}\| \le \lambda \|x\| + \lambda \|y\| \le 2\lambda R$. The upper bound is optimized by $\lambda = \frac{\Delta}{8LR^2}$ and plugging in gives the result.
\end{proof}
\begin{lemma}\label{lem:getfirstorder}
	{\sc SCRAM} with $\epsilon_{\mathsf{OPT}} = \graderr^2/4R^2$  outputs vector $w$ and weights $a_1,\ldots,a_n$ satisfying the constraints of Program~\ref{program:biconvex} such that:
	\begin{enumerate}
	\item (Partial optimality) The variables $a$ are global minimizers of \textbf{SDP}$_w$ (Program~\ref{program:basic_sdp}).
	\item (First order stationarity)
	%\textbf{SDP}$_w$ (Program~\ref{program:basic_sdp}) for which 
	\begin{equation}
		\frac{1}{n}\sum_t a_t(y_t - \iprod{w,x_t})\iprod{x_t, v - w} \le \graderr \label{eq:firstorder}
	\end{equation} for all $v$ with $\|v\| \le R$.
	\end{enumerate}
	Furthermore, the expected number of iterations in the main loop is at most $O((R^2 + \sigma^2)R^2/\graderr^2)$.
\end{lemma}

\begin{proof}
	By definition $a^{(s)}$ is the minimizer of the \textbf{SDP}$_{w^{(s)}}$ so the first property is satisfied by construction. We now prove the second property. Observe that the objective $L(w,a^{(s)})$ is $1$-smooth in $w$ and
	\begin{equation}\label{eqn:gradient-calc}
	\nabla_w L(w,a^{(s)}) = -\frac{2}{n}\sum_t (y_t - \langle w, x_t \rangle) x_t. 
	\end{equation}
	Therefore by Lemma~\ref{lem:descent} and the fact that $w^{(s + 1)}$ is the optimizer for fixed $a^{(s)}$, we know that if there exists $v$ with $\langle \nabla_w L(w^{(s)},a^{(s)}), w^{(s)} - v \rangle \ge \Delta > 0$
	\[ L(w^{(s + 1)}, a^{(s)}) \le L(w^{(s)}, a^{(s)}) - \frac{\Delta^2}{16 R^2}. \]
	By the contrapositive, if the decrease in objective value when moving from $w^{(s)}$ to $w^{(s + 1)}$ is less than $\epsilon_{\mathsf{OPT}}$, then it implies that
	\[ \langle \nabla_w L(w^{(s)},a^{(s)}), w^{(s)} - v \rangle \le 4R\sqrt{\epsilon_{\mathsf{OPT}}} \]
	for all $v$ in the unit ball. Hence by \eqref{eqn:gradient-calc} taking $\epsilon_{\mathsf{OPT}} = \graderr^2/4R^2$ gives the stated guarantee.
	
	Finally, we bound the number of iterations needed. Every time the loop is repeated, the objective value $L(w,a)$ decreases by at least $\epsilon_{\mathsf{OPT}}$ and clearly $L(w,a) \ge 0$. Therefore the total number of iterations can be upper bounded by $L(0,a^{(1)})/\epsilon_{\mathsf{OPT}}$. By considering the (possibly suboptimal solution) $a_t = a^*_t$ to the first SDP, we see that the expected value of $L(0,a^{(1)})$ is at most $R^2 + \sigma^2$. Therefore the expected total number of iterations is at most $(R^2 + \sigma^2)/\epsilon_{\mathsf{OPT}}$.
	%In the first step, we see that by considering $a_t = a^*_t$ that the initial objective value $L(0,a^{(1)})$ is $O(R^2 + \sigma^2\log(2/\delta))$ with probability at least $1 - \delta$.  
\end{proof}
\iffalse %obsolete
\begin{remark}[Sum of Squares Alternative]\label{remark:sos}
Instead of searching for a first-order stationary point of the nonconvex formulation of Program~\ref{program:biconvex}, we could also have formulated a convex relaxation of the problem and solved that (to global optimality) instead. It turns out that the analysis proving that a first-order stationary point has good statistical guarantees (in the next subsection) is simple enough that it can be performed in the Sum of Squares degree-4 proof system \cite{Shor87,Parrilo00,Lasserre01}. Therefore, instead of working with Program~\ref{program:biconvex}, we could have solved a natural degree-4 SoS relaxation of Program~\ref{program:biconvex} and the proof below would show that the pseudoexpectation value of the weight vector $w$ would inherit all of the desired statistical guarantees. The main advantage of the alternating minimization approach is that it is significantly more computationally efficient. \todo{update discussion/remove this.}
\end{remark}
\fi
\subsection{All Stationary Points are Good}
\label{subsec:allstationarygood}

It remains to show why condition \eqref{eq:firstorder} implies the desired error guarantee. To establish the general guarantee of Theorem~\ref{thm:nonsos}, it's sufficient to reduce to the case where the noise $\xi_t$ is bounded, unless we care about the precise sample complexity. For this reason, we start with this setting (Section~\ref{subsubsec:bounded}), show how to reduce the $L_{q,\infty}$ setting of Theorem~\ref{thm:nonsos} to the bounded case, and then discuss how to tailor the analysis to get refined guarantees for subgaussian noise in Section~\ref{subsubsec:subgauss}. Later in Section~\ref{sec:geometric-median}, we give an improved version of Theorem~\ref{thm:nonsos} when the noise $\brc{\xi_t}$ is $L_q$ for $q \ge 2$, see Theorem~\ref{thm:geometricmedian}.

\subsubsection{Bounded Noise Analysis}
\label{subsubsec:bounded}

In the bounded case we establish the following result:
\begin{theorem}[\textsc{SCRAM} Guarantee with Bounded Noise]\label{thm:nonsos-bdd}
Suppose that $\eta < 1/3$, define $\beta$ as in \eqref{eqn:rho-defn}, and suppose for some $\sigma \ge 0$ that for all $t$,
\begin{equation}\label{eqn:bdd-noise}
    |\xi_t| \le \sigma
\end{equation}
almost surely. Then if $\eta = 0$ or
%\todo{check why we have this condition, it's odd when $\eta$ is close to zero. and definitely unneeded when $\eta = 0$.}
\begin{equation}\label{eqn:n-lb-bdd}
n \gtrsim \log(\min(n,d)/\delta)/\eta, 
\end{equation}
taking $\alpha = \Theta\left(\sqrt{\frac{\eta \log(\min(n,d)/\delta)}{n}}\right)$ and $\overline{\eta} = \eta$, the output $w$
of \textsc{SCRAM} with $poly(R/\sigma,\log(2/\delta),d,n)$ many steps satisfies for oblivious covariates the bound
\begin{align*} 
\beta \norm{u^* - w}_{\Sig_n}
    &\lesssim \eta \sigma + \eta^{1/2}\epsilon + \eta^{1/8} R^{1/2} (\eta^{1/2}\sigma  + \epsilon)^{1/2}\sqrt[8]{\frac{\log(\min(n,d)/\delta)}{n}} +  \eta^{1/4} R\sqrt[4]{\frac{\log(\min(n,d)/\delta)}{n}}  \\
    &\quad + \min \left\{\sigma\sqrt{\frac{d + \log(2/\delta)}{n}}, (R\sigma)^{1/2}\sqrt[4]{\frac{\log(2/\delta)}{n}}\right\}
\end{align*}
with probability at least $1 - \delta$. In the more general case of adaptive covariates, it satisfies the bound
\begin{align*} 
\beta \norm{u^* - w}_{\Sig_n}
    &\lesssim \eta \sigma + \eta^{1/2}\epsilon + \eta^{1/8} R^{1/2} (\eta^{1/2}\sigma  + \epsilon)^{1/2}\sqrt[8]{\frac{\log(\min(n,d)/\delta)}{n}} +  \eta^{1/4} R\sqrt[4]{\frac{\log(\min(n,d)/\delta)}{n}}  \\
    &\quad + (R\sigma)^{1/2}\sqrt[4]{\frac{\log(2/\delta)}{n}},
\end{align*}
i.e. the same bound except the second line was changed.
%, provided that $n = \Omega(R^2M^2(d + \log(2/\delta)))$. 
\end{theorem}
\begin{example}[Lower bound when $\sigma = \epsilon = 0$]\label{example:extra}
Consider the special case with $\sigma = \epsilon = 0$ with oblivious contexts. Observe that when $\sigma = \epsilon = 0$ the only nonzero term in the upper bound is $\eta^{1/4} R \sqrt[4]{\frac{\log(n/\delta)}{n}}$.
Now consider the setting  where the clean regression model with $d = n$ is given by $Y^* = w^* \in \mathbb{R}^n$, and we consider the $\eta$-contaminated version of this model with $\delta=1/n$ and $\eta = \log(n/\delta)/n$,. 
The number of contaminated coordinates of $Y$ will be close to $\eta n = \Theta( \log(n/\delta))$, and for each of those coordinates $i$, the algorithm observes no information about $w^*_i$. Considering letting $w^* = \pm R e_j$ for an arbitrary $j \in [n]$, then the probability coordinate $j$ is missed is $\Theta(\eta) = \Theta( \log(n/\delta)/n) = \omega(\delta)$ and on this event the algorithm must pay a cost in squared loss $\|u^* - w\|_{\Sig_n}^2$ of $R^2/n = \frac{1}{\log(n/\delta)} R^2 \eta^{1/2} \sqrt{\frac{\log(n/\delta)}{n}}$, matching the upper bound up to the log factor. 

This example also shows the necessity of \eqref{eqn:n-lb-bdd} when $\eta \ne 0$: without this lower bound, we could take $\eta = 1/n^{1 + \gamma}$ for some $\gamma > 0$, $\delta = 0.1/n^{1 + \gamma}$ and we would conclude by the same argument that $R^2/n \lesssim R^2 \eta^{1/2} \sqrt{\log(n/\delta)/n} = \Theta(R^2 \sqrt{\log(n)}/n^{1 + \gamma/2})$ which is false. 
%\todo{it doesn't show the exact relationship between $\eta$ and $1/n$ is correct, is this possible? could we get improved bounds for oblivious case using optimistic rates result, at least in the stochastic setting? Might want to tweak the constraint somehow in that case.}
\end{example}
Given this result, Theorem~\ref{thm:nonsos} follows by slightly increasing the value of $\eta$, so that heavy tail events are counted as contamination; we have to be slightly careful when the noise is asymmetric, because truncating can also induce also a small amount of misspecification, but it does not affect the final bound. 
\begin{proof}[Proof of Theorem~\ref{thm:nonsos}]
We prove this Theorem by reducing to Theorem~\ref{thm:nonsos-bdd}. 
\begin{comment} First we make a computation related to the breakdown point: observe that if $\eta = \frac{1}{2 + 2\sqrt{2\beta }}$ and $\overline{\eta} \triangleq \eta + \eta\sqrt{\beta}$, then if we define $\overline{\beta}$ by $\overline{\eta} = \frac{1}{2 + 2\sqrt{2\overline \beta}}$, we have that
\[ 1 = (2 + 2\sqrt{2 \overline \beta})\frac{1 + \sqrt{\beta}}{2 + 2\sqrt{2\beta}} = \frac{(1 + \sqrt{2\overline{\beta}})(1 + \sqrt{\beta})}{1 + \sqrt{2\beta}} \]
i.e.
\[ \frac{(\sqrt{2} - 1)\sqrt{\beta}}{1 + \sqrt{\beta}} = \sqrt{2\overline{\beta}}\]
which shows that $\overline{\beta} = \Theta(\beta)$.
\end{comment}
We consider the effect of treating all clean responses with $|\xi| \ge M \sigma_q$ for some $M \ge 1$ as contamination, increasing the effective $\eta$ to $\overline{\eta} = \eta + \sqrt{\beta} \eta/2$ and making the noise bounded. 
%\todo{here is sketch for now, need to include breakpoint details.} 
Recall from the definition that% Lemma~\ref{lem:hypercontraction} that
\[ \Pr{|\xi| \ge M \sigma_q} \le \frac{1}{M^q} \]
so by solving $\beta^{1/2} \eta/2 \ge 1/M^q$ we find that setting 
\[ M = (\beta^{1/2} \eta)^{-1/q} \]
ensures the total contamination level is at most $\eta + \sqrt{\beta} \eta/2 = \overline{\eta}$ as desired. Applying %the second part of Lemma~\ref{lem:hypercontraction} 
Lemma~\ref{lem:truncation}
shows that this reduction this causes an additional misspecification cost of
\[ \frac{q \sigma_q}{(q - 1)M^{q - 1}}  =  \frac{q}{q - 1}\sigma (\beta^{1/2} \eta)^{1 - 1/q}. \]
Now plugging into the conclusion of Theorem~\ref{thm:nonsos-bdd} with $\sigma_{\infty} = M \sigma$, $\overline{\eta}$, and $\epsilon' = \epsilon + \Theta(\frac{q}{q - 1}\sigma (\beta^{1/2} \eta)^{1 - 1/q})$ gives, as long as
\[ \eta \cdot n \gtrsim \log(\min(n,d)/\delta) \]
a bound of the form
\begin{align*} 
\beta \norm{u^* - w}_{\Sig_n}
    &\lesssim \eta M \sigma + \eta^{1/2}\epsilon' + \eta^{1/8} R^{1/2} (\eta^{1/2}\sigma M + \epsilon')^{1/2}\sqrt[8]{\frac{\log(\min(n,d)/\delta)}{n}} +  \eta^{1/4} R\sqrt[4]{\frac{\log(\min(n,d)/\delta)}{n}}  \\
    &\quad + M\min \left\{\sigma\sqrt{\frac{d + \log(2/\delta)}{n}}, (R\sigma)^{1/2}\sqrt[4]{\frac{\log(2/\delta)}{n}}\right\}
\end{align*}
where the first term is bounded as
\[ \eta M \sigma \lesssim \beta^{-1/2q} \eta^{1 - 1/q} \sigma_q \]
the second term is bounded as
\[  \eta^{1/2} \epsilon' \lesssim \eta^{1/2} \epsilon + \frac{q}{q - 1} \beta^{- 1/2q} \eta^{3/2 - 1/q} \sigma_q  \]
and the third term is bounded by observing
\[ \eta^{1/2} \sigma M + \epsilon' \lesssim   \beta^{-1/2q} \eta^{1/2 - 1/q} \sigma_q + \epsilon + \frac{q}{q - 1} \beta^{- 1/2q} \eta^{1 - 1/q} \sigma_q \lesssim \epsilon + \frac{q}{q - 1}\beta^{- 1/2q} \eta^{1/2 - 1/q} \sigma_q \]
and the last term is bounded by plugging in $M$.
Combining these bounds and upper bounding gives 
\begin{align*} 
\beta^{1 + 1/q} \norm{u^* - w}_{\Sig_n}
    &\lesssim \frac{q}{q - 1} \eta^{1 - 1/q} \sigma_q + \eta^{1/2}\epsilon + \eta^{1/8} R^{1/2} (\epsilon + \frac{q}{q - 1} \eta^{1/2 - 1/q} \sigma_q)^{1/2}\sqrt[8]{\frac{\log(\min(n,d)/\delta)}{n}} \\
    &\quad+  \eta^{1/4} R\sqrt[4]{\frac{\log(\min(n,d)/\delta)}{n}} 
     + \eta^{-1/q} \min \left\{\sigma\sqrt{\frac{d + \log(2/\delta)}{n}}, (R\sigma)^{1/2}\sqrt[4]{\frac{\log(2/\delta)}{n}}\right\}
\end{align*}
which is the result in the oblivious setting. Dropping one of the terms in the min gives the adaptive setting result.
%TODO (misspecification bound is used when applying triangle inequality). Follows from Theorem~\ref{thm:nonsos-bdd}.
\end{proof}
We will now prove Theorem~\ref{thm:nonsos-bdd}, so for the remainder of this section we proceed under assumption \eqref{eqn:bdd-noise}. In Lemma~\ref{lem:regularity} we establish deterministic regularity conditions which hold with high probability. First, in Lemma~\ref{lem:sphere-max} we prove a version of a standard maximal inequality used in the analysis of Ordinary Least Squares (see e.g. \cite{rigollethigh}), which shows that the norm of the noise vector shrinks when projecting onto a lower-dimensional subspace.
\begin{comment} We introduce some notation used throughout the analysis. 
For 
\begin{equation}\label{eqn:gamma}
\bareta \triangleq \eta + \gamma^2
\end{equation}
with $\gamma \ge 0$ a parameter to be tuned, it will be convenient to define 
\begin{equation}\label{eqn:aprime}
a'_t \triangleq a^*_t \bone{\xi^2_t \le s}, 
\qquad s \triangleq \frac{\sigma^2}{2\gamma^2}
\end{equation}
to handle the rounds $t$ for which $\xi_t$ lives in the tails of $\calD$. It is helpful for the reader to think of the case where $\eta$ is a small fixed constant, in which case we can take $\bareta = 1.01 \eta$, $\gamma^2 = 0.01 \eta$, and $s = \Theta(\sigma^2/\eta)$.
\end{comment}
\begin{lemma}\label{lem:sphere-max}
Suppose that $\xi_1,\ldots,\xi_n$ is a martingale difference sequence with $|\xi_t| \le \sigma$ almost surely for all $t$. Suppose that $V$ is a subspace of dimension $d$, $P_V : n \times n$ is the projection map onto $V$, and $\xi = (\xi_1,\ldots,\xi_n)$. Then
\[ \|P_V \xi\| \lesssim \sigma \sqrt{d + \log(2/\delta)} \]
with probability at least $1 - \delta$. % resolved: \todo{do we need something like this for a random subspace $V$, annoying. May need a log det type potential. Curently the analysis can only succeed with oblivious contexts. }
\end{lemma}
\begin{proof}
%The coordinates of $\xi$ are $\sigma$-subgaussian, so $\xi$ is an $\sigma$-subgaussian random vector. 
For $v \in V$ with $\|v\| = 1$, define $Z_v = \langle v, \xi \rangle = \sum_i v_i \xi_i$ which is a martingale. Since $|v_i \xi_i|  \le \sigma M |v_i|$ almost surely and 
$\sum_i v_i^2 = 1$, it follows from Azuma-Hoeffding inequality (Fact~\ref{fact:hoeffding}) that
	\begin{equation}
		\Pr{\abs{Z_v} \ge t} 
		\le \exp\left(-\frac{Ct^2}{\sigma^2}\right)
	\end{equation} 
	By a well-known chaining argument over the sphere (Exercise 4.4.2 of \cite{vershynin2018high}), we can upper bound
	\[ \|P_V \xi\| = \max_{\|v\| = 1} Z_v \le 2\max_{v \in \calN} Z_v \]
	where $\calN$ is a $1/2$-net of the unit sphere in $V$. 
	Standard covering number bounds (e.g. Corollary 4.2.13 of \cite{vershynin2018high}) let us take $|\calN| \le 6^d$. Therefore by the union bound
	%The covering number for this is upper bounded by the covering number for the Euclidean sphere of radius $R$, so standard covering number bounds (e.g. Corollary 4.2.13 of \cite{vershynin2018high}) let us take $|\calN| \le O(R)^d$. Therefore
	\[ \Pr*{\max_{\|v\| = 1} Z_v \ge t} \le 6^d \exp\left(-\frac{Ct^2}{\sigma^2}\right).  \]
	Taking $t = \Theta(\sigma \sqrt{(d + \log(2/\delta)})$ gives the result.
\end{proof}

\begin{lemma}\label{lem:regularity}
%\todo{part 2 of this lemma only for oblivious contexts. See below lemma for adaptive case.}
	For any $\alpha \in (0,\eta)$, 
	suppose 
	\begin{equation}
		n \gtrsim \frac{\eta \log(\min(n,d)/\delta)}{\alpha^2}
		%\td{\Omega}\left(\frac{d}{(c^2k\eta)^{k}}\log\left(\sigma/\delta\eta\right)\right).
		\label{eq:nbound}
	\end{equation}
	For any sequence of $x_1,...,x_n$ chosen during the process in Definition~\ref{defn:huberreg}, we have that with probability at least $1 - \delta$ over the randomness of the $\Ber(\eta)$ coins generating $a^*_1,...,a^*_n$, the following event holds. Let $\Sig' \triangleq \frac{1}{n}\sum_t a^*_t x_t x_t^{\top}$. Then:
	\begin{enumerate}
		\item $\frac{1}{n} \sum_{t = 1}^n a^*_t \ge 1 - \eta - \alpha$. \label{item:manygood}% $\sum a'_t \ge 1 - 3\eta/2$.
		\item $\left|\frac{1}{n}\sum^n_{t=1}a^*_t \xi_t \iprod{x_t,v}\right| \le \sigma \lambda \norm{v}_{\Sig'} + \sigma \lambda'\norm{v}$ for all $v$ where:
		\begin{enumerate}
		    \item In the special case of obliviously chosen covariates $x_t$: $\lambda \triangleq \Theta\left(\sqrt{\frac{d + \log(2/\delta)}{n}}\right)$ and $\lambda' \triangleq 0$.
		    \item In the general case of adaptive chosen covariates $x_t$: $\lambda \triangleq 0$ and $\lambda' \triangleq \Theta\left( \sqrt{\frac{\log(2/\delta)}{n}}\right)$
		\end{enumerate}
		 \label{item:vanish}
		\item $\Sig' \succeq (1 - \eta)\Sig_n - \alpha\cdot \Id$. \label{item:sigconc}
	\end{enumerate}
\end{lemma}

\begin{proof}
    We start with part~\ref{item:manygood}.
	We have $\E{\frac{1}{n} \sum_{t = 1}^n a^*_t} = 1 - \eta$
	and using that the variance of $Ber(p)$ is $p(1 - p)$ we have
	$\Var{a^*_t} \le \eta$.
	Then by Bernstein's inequality (Fact~\ref{fact:bernstein})
	we know that
	\[ \Pr{\frac{1}{n} \sum_{t = 1}^n a^*_t \ge 1 - \eta - \alpha} \le \exp\left(-\frac{Cn \alpha^2}{\frac{1}{n} \sum \Var{a^*_t} + \alpha}\right) \le \exp\left(-\frac{C n \alpha^2}{\eta + \alpha}\right)  \]
	so we find $\frac{1}{n} \sum_{t = 1}^n a'_t \ge 1 - \eta - \alpha$ with probability $1 - \delta$, provided $n = \Omega\left(\frac{\eta}{\alpha^2} \log(1/\delta)\right)$.

	For part~\ref{item:vanish} (a), let $T\subseteq[n]$ denote the set of indices $t$ for which $a^*_t = 1$; we now treat $x$ and $T$ as fixed and consider only $\xi$. %\todo{this requires oblivious $x_t$, see below.}%\todo{this requires oblivious $x_t$. However, the high-dimensional version probably doesn't require it. Also, a version of the argument from before (which used a cover of the ball of radius $R$) also might be fine? In that case, make two versions of this Lemma.} Observe that
	\[ \frac{1}{n}\sum^n_{t=1}a^*_t \xi_t \iprod{x_t,v} = \frac{1}{n} \langle (X')^T \xi, v \rangle = \frac{1}{n} \langle P_{V} \xi, (X') v \rangle \le \frac{1}{\sqrt{n}} \|P_V \xi\| \|v\|_{\Sig'} \]
	where $X' : n \times d$ has rows $a^*_1 x_1,\ldots,a^*_n x_n$, $P_V$ is the projection onto subspace $V$ and $V$ is the column span of $X'$, the last step applies Cauchy-Schwarz and the definition of $\Sig'$. Finally, the result follows by bounding $P_V \xi$ using Lemma~\ref{lem:sphere-max}.
	
	For part 2(b), observe by Cauchy-Schwarz
\[ \frac{1}{n} \sum_{t = 1}^n a^*_t \xi_t \langle x_t, v\rangle = \left\langle \frac{1}{n} \sum_{t = 1}^n a^*_t \xi_t x_t, v \right\rangle \le \left\|\frac{1}{n} \sum_{t = 1}^n a^*_t \xi_t x_t\right\| \|v\| \]
and the sum inside the absolute value is a vector-valued martingale with step size at most $\sigma$, so the result follows from Theorem~\ref{thm:azuma-vector}.

	We now show part~\ref{item:sigconc}. We can apply the matrix Freedman inequality in the form of Corollary~\ref{corr:matrix-bernstein-eff} to the matrix martingale difference sequence \begin{equation}
		(a^*_1 - (1 - \eta))\cdot x_1x_1^{\top}, (a^*_2 - (1 - \eta))\cdot x_2x_2^{\top}, \ldots, (a^*_t - (1 - \eta)) \cdot x_tx_t^{\top}, \label{eq:sequence_mats}
	\end{equation}
	which satisfies $\E{(a^*_t - (1 - \eta))^2 (x_t x_t^T)^2 | \mathcal{F}_{t - 1}} \preceq \eta$ 
	to get
	\begin{equation}
	\Pr*{\norm*{\frac{1}{n}\sum^n_{t=1}a^*_t \cdot x_tx_t^{\top} - \frac{(1 - \eta)}{n}\sum^n_{t=1} x_tx_t^{\top}} \ge \alpha} \le d_2(\alpha) \exp\left( \frac{-C n \alpha^2}{\eta + \alpha}\right)
	\end{equation} 
	where the probability is over the randomness of the martingale, and from Corollary~\ref{corr:matrix-bernstein-eff} we recall $f(x) = \min(1,x)$ hence
	\[ d_2(\alpha) = \Tr f(\alpha \sum_t \E{(a^*_t - (1 - \eta))^2 (x_t x_t^T)^2}/\eta) \le \Tr f(\alpha \sum_t \E{x_t x_t^T}) \le \min \{ d, \alpha n \}.  \]
	%For any $t\in[n]$, we have by \eqref{eq:residual_large_prob} that $\E{a'_t} = (1 - \eta)\Pr[\xi\sim\calD]{\xi^2 \le s} \ge (1 - \eta)(1 - \frac{\overline{\eta} - \eta}{2}) \ge 1 - \overline{\eta}$.
	%\ge 1 - 2\bareta$. 
	Using that $\alpha < \eta < 1$ by assumption, we conclude that as long as $n = \Omega(\frac{\eta \log(\min(n,d)/\delta)}{\alpha^2})$, then
	\begin{equation}
		\frac{1}{n}\sum^n_{t=1}a^*_t x_t x_t^{\top} \succeq (1 - \eta)\Sig_n - \alpha\cdot \Id,
	\end{equation} from which part~\ref{item:sigconc} follows.
%    For part~\ref{item:hyper}, this follows from \eqref{eqn:bdd-noise}.
\begin{comment}	For part~\ref{item:hyper}, note that $\frac{1}{n}\sum^n_{t=1}a'_t\xi_t^{k} = \frac{1}{n}\sum^n_{t=1}\bone{\xi^2_t \le s}\xi_t^{k}$
	where $\xi_1,...,\xi_n$ are sampled independently from $\calD$. Let $\calD'$ be the distribution over $\xi\cdot\bone{\xi^2 \le s}$ where $\xi\sim\calD$. Clearly $\E[\xi'\sim\calD']{|\xi'|^{k}} \le \E[\xi\sim\calD]{|\xi|^{k}} \le (c\sigma k^{1/2})^{k}$. So by Chernoff, we conclude that 
	\begin{equation}
		\Pr[\xi'_1,...,\xi'_n\sim\calD']*{\frac{1}{n}\sum_t {\xi'_t}^{k} \ge 2(c\sigma k^{1/2})^{k}} \le \exp\left(-\Omega\left(\frac{n(c^2\sigma^2k)^{k}}{s^{k}}\right)\right) \le \exp(-\Omega(nc^{2k}k^{k}\bareta^{k})),
	\end{equation} so as long as $n \ge \Omega(\log(1/\delta)/(c^2k\bareta)^{k})$, part~\ref{item:hyper} holds with probability at least $1 - \delta/3$.
\end{comment}
\end{proof}

% \begin{lemma}\label{lem:regularity}
% 	With probability $???$ over the randomness of $\brc{x_t}, \brc{a^*_t}, \brc{\xi_t}$, the following holds. First define $\Sig' \triangleq \frac{1}{n}\sum_{t: a'_t = 1} x_t x_t^{\top}$. Then
% 	\begin{enumerate}
% 		\item $\sum_t a'_t \ge 1 - \bareta - \alpha$.
% 		\item $\frac{1}{n}\sum_{t\in T}\xi_t \iprod{x_t, v} \le O(\sigma\sqrt{d\log(2/\delta)/n})$ for all $v$ for which $\norm{v}_{\Sig'} \le 1$. \label{item:net}
% 		\item $\Sig' \succeq (1 - \bareta)\Sig_n - \alpha\cdot \Id$. \label{item:psdupper}
% 		\item For any $2\le k' \le k$, $\frac{1}{n}\sum_{t\in T} |\xi_t|^k \le (2c\sigma k^{1/2})^k$. \label{item:hyper}
% 	\end{enumerate}
% \end{lemma}
We are now ready to prove Theorem~\ref{thm:nonsos-bdd}. We present the deterministic argument in Lemma~\ref{lem:deterministic-argument} below, then show how combining it with the previous Lemma establishes the result. 
\begin{lemma}\label{lem:deterministic-argument}
Suppose that:
\begin{enumerate}
\item $x_1,\ldots,x_n \in \mathbb{R}^d$, $a^*_1,\ldots,a^*_n \in \{0,1\}$, and for $t = 1,\ldots,n$ we have a sequence $y^*_t$ such that $u^*,\delta_t$ defined by \eqref{eq:ustardef} satisfies \eqref{eq:delta_bound}. 
\item $y_1,\ldots,y_n \in \mathbb{R}$ satisfy
\[ y_t = y^*_t + \xi_t \]
whenever $a^*_t = 1$ and $|\xi_t| \le \sigma$ as in \eqref{eqn:bdd-noise}.
\item The conclusions of Theorem~\ref{thm:nonsos-bdd} are satisfied with parameters $\eta,\lambda,\lambda',\alpha$. The parameter $\beta$ is defined in terms of $\eta$ by \eqref{eqn:rho-defn}.
\item $w \in \mathbb{R}^d$ and $a_1,\ldots,a_n \in [0,1]$ are feasible for Program~\ref{program:biconvex} and satisfy the conclusion of Lemma~\ref{lem:getfirstorder}, i.e. partial optimality and $\epsilon_{grad}$-approximate first order stationarity. 
\end{enumerate}
Then, the following conclusion holds:
\[ \beta \norm{u^* - w}_{\Sig_n} 
    \lesssim \eta \sigma + \eta^{1/2}\epsilon + \sigma\lambda + \left(\graderr^{1/2} + (R\sigma\lambda')^{1/2} + (R^2 \alpha)^{1/4} \left(\sqrt{\eta^{1/2} \sigma + \epsilon} + (R^2 \alpha)^{1/4}\right)\right). \]
\end{lemma}
\begin{proof}[Proof of Theorem~\ref{thm:nonsos-bdd}]
	Let $w$ and $a_1,\ldots,a_n$ be given by Lemma~\ref{lem:getfirstorder}. Let $T$ denote the subset of $t\in[n]$ for which $a^*_t = 1$, i.e. $T$ is the set of rounds which are uncorrupted.  We apply the first order optimality condition \eqref{eq:firstorder} with $v = u^*$ to get that 
	\begin{equation}\label{eq:firstorder-u}
		\frac{1}{n}\sum_t a_t(y_t - \iprod{w,x_t})\iprod{x_t, u^* - w} \le \graderr.
	\end{equation}
	
	We will lower bound the left-hand side of \eqref{eq:firstorder-u} by considering the contribution from $T$ and $[n]\setminus T$. 
	
	\paragraph{Contribution from $T$.} For the former, we have \begin{align}
		\MoveEqLeft \frac{1}{n}\sum_{t\in T} a_t(y_t - \iprod{w,x_t})\iprod{x_t, u^* - w} \\
		&= \frac{1}{n}\sum_{t\in T} a_t(\delta_t + \xi_t + \iprod{u^* - w,x_t})\iprod{x_t, u^* - w} \\
		&= \frac{1}{n}\sum_{t\in T}\big[\underbrace{a_t\iprod{x_t, u^* - w}^2}_{\circled{1}} + \underbrace{\xi_t\iprod{x_t, u^* - w}}_{\circled{2}} - \underbrace{(1 - a_t)\xi_t\iprod{x_t, u^* - w}}_{\circled{3}} + \underbrace{a_t\delta_t\iprod{x_t, u^* - w}}_{\circled{4}}\big]. \label{eq:1234}
	\end{align}
	We control all four terms separately, $\circled{1}$ being the dominant term. Define $\Sig'$ as in Lemma~\ref{lem:regularity}.

	%For $\circled{1}$, we actually prove two lower bounds and then average them. In the first lower bound, we write $a_t = 1 - (1 - a_t)$ and use the last constraint in Program~\ref{program:basic_sdp} to get 
	%\begin{align}
	%	\frac{1}{n}\sum_{t\in T} a_t\iprod{x_t, u^* - w}^2 
	%	&= \frac{1}{n}\sum_{t\in [n]} a_t\iprod{x_t, u^* - w}^2 - \frac{1}{n}\sum_{t\notin T} a_t\iprod{x_t, u^* - w}^2  \\
	%	&\ge (1 - \eta)\norm{u^* - w}^2_{\Sig_n} - \frac{1}{n} \sum_{t \notin T} a_t\iprod{x_t, u^* - w}^2  - \alpha\norm{u^* - w}^2_2 \\
	%	&\ge (1 - \eta)\norm{u^* - w}^2_{\Sig_n} - \frac{1}{n} \sum_{t \notin T} a_t\iprod{x_t, u^* - w}^2 - O(\alpha R^2), \label{eq:circle1}
	%\end{align}
	%where in the last step 
	%we used Lemma~\ref{lem:regularity} and  
	%$\|u^* - w\|_2 \le 2R$ by triangle inequality. In the second lower bound,
	For $\circled{1}$,
	we write $a_t = 1 - (1 - a_t)$ and use Lemma~\ref{lem:regularity} to get
	\begin{align}
		\frac{1}{n}\sum_{t\in T} a_t\iprod{x_t, u^* - w}^2 
		&= \frac{1}{n}\sum_{t\in T} \iprod{x_t, u^* - w}^2 - \frac{1}{n}\sum_{t \in T} (1 - a_t)\iprod{x_t, u^* - w}^2  \\
		&= \norm{u^* - w}^2_{\Sig'} - \frac{1}{n}\sum_{t \in T} (1 - a_t)\iprod{x_t, u^* - w}^2  \\
		&\ge (1 - \eta)\norm{u^* - w}^2_{\Sig_n} - \frac{1}{n}\sum_{t \in T} (1 - a_t)\iprod{x_t, u^* - w}^2 - O(\alpha R^2) \\
		&\ge (1 - 2\eta)\norm{u^* - w}^2_{\Sig_n} - O(\alpha R^2), \label{eq:circle2}
	\end{align}
	where in the last step we expanded the sum from $i \in T$ to $i \in [n]$ and then used the last constraint in Program~\ref{program:basic_sdp}.
%	Since the geometric mean of two positive numbers is a convex combination of them, taking an appropriate convex combination of \eqref{eq:circle1} and \eqref{eq:circle2} shows that
%	\begin{align}
%	    \frac{1}{n}\sum_{t\in T} a_t\iprod{x_t, u^* - w}^2 
%        &\ge (1 - \eta)\norm{u^* - w}^2_{\Sig_n} - \Delta - O(\alpha R^2) \label{eq:circle3}.
%	\end{align}
%	where
%	\begin{equation}\label{eqn:Delta}
%	    \Delta \triangleq \left(\frac{1}{n} \sum_{t \notin T} a_t\iprod{x_t, u^* - w}^2\right)^{1/2}\left(\frac{1}{n}\sum_{t \in T} (1 - a_t)\iprod{x_t, u^* - w}^2\right)^{1/2}.
%	\end{equation}

	%\todo{actually, we get $1 - 2\eta$ here and we need to merge a term from the $[n] \setminus T$ case.}
	%Part~\ref{program:basic_sdp} of Lemma~\ref{lem:regularity}.

	For $\circled{2}$, note that \begin{equation}
		\frac{1}{n}\sum_{t\in T} \xi_t\iprod{x_t,u^* - w} 
		%\le \norm{u^* - w}_{\Sig'}\sup_{v: \norm{v}_{\Sig'} \le 1}\frac{1}{n}\sum_{t\in T}\xi_t\iprod{x_t, v} 
		\le O\left(\norm{u^* - w}_{\Sig_n} \sigma \lambda + R \sigma \lambda' \right)
	\end{equation} by Part~\ref{item:vanish} of Lemma~\ref{lem:regularity}, the fact that $\Sig' \preceq \Sig_n$, and $\norm{u^* - w} \le 2R$. 

	For $\circled{3}$, we have that \begin{equation}
		\frac{1}{n}\sum_{t\in T}(1 - a_t)\xi_t\iprod{x_t, u^* - w}\le \left(\frac{1}{n}\sum_{t\in T}(1 - a_t)\iprod{x_t, u^* - w}^2\right)^{1/2} \left(\frac{1}{n}\sum_{t\in T}(1 - a_t)\xi^2_t\right)^{1/2}. \label{eq:term2_cs}
	\end{equation} 
	By the last constraint in Program~\ref{program:basic_sdp}, we can upper bound the first factor on the right-hand side by 
	$\sqrt{\eta\norm{u^* - w}^2_{\Sig_n} + \alpha\norm{u^* - w}^2_2} \le \eta^{1/2}\norm{u^* - w}_{\Sig_n} + \sqrt{\alpha} R$. For the second factor, we can upper bound it by Holder's inequality as (recalling $\alpha \le \eta$) we have
	\begin{equation}\label{eqn:noise-holder}
	\left(\frac{1}{n}\sum_{t\in T}(1 - a_t)\xi^2_t\right)^{1/2} \le \sqrt{\eta} \sigma 
	\end{equation}
	so overall we get a bound on \eqref{eq:term2_cs} of 
	$\left(\eta^{1/2}\norm{u^* - w}_{\Sig_n} + \sqrt{\alpha} R\right)\cdot \eta^{1/2}\sigma$.
	\begin{comment}
	For the second factor on the right-hand side, if $k = 2$ then we can naively upper bound it by $\left(\frac{1}{n}\sum_{t\in T}\xi^2_t\right)^{1/2}\le O(\sigma^2)$, where the last inequality follows by Part~\ref{item:hyper} of Lemma~\ref{lem:regularity}.

	If $k > 2$, then we have 
	\begin{align}
		\frac{1}{n}\sum_{t\in T}(1 - a_t)\xi^2_t &\le \left(\frac{1}{n}\sum_{t\in T}(1 - a_t)^{k/(k-2)}\right)^{1 - 2/k}\left(\frac{1}{n}\sum_{t\in T} \xi^k_t\right)^{2/k} \\
		&\le \left(\frac{1}{n}\sum_{t\in T}1 - a_t\right)^{1 - 2/k}\cdot O(\sigma^2 M^2) \\
		&\le O((\eta + \alpha)^{1 - 2/k}\cdot \sigma^2 k) \le O(\eta^{1-2/k}\sigma^2 k),
	\end{align} where the second step follows by Part~\ref{item:hyper} of Lemma~\ref{lem:regularity} and the constraint that $a_t\in[0,1]$ for all $t$ so that $(1 - a_t)^{k/(k-2)} \le 1 - a_t$, and the third step follows by the second constraint in Program~\ref{program:basic_sdp}.

	In either case, we can therefore upper bound the right-hand side of \eqref{eq:term2_cs} by $\left(\eta^{1/2}\norm{u^* - w}_{\Sig_n} + \sqrt{\alpha} R\right)\cdot \eta^{1/2 - 1/k}\cdot \sigma\sqrt{k}$.
	\end{comment}

	Finally, for $\circled{4}$, note that first-order optimality of $u^*$ implies that $\frac{1}{n}\sum_t \delta_t\iprod{x_t, u^* - w} = 0$. So we can write \begin{align}
		\MoveEqLeft \frac{1}{n}\sum_{t\in T} a_t \delta_t\iprod{x_t, u^* - w} \\
		&= \frac{1}{n}\sum_{t\in[n]} (1 - a_t)\delta_t\iprod{x_t, u^* - w} - \frac{1}{n}\sum_{t\not\in T} a_t\delta_t\iprod{x_t, u^* - w}. \\
		&\le \left(\frac{1}{n}\sum_{t\in [n]}(1 - a_t)^2\iprod{x_t,u^* - w}^2\right)^{1/2}\left(\frac{1}{n}\sum_{t\in [n]}\delta_t^2\right)^{1/2} + \left(\frac{1}{n}\sum_{t\not\in T} a_t^2\iprod{x_t, u^* - w}^2\right)^{1/2}\left(\frac{1}{n}\sum_{t\not\in T}\delta_t^2\right)^{1/2} \\
		&\le \left(\frac{1}{n}\sum_{t\in[n]}\delta^2_t\right)^{1/2}\cdot \left(\eta^{1/2}\norm{u^* - w}_{\Sig_n} + \left(\eta\norm{u^* - w}^2_{\Sig_n} + \alpha\norm{u^* - w}^2_2\right)^{1/2}\right),
	\end{align} where in the last step we used the fact that $(1 - a_t)^2\le 1 - a_t$ and $a^2_t \le 1$ by the first constraint in Program~\ref{program:basic_sdp}, as well as the third constraint in Program~\ref{program:basic_sdp} and Part~\ref{item:sigconc}.

    Using \eqref{eq:delta_bound} to upper bound the first parenthesized term, we conclude that \begin{equation}
		\frac{1}{n}\sum_{t\in T} a_t \delta_t\iprod{x_t, u^* - w} \le O(\epsilon\cdot (\eta^{1/2}\norm{u^* - w}_{\Sig_n} + \sqrt{\alpha}R)).
	\end{equation}
	Having controlled $\circled{1}, \circled{2}, \circled{3}, \circled{4}$, from \eqref{eq:1234} we can therefore lower bound $\frac{1}{n}\sum_{t\in T} a_t(y_t - \iprod{w,x_t})\iprod{x_t, u^* - w}$ by 
	\begin{align}
		&(1 - 2\eta)\norm{u^* - w}^2_{\Sig_n} \\
		%- \Delta \\
		&\quad- O\Bigg(\norm{u^* - w}_{\Sig_n}\left(\sigma \lambda + \eta \sigma + \epsilon\eta^{1/2}\right)
		+ \alpha R^2 + R\sigma\lambda' + \sqrt{\alpha} R \eta^{1/2}\sigma\Bigg). \label{eq:goodterms}
	\end{align}
%	The second term $\Delta$ will be analyzed along with the contributions of $[n] \setminus T$. \note{it isn't done so far, though some changes were made.}
%	\todo{handle the second term here, should merge with terms from $[n] \setminus T$?}

    \paragraph{Contribution from $[n] \setminus T$.}% and $\Delta$.}
	It remains to control the contribution to the left-hand side of \eqref{eq:firstorder-u} coming from the corrupted summands indexed by $[n]\setminus T$, which we do by upper bounding the term in absolute value. By Cauchy-Schwarz and $a^2_t \le a_t$,
	\begin{align}
		\left|\frac{1}{n}\sum_{t\not\in T} a_t(y_t - \iprod{w,x_t})\iprod{x_t, u^* - w}\right| 
		&\le \left(\frac{1}{n}\sum_{t\not\in T} a_t(y_t - \iprod{w,x_t})^2\right)^{1/2}\left(\frac{1}{n}\sum_{t\not\in T} a_t \iprod{x_t, u^* - w}^2\right)^{1/2} \\
		&\le \left(\frac{1}{n}\sum_{t\not\in T} a_t(y_t - \iprod{w,x_t})^2\right)^{1/2}\left(\eta^{1/2}\norm{u^* - w}_{\Sig_n} + \sqrt{\alpha} R\right) \label{eq:twofactors}
	\end{align}
%	which means, recalling \eqref{eqn:Delta}, that
%	\begin{align}
%		&\left|\frac{1}{n}\sum_{t\not\in T} a_t(y_t - \iprod{w,x_t})\iprod{x_t, u^* - w}\right| + \Delta \\
%		&\le \left(\frac{1}{n}\sum_{t\not\in T} a_t(y_t - \iprod{w,x_t})^2 
%		+ \frac{1}{n}\sum_{t \in T} (1 - a_t)\iprod{x_t, u^* - w}^2\right)^{1/2}\left(\frac{1}{n}\sum_{t\not\in T} a_t \iprod{x_t, u^* - w}^2\right)^{1/2} \\
%		&\le \left(\frac{1}{n}\sum_{t\not\in T} a_t(y_t - \iprod{w,x_t})^2 + \frac{1}{n}\sum_{t \in T} (1 - a_t)\iprod{x_t, u^* - w}^2\right)^{1/2}\left(\eta^{1/2}\norm{u^* - w}_{\Sig_n} + \sqrt{\alpha} R\right) \label{eq:twofactors}
%	\end{align}
	where in the second step we used the fact that $a_t\in\brk{0,1}$ along with Part~\ref{item:sigconc} of Lemma~\ref{lem:regularity}. As for the first factor on the right-hand side, by the fact that $\brc{a_t}$ were chosen in Program~\ref{program:basic_sdp} to minimize $\frac{1}{n}\sum_{t\in [n]} a_t(y_t - \iprod{w,x_t})^2$, we have that 
	\begin{equation}
		\frac{1}{n}\sum_{t\in [n]}a_t(y_t - \iprod{w,x_t})^2 \le \frac{1}{n}\sum_{t\in [n]}a^*_t(y_t - \iprod{w,x_t})^2 = \frac{1}{n} \sum_{t\in T}(y_t - \iprod{w,x_t})^2,
	\end{equation} 
	% small typo here
	%so by subtracting both sides from $\frac{1}{n}\sum_{t\in[n]}(y_t - \iprod{w,x_t})^2$, we get that \todo{FIXME: shouldn't this be $t \in [n]$? May need a little tweaking here. Temporary section}
	hence rearranging gives %fixed: %\todo{minor fixme: in the last line it should probably be $\epsilon^2$ instead of $\eta \epsilon^2$ below? not a big deal since it's just misspecification-related.}
%	\[ \frac{1}{n} \sum_{t \notin T} a_t (y_t - \langle w, x_t \rangle)^2 \le \frac{1}{n} \sum_{t \in T} (y_t - \langle w, x_t\rangle)^2 - \frac{1}{n} \sum_{t \in T}^n a_t (y_t - \langle w, x_t \rangle)^2 = \frac{1}{n} \sum_{t \in T} (1 - a_t)(y_t - \langle w, x_t \rangle)^2 \]
%	and so
	%\todo{end temporary section}
	\begin{align}
		\MoveEqLeft \frac{1}{n}\sum_{t\not\in T}a_t (y_t - \iprod{w,x_t})^2 \\
		&\le \frac{1}{n} \sum_{t \in T} (y_t - \langle w, x_t\rangle)^2 - \frac{1}{n} \sum_{t \in T}^n a_t (y_t - \langle w, x_t \rangle)^2 \\
		&= \frac{1}{n} \sum_{t \in T} (1 - a_t)(y_t - \langle w, x_t \rangle)^2 \\
		&= \frac{1}{n}\sum_{t\in T}(1 - a_t)(\iprod{u^* - w, x_t} + \delta_t + \xi_t)^2 \\
		&\le \frac{2 + 1/\beta}{n}\sum_{t\in T}(1 - a_t)\xi^2_t + \frac{2 + 1/\beta}{n}\sum_{t\in T}(1 - a_t)\delta^2_t + \frac{1 + 2\beta}{n}\sum_{t\in T}(1 - a_t)\iprod{u^* - w, x_t}^2
	\end{align} 
	where %$0<\beta<1$ is a small constant to be tuned later, and where 
	in the second-to-last step we used Cauchy-Schwarz to show
	\[ (a + b + c)^2 \le (2 + 1/\beta)(a^2 + b^2 + c^2\beta) = (2 + 1/\beta)(a^2 + b^2) + (1 + 2 \beta)c^2.\]
	\begin{comment}
	Now observe \todo{todo}
	\begin{align}
	    \frac{1}{n}\sum_{t\in T}(1 - a_t)\iprod{u^* - w, x_t}^2 + \frac{1}{n} \sum_{t \notin T} a_t\iprod{u^* - w, x_t}^2 = \frac{1}{n} \sum_{t \in [n]}(1 - a_t)\iprod{u^* - w, x_t}^2 + \frac{1}{n} \sum_{i \notin T} (2a_t - 1)\langle u^* - w, x_t\rangle^2
	\end{align}
	Therefore \todo{figure out this breakpoint business. try inclusion-exclusion instead of averaging?}
	\end{comment}
	We continue and see
	\begin{align}
	\frac{1}{n}\sum_{t\not\in T}a_t (y_t - \iprod{w,x_t})^2
	    \le (1 + 2\beta) \eta \|u^* - w\|_{\Sig_n}^2 + O\left(\frac{1}{\beta}\eta \sigma^2 + \frac{1}{\beta} \epsilon^2 + \alpha R^2\right)
	\end{align}
	where in the last step we used Holder's inequality and \eqref{eqn:bdd-noise},
	%and Part~\ref{item:hyper} of Lemma~\ref{lem:regularity}, 
	\eqref{eq:delta_bound}, and the last constraint in Program~\ref{program:basic_sdp} with $\|u^* - w\| \le 2R$.

	So by \eqref{eq:twofactors} we can upper bound $\frac{1}{n}\sum_{t\not\in T} a_t(y_t - \iprod{w,x_t})\iprod{x_t, u^* - w}$ by 
	\begin{multline}
		\left((1 + 2\beta)^{1/2} \eta^{1/2} \|u^* - w\|_{\Sig_n} + O\left(\beta^{-1/2}\eta^{1/2}\sigma + \beta^{-1/2} \epsilon + \alpha^{1/2} R\right)\right)\left(\eta^{1/2}\norm{u^* - w}_{\Sig_n} + \sqrt{\alpha} R\right) \\
		= (1 + 2\beta)^{1/2}\eta\norm{u^* - w}_{\Sig_n}^2 + O(\beta^{-1/2}\eta \sigma + \beta^{-1/2}\eta^{1/2}\epsilon + \alpha^{1/2}\eta^{1/2}R)\norm{u^* - w}_{\Sig_n} + \mathcal{E}, \label{eq:badterms}
	\end{multline} where
	\begin{equation}\label{eqn:error-terms}
		%\mathcal{E}\triangleq \sqrt{\alpha}R \cdot O\left(\eta^{1/2} \|u^* - w\|_{\Sig_n} + \beta^{-1/2}\eta^{1/2 - 1/k}\sigma k^{1/2} + \beta^{-1/2} \eta^{1/2} \epsilon + \alpha^{1/2} R\right)
		% simplified
		\mathcal{E}\triangleq \sqrt{\alpha} R \cdot O(\beta^{-1/2}(\eta^{1/2} \sigma + \epsilon) + \sqrt{\alpha} R)
	\end{equation}
	captures all the error terms that vanish as $\alpha \to 0$. 
	%\todo{check this minor term more carefully and optimize}

    \paragraph{Combining.}
	Putting the bounds on $\frac{1}{n}\sum_{t\in T} a_t(y_t - \iprod{w,x_t})\iprod{x_t, u^* - w}$ and $\frac{1}{n}\sum_{t\not\in T} a_t(y_t - \iprod{w,x_t})\iprod{x_t, u^* - w}$ by \eqref{eq:goodterms} and \eqref{eq:badterms} together with \eqref{eq:firstorder}, we conclude that \begin{equation}
		(1 - 3\eta - \sqrt{2\beta} \cdot \eta))\norm{u^* - w}^2_{\Sig_n} \le O\left(\beta^{-1/2}\eta \sigma + \beta^{-1/2} \eta^{1/2}\epsilon + \alpha^{1/2}R + \sigma \lambda \right)\norm{u^* - w}_{\Sig_n} + \calE',
	\end{equation} 
	where $\calE' \triangleq \graderr + R\sigma \lambda' + O(\calE)$.
%		\calE' \triangleq \graderr + O\left(\frac{\alpha^{1/2}R}{\beta^{1/2}}\eta^{1/2 - 1/k}\sigma\sqrt{k} + \sqrt{\alpha R}\epsilon + \alpha R^2\right).
% simplified
	We do case analysis based on which of the two terms on the rhs of the above bound dominates:
	\begin{enumerate}
		\item In the first case, the first term is at least as large as $\calE'$. Then the bound simplifies to 
		\[ (1 - 3\eta - \sqrt{2\beta} \cdot \eta)\norm{u^* - w}_{\Sig_n} \lesssim \beta^{-1/2}\eta \sigma + \beta^{-1/2} \eta^{1/2}\epsilon + \alpha^{1/2}R + \sigma \lambda \]
	    \item Otherwise, $\calE'$ is larger than the first term. Then taking a square root  the bound can be simplified to
	    \[ (1 - 3\eta - \sqrt{2\beta} \cdot \eta)\norm{u^* - w}_{\Sig_n} \lesssim \graderr^{1/2} + (R\sigma \lambda')^{1/2} + \mathcal{E}^{1/2}. \]
	    % (R^2 \alpha)^{1/4} (\beta^{-1/4}\eta^{1/4}\sqrt{\sigma M + \epsilon} + (R^2 \alpha)^{1/4})
	\end{enumerate}
	In either case, since $\alpha^{1/2} R = O(\calE^{1/2})$ we see the inequality
	\[ (1 - 3\eta - \sqrt{2\beta} \cdot \eta)\norm{u^* - w}_{\Sig_n} \lesssim \beta^{-1/2}\eta \sigma + \beta^{-1/2} \eta^{1/2}\epsilon + \sigma\lambda + \graderr^{1/2} + (R\sigma \lambda')^{1/2} + \mathcal{E}^{1/2} \]
	holds. Since $\beta = (1/3 - \eta)^2$ and $\eta < 1/3$ we know
	\[ (1 - 3\eta - \sqrt{2\beta} \eta) \ge 3\sqrt{\beta} - \sqrt{2\beta} = \Theta(\sqrt{\beta}) \]
	\begin{comment}
    Observe that
	\[ (2 + \sqrt{2\beta})\eta = \frac{2 + \sqrt{2\beta}}{2 + 2\sqrt{2\beta}} = 1 - \frac{\sqrt{2\beta}}{2 + 2\sqrt{2\beta}}  \]
    hence by dividing through we get a final bound of
    \[ \frac{1}{1 - (2 + \sqrt{2\beta})\eta} = \frac{2 + 2\sqrt{2\beta}}{\sqrt{2\beta}} = O(\beta^{-1/2}) \]
    \end{comment}
    so we get a final bound of 
    \begin{align*} 
    &\norm{u^* - w}_{\Sig_n}  \\
    &\lesssim \beta^{-1}\eta \sigma + \beta^{-1} \eta^{1/2}\epsilon + \beta^{-1/2}\sigma\lambda + \beta^{-1/2}(\graderr^{1/2} + (R\sigma \lambda')^{1/2} + \calE^{1/2}) \\
    &\lesssim \beta^{-1}\eta \sigma + \beta^{-1} \eta^{1/2}\epsilon + \beta^{-1/2}\sigma\lambda + \beta^{-1/2}(\graderr^{1/2} + (R\sigma\lambda')^{1/2} + (R^2 \alpha)^{1/4} (\beta^{-1/4}\sqrt{\eta^{1/2} \sigma + \epsilon} + (R^2 \alpha)^{1/4})).
    \end{align*}
    Using $\beta < 1$ to upper bound all of the powers of $\beta$ by $\beta^{-1}$ gives the result. 
    %The simplified bound follows by setting $\graderr$ sufficiently small and plugging in the stated values of $\overline{\eta},\alpha$ and upper bounding using $\beta < 1$ and $1 - 1/k \ge 1/2$ since $k \ge 2$.
%	\TODO{specify $\alpha$, $\beta$...}
\end{proof}
Now combining our claims proves Theorem~\ref{thm:nonsos-bdd}:
%\todo{right now, this is for the oblivious covariates setting. add another version with the adaptive covariates setting.}
\begin{proof}[Proof of Theorem~\ref{thm:nonsos-bdd}]
\textbf{Oblivious covariates.}
By Lemma~\ref{lem:deterministic-argument}  and Lemma~\ref{lem:regularity} we know the output $w$ of Lemma~\ref{lem:getfirstorder} with $\epsilon_{grad} = O(\sigma^2 \lambda^2) = O(\sigma^2 \frac{d + \log(2/\delta)}{n})$ satisfies
\[ \beta \norm{u^* - w}_{\Sig_n}
    \lesssim \eta \sigma +  \eta^{1/2}\epsilon + \sigma\sqrt{\frac{d + \log(2/\delta)}{n}} + (R^2 \alpha)^{1/4} (\sqrt{\eta^{1/2} \sigma + \epsilon} + (R^2 \alpha)^{1/4})
\]
with probability at least $1 - \delta$, as long as $\alpha < \eta$ and \eqref{eq:nbound} holds:
\[ n \gtrsim \frac{\eta \log(\min(n,d)/\delta)}{\alpha^2}. \]
Based on this we take $\alpha = \Theta\left(\sqrt{\frac{\eta \log(\min(n,d)/\delta)}{n}}\right)$ and require
\[ n \gtrsim \log(\min(n,d)/\delta)/\eta \]
so that $\alpha < \eta$. 
Then we can write the error bound as
\begin{align}\label{eqn:preliminary-conclusion}
\beta \norm{u^* - w}_{\Sig_n}
    &\lesssim \eta \sigma +  \eta^{1/2}\epsilon +  \eta^{1/8} R^{1/2}\sqrt[8]{\frac{(\eta^{1/2} \sigma + \epsilon)^4\log(\min(n,d)/\delta)}{n}}  \\
    &\quad + \eta^{1/4} R\sqrt[4]{\frac{\log(\min(n,d)/\delta)}{n}} + \sigma\sqrt{\frac{d + \log(2/\delta)}{n}}.
\end{align}
% OLD: not helpful anymore due to change
\begin{comment}
We can simplify the bound slightly using the following observation, which follows from the AM-GM inequality:
\begin{align*}
\beta^{-3/4} \eta^{3/8} R^{1/2}\sqrt[8]{\frac{\epsilon^4\log(d/\delta)}{n}}
&\lesssim \beta^{-1} \eta^{1/2} \epsilon + \beta^{-1/2} \eta^{1/4} R \sqrt[4]{\frac{\log(d/\delta)}{n}}
\end{align*}
which lets us drop part of the third term in \eqref{eqn:preliminary-conclusion}. 
\end{comment}
\paragraph{Adaptive covariates.} The only change is that the term $ \sigma \lambda$ disappears and the term 
\[ (R \sigma \lambda')^{1/2} =  (R \sigma)^{1/2} \sqrt[4]{\frac{\log(2/\delta)}{n}}  \]
appears, which gives
\begin{align}\label{eqn:adaptive-conclusion}
\beta \norm{u^* - w}_{\Sig_n}
    &\lesssim \eta \sigma + \eta^{1/2}\epsilon +  \eta^{1/8} R^{1/2}\sqrt[8]{\frac{(\eta^{1/2} \sigma + \epsilon)^4\log(\min(n,d)/\delta)}{n}}  \\
    &\quad + \eta^{1/4} R\sqrt[4]{\frac{\log(\min(n,d)/\delta)}{n}}  +  (R \sigma)^{1/2} \sqrt[4]{\frac{\log(2/\delta)}{n}}.
\end{align}
Since this bound also applies in the special case of oblivious covariates, we get the stated result. 
\end{proof}
\subsubsection{Subgaussian noise}
\label{subsubsec:subgauss}
In this section we consider the case where the noise is subgaussian. Subgaussian random variables are in $L_q$ for every $q$, so we could analyze them using our previous result (taking $q = \log(1/\eta)$), but since subgaussian noise behaves similar to bounded noise, we can optimize the argument by avoiding truncation. This yields the following result, which in the uncontaminated $\eta = 0$ setting with oblivious covariates, recovers the same (minimax optimal) rate achieved by Ordinary Least Squares/Ridge Regression and gracefully degrades with increasing $\eta$.
\begin{theorem}[\textsc{SCRAM} Guarantee with Subgaussian Noise]\label{thm:nonsos-subgaussian}
Suppose that $\eta < 1/3$ is an upper bound on the contamination level, define $\beta$ as in \eqref{eqn:rho-defn}, and suppose for some $\sigma \ge 0$ that 
for all $t$ the noise $\xi_t$ is $\sigma^2$-subgaussian.
%and $\lambda \in \mathbb{R}$
%\begin{equation}\label{eqn:subgaussian-noise}
%     \log \E{\exp(\lambda \xi_t)} \le \lambda^2 \sigma^2/2
%\end{equation}
%almost surely. 
Then if $\eta = 0$ or
\[ n \gtrsim \log(\min(n,d)/\delta)/\eta, \]
$\alpha = \Theta\left(\sqrt{\frac{\eta \log(d/\delta)}{n}}\right)$ and $\overline{\eta} = \eta$, the output $w$
of \textsc{SCRAM} with $poly(R/\sigma,\log(2/\delta),d,n)$ many steps satisfies for oblivious covariates the bound
\begin{align*} 
\beta \norm{u^* - w}_{\Sig_n}
    &\lesssim c_{\delta,\eta,n} \eta \sigma + \eta^{1/2}\epsilon + \eta^{1/8} R^{1/2} (\sqrt{c_{\delta,\eta,n} \eta}\sigma  + \epsilon)^{1/2}\sqrt[8]{\frac{\log(\min(n,d)/\delta)}{n}} \\
    &\quad+  \eta^{1/4} R\sqrt[4]{\frac{\log(\min(n,d)/\delta)}{n}} 
     + \min \left\{\sigma\sqrt{\frac{d + \log(2/\delta)}{n}}, (R\sigma)^{1/2}\sqrt[4]{\frac{\log(2/\delta)}{n}}\right\}
\end{align*}
with probability at least $1 - \delta$, where
\begin{equation} 
    c_{\delta,\eta,n} \triangleq \sqrt{\log(1/\eta)} \exp\left(\max\left(1,\frac{\log\log(1/\delta) \cdot \log(1/\eta)}{2\log(n)}\right)\right) \label{eq:cdef}
\end{equation}
captures a logarithmic term which is $O(\sqrt{\log(1/\eta)})$ assuming $\log n \ge (1/100) \log\log(1/\delta)\log(1/\eta)$.
In the more general case of adaptive covariates, \textsc{SCRAM} satisfies the bound 
\begin{align} 
\beta \norm{u^* - w}_{\Sig_n}
    &\lesssim c_{\delta,\eta,n}\eta \sigma + \eta^{1/2}\epsilon + \eta^{1/8} R^{1/2} (\sqrt{c_{\delta,\eta,n} \eta}\sigma  + \epsilon)^{1/2}\sqrt[8]{\frac{\log(\min(n,d)/\delta)}{n}} \\
    &\quad+  \eta^{1/4} R\sqrt[4]{\frac{\log(\min(n,d)/\delta)}{n}} 
     +  (R\sigma)^{1/2}\sqrt[4]{\frac{\log(2/\delta)}{n}} \label{eq:nosos_subgauss_adaptive}
\end{align}
i.e. the same bound except the last term was changed.
\end{theorem}
\begin{proof}
    The proof is the same as Theorem~\ref{thm:nonsos-bdd} with a few modifications which we describe now. The main difference is in the use of Holder's inequality to bound terms including noise, e.g. \eqref{eqn:noise-holder}. In this case, since $\xi_t$ is no longer bounded we use for $q = \min(2\log(n)/\log\log(1/\delta), \log(1/\eta))$ that by Holder's inequality
    \[ \left(\frac{1}{n}\sum_{t\in T}(1 - a_t)\xi^2_t\right)^{1/2} \le \left(\frac{1}{n} \sum_{t \in T} (1 - a_t)\right)^{1/2p}\left(\frac{1}{n} \sum_{t \in T} \xi_t^{2q} \right)^{1/2q} \lesssim \eta^{1/2 - 1/2q} \sigma \sqrt{q}   \]
    where $1/p + 1/q = 1$ and we used Lemma~\ref{lem:latala-app} below.
    %moment bounds for subgaussian random variables (Proposition 2.5.2 of \cite{vershynin2018high}). 
    Plugging in the value of $q$ gives an upper bound of
    \[ \sigma \eta^{1/2} \sqrt{\log(1/\eta)} \cdot \exp\left(\max\left(1,\frac{\log\log(1/\delta) \cdot \log(1/\eta)}{4\log(n)}\right)\right). \]
    %Then we take $q = \log(1/\eta)$ to optimize the bound giving $\sigma \eta^{1/2} \sqrt{\log(1/\eta)}$. \todo{what is the concentration step here fixme.}
    
    The other change is that in Lemma~\ref{lem:regularity}, we can use the subgaussian property to establish Part 2 without needing boundedness of the noise: we use the generalization of the vector Azuma-Hoeffding inequality to the subgaussian step size setting, Theorem~\ref{thm:azuma-vector}.
\end{proof}
The following Lemma~\ref{lem:latala-app} gives a fairly sharp upper deviation bound for power sums of subgaussian random variables. This result is not so easy to prove directly, but follows from the main result of \cite{latala1997estimation}.
\begin{lemma}\label{lem:latala-app}
Suppose that $Z_1,\ldots,Z_n$ are independent $\sigma^2$-subgaussian random variables. 
Then
\[ \left(\frac{1}{n} \sum_i |Z_i|^p\right)^{1/p} \lesssim \sigma \sqrt{p} \]
with probability at least $1 - \delta$, provided $n \ge \log(2/\delta)^{p/2}$. 
\end{lemma}

\begin{proof}
    We rescale so that $\sigma = 1$. In this proof we use the notation $\|X\|_q = \E{|X|^q}^{1/q}$ for the function space $L_p$ norm. 
    
    Define $S = \sum_i |Z_i|^p$. By Markov's inequality,
    $\Pr{S \ge t} = \Pr{S^q \ge t^q} \le \frac{\|S\|_{q}^q}{t^q}$ for any $q \ge 1$. 
    By Theorem 1 and Corollary 1 of \cite{latala1997estimation}, for $q \le n$ we have
    \begin{align*} 
    \|S\|_q 
    &\lesssim \sup \left\{ (q/s)(n/q)^{1/s} \max_i \|Z_i^p\|_{s} : 1 \le s \le q \right\}.
    \end{align*}
    We observe from standard subgaussian moment bounds \cite{rigollethigh,vershynin2018high} that
    \[ \|Z_i^p\|_s = \|Z_i\|_{sp}^p \lesssim (esp)^{p/2} \]
    so
    \begin{align*} 
    \|S\|_q 
    &\lesssim (ep)^{p/2}q \sup \left\{ (n/q)^{1/s} s^{p/2 - 1} : 1 \le s \le q \right\} \\
    &= (ep)^{p/2}q \sup \left\{\exp((1/s)\log(n/q) + (p/2 - 1) \log(s) : 1 \le s \le q \right\}.
    \end{align*}
    We consider the optimization over $s$ inside the exponential. 
    The unique critical point is when $-s^{-2}\log(n/q) + (p/2 - 1)/s = 0$, i.e. $s = \log(n/q)/(p/2 - 1)$. Since the function goes to infinity as $s \to 0$ and $s \to \infty$, that critical point must be a minimum. It suffices therefore to consider the boundary points. This shows
    \[ \|S\|_q \lesssim (ep)^{p/2} \left(n + n^{1/q} q^{p/2 - 1/q} \right) \lesssim  (ep)^{p/2} \left(n + n^{1/q} q^{p/2} \right)  \]
    using $\max_{q \ge 1} q^{-1/q} = 1$. Now taking $t = e\|S\|_q$ and $q = \log(1/\delta)$ shows
    \[ S \le e\|S\|_q \lesssim (ep)^{p/2}n(1 + n^{1/\log(1/\delta) - 1}\log(1/\delta)^{p/2}) \]
    with probability at least $1 - \delta$. In particular, if $n \ge \log(1/\delta)^{p/2}$ then 
    \[ S \lesssim (ep)^{p/2} n(1 + e^{(p/2)\log\log(1/\delta)/\log(1/\delta)}) \le e^p p^{p/2} n \]
    as claimed.
\end{proof}

\subsection{Stochastic Setting and Generalization Bounds}
\label{subsec:stochastic}

Finally, we note that while the guarantees in this section so far have been in the usual \emph{fixed design} setting, from these guarantees we also obtain strong results in the stochastic (or \emph{random design}) setting often considered in statistical learning. We first review the setup.
We assume there exists a joint distribution $\calD_{x,y^*}$ over clean examples $(x,y^*)$ and clean training data $(x_1,y^*_1),\ldots,(x_n,y^*_n)$ are sampled identically from this distribution. We define the \emph{population loss} to be the error of $w$ on a fresh clean example $(x,y^*)$ in squared loss,
\[ L(w) = \E[x,y \sim \calD_{x,y^*}]{(y^* - \iprod{w,x})^2}, \]
and our goal is to find a near minimizer of the population loss, i.e. compute $\wh{w}$ from training data such that $\|\wh{w}\| \le R$ and the gap in population loss $L(\wh{w}) - L(u^*)$
%\[ L(\wh{w}) - L(u^*) \]
is as small as possible, where we define
\[ u^* \triangleq \arg\min_{\|u\| \le R} L(u) \]
to be the optimal predictor of norm at most $R$.
Concretely, the gap in loss can be rewritten in a more convenient form in the following way
\begin{align*} 
L(\wh{w}) - L(u^*) 
&= \E{(y^* - \iprod{u^*,x} + \iprod{u^* - w, x})^2} - \E{(\iprod{u^* - w, x})^2} \\
&= \E{\iprod{\wh{w} - u^*,x}^2} + 2\E{(y^* - \iprod{u^*,x})\iprod{u^* - w, x})} \\
&= \| \wh{w} - u^*\|_{\Sig^*}^2  + 2\E{(y^* - \iprod{u^*,x})\iprod{u^* - w, x})}
\end{align*}
%\todo{discuss the gap further/upper bound}
where
$\Sig^* = \E[\calD_{x}]{xx^T}$ is the second moment matrix, i.e. covariance matrix if $x$ is mean zero, and the second term on the rhs is $O(\epsilon \|u^* - w\|_{\Sig^*})$ under \eqref{eqn:misspec-stochastic}, showing that as $\epsilon \to 0$, the slightly different goals of minimizing $\|u^* - w\|_{\Sig^*}$ and minimizing the suboptimality in population loss become exactly equivalent.
%The first equality follows follows by writing $y^* - \langle w, x \rangle = y^* - \langle u^*, x \rangle + \langle u^* - w, x \rangle$, expanding the squared loss $L(w)$ and using
%and we used the first-order optimality of $u^*$ to show the cross term $\E{(y^* - \langle u^*, x \rangle)(w - \langle u^*, x \rangle)} \le 0$. 
As before, we assume that the conditional law of $y^*$ given $x$ is
\begin{equation} \label{eqn:misspec-stochastic}
y^* = \langle w^*, x \rangle + \epsilon_x + \xi 
\end{equation}
where $\|w^*\| \le R$, $|\epsilon_x| \le \epsilon$ is misspecification, and $\xi$ is noise independent of $x,\epsilon_x$. If $\epsilon = 0$ then we can take $u^* = w^*$, otherwise we always have $\|u^* - w^*\|_{\Sigma} \le 2\epsilon$ since $|\langle w^*, x \rangle - \E{y^* | x}| \le \epsilon$ and $u^*$ is only closer in average squared loss.

We will use the following Lemma to relate the error when measured according to the population second moment matrix $\Sig^*$ and the random matrix $\Sig_n$: this ``localized'' generalization bound follows from the main result of \cite{srebro2010optimistic}, which builds upon the local Rademacher complexity framework of \cite{bartlett2005local}; it gives tighter results than e.g. naively applying matrix concentration because it focuses in on the behavior of the bottom singular value. We note that the general connection between generalization theory and the bottom singular value of the empirical covariance matrix is well known and has been used in other contexts, see e.g. \cite{koltchinskii2015bounding}.
\begin{lemma}[Consequence of Theorem 1 of \cite{srebro2010optimistic}]\label{lem:generalization}
Suppose $w^*$ is any fixed vector with $\|w^*\| \le R$. Suppose that $x_1,\ldots,x_n$ are iid copies of a random variable $x$ with $\Sig^* = \E{x x^T}$ and $\|x\| \le 1$ almost surely. 
Uniformly over all $w$ with $\|w\| \le R$ and with probability at least $1 - \delta$,
where $\Sig_n = \frac{1}{n} \sum_{i = 1}^n x_i x_i^T$ is the empirical second moment matrix, the following holds:
    \[ \|w - w^*\|_{\Sig^*}^2 - \|w - w^*\|_{\Sig_n}^2 \lesssim R\|w - w^*\|_{\Sig_n} \sqrt{\frac{\log^3(n) + \log(1/\delta)}{n}} +  \frac{R^2(\log^3(n) + \log(1/\delta))}{n} \]
    and as a consequence
    \[ \|w - w^*\|_{\Sig^*} \lesssim \|w - w^*\|_{\Sig_n} + R\sqrt{\frac{\log^3(n) + \log(1/\delta)}{n}}. \]
\end{lemma}
\begin{proof}
    We explain how this follows from Theorem 1 of \cite{srebro2010optimistic}, which requires us to interpret the gap $\|w - w^*\|_{\Sig^*}^2 - \|w - w^*\|_{\Sig_n}^2$ as the generalization gap in a statistical learning problem; we refer the reader  there for a detailed explanation of the setup. We now describe the new learning problem, which is not the same as the one considered outside the proof of this Lemma, as it has no noise, contamination, or misspecification. In this problem, $x$ is defined as in the theorem statement, and the label $y = \langle w^*, x \rangle$. The population loss is $\E{\ell(y - \langle w, x \rangle)} = \langle w^* - w, \Sig (w^* - w) \rangle$ where $\ell(e) = e^2$ is the squared loss which is $1$-smooth, and the empirical loss is $\frac{1}{n} \sum_{i = 1}^n \ell(y_i - \langle w, x \rangle) = \langle w^* - w, \Sig_n (w^* - w) \rangle$. We observe that the loss $\ell(y_i - \langle w, x \rangle)$ is upper bounded by $4R^2$ almost surely, and finally we use (see \cite{srebro2010optimistic}) that the Rademacher complexity $R_n$ of the function class $\{ x \mapsto \langle w, x \rangle : \|w\| \le R\}$ is $O(R\sqrt{1/n})$ where $n$ is the number of samples. Plugging all of this information into Theorem 1 of \cite{srebro2010optimistic} gives
    \[ \|w - w^*\|_{\Sig^*}^2 - \|w - w^*\|_{\Sig_n}^2 \lesssim \|w - w^*\|_{\Sig_n}\left(R\log^{1.5}(n)\sqrt{\frac{1}{n}} + R\sqrt{\frac{\log(1/\delta)}{n}}\right) + \log^3(n) \frac{R^2}{n} + \frac{R^2 \log(1/\delta)}{n}\]
    and up to constants this is equivalent to the first stated bound. The second (weaker) bound follows by adding $\|w - w^*\|_{\Sig_n}^2$ to the right hand side and taking a square root. 
\end{proof}
Given this result, we can immediately obtain versions of all of the previous results for the stochastic setting (e.g. Theorem~\ref{thm:nonsos}, Theorem~\ref{thm:nonsos-subgaussian}, Theorem~\ref{thm:nonsos-bdd}). 
%\todo{actually state the bound for one of these settings?} 
We describe a more involved application below in Section~\ref{sec:geometric-median}, where we obtain improved results for learning in the stochastic setting by using this generalization bound combined with the generalized median of \cite{minsker2015geometric}.

We note that in the case where the contexts are chosen stochastically, \cite{simchi2020bypassing} recently showed that a modified version of the reduction from \cite{foster2020beyond} can reduce from stochastic contextual bandits to offline regression with stochastic contexts. It should be possible to combine this reduction with our results; however, we omit the details since we will give an algorithm for the more general online setting anyway.
%It should be possible to combine this reduction with Corollary~\ref{cor:random-design}; however, we omit the details since we will give an algorithm for the more general online setting anyway.
%\todo{this should probably be integrated into the beginning of the whole section.}

\subsection{Heavy-Tailed Setting Using Geometric Median}\label{sec:geometric-median}

In this section, we focus on the setting where the noise $\brc{\xi_t}$ is in $L_q$ with $q \ge 2$ and obtain improved sample complexity guarantees. In this context, there is a  fairly general way to boost the success probability of algorithms by using the geometric median \cite{minsker2015geometric} or a related high-dimensional median of \cite{hsu2016loss}; in the context of (uncontaminated) ridge regression itself, this kind of estimator was considered in \cite{hsu2016loss}, see Theorem 21 there. To take advantage of the geometric median, we start by establishing improved guarantees for our algorithm, but which hold with only a fixed probability of success.
\begin{lemma}\label{lem:nonsos-lowprob}
Suppose that $\eta < 1/3$ is an upper bound on the contamination level, define $\beta$ as in \eqref{eqn:rho-defn}, and suppose for some $q \in [2,\infty],\sigma_q \ge 0$ and all $t$ that
\begin{equation}\label{eqn:q-noise2}
    \|\xi_t\|_{q} \triangleq \E{|\xi|^q}^{1/q} \le \sigma_q.
\end{equation}
Then provided $\eta = 0$ or
\[ n \gtrsim \log(\min(n,d))/\eta, \]
we can take
$\alpha = \Theta\left(\sqrt{\frac{\eta \log(d)}{n}}\right)$ and $\overline{\eta} = \eta + \Theta(\eta \sqrt{\beta})$ such that the output $w$
of \textsc{SCRAM} with $poly(R/\sigma,d,n)$ many steps satisfies for oblivious covariates the bound
\begin{align}\label{eqn:lowprob-guarantee}
\beta^{1 + 1/q} \norm{u^* - w}_{\Sig_n}
    &\lesssim \eta^{1 - 1/q} \sigma_q + \eta^{1/2}\epsilon + \eta^{1/8} R^{1/2} (\epsilon +  \eta^{1/2 - 1/q} \sigma_q)^{1/2}\sqrt[8]{\frac{\log(\min(n,d))}{n}} \\
    &\quad+  \eta^{1/4} R\sqrt[4]{\frac{\log(\min(n,d))}{n}} 
     + \min \left\{\sigma\sqrt{\frac{d}{n}}, (R\sigma)^{1/2}\sqrt[4]{\frac{1}{n}}\right\}
\end{align}
with probability at least $0.99$. In the more general case of adaptive covariates, it satisfies the bound
\begin{align*} 
\beta^{1 + 1/q} \norm{u^* - w}_{\Sig_n}
    &\lesssim \eta^{1 - 1/q} \sigma_q + \eta^{1/2}\epsilon + \eta^{1/8} R^{1/2} (\epsilon +  \eta^{1/2 - 1/q} \sigma_q)^{1/2}\sqrt[8]{\frac{\log(\min(n,d))}{n}} \\
    &\quad+  \eta^{1/4} R\sqrt[4]{\frac{\log(\min(n,d))}{n}} 
     + (R\sigma)^{1/2}\sqrt[4]{\frac{1}{n}}
\end{align*}
i.e. the same bound except the last term was changed.
\end{lemma}
\begin{proof}
The proof is the same as Theorem~\ref{thm:nonsos} except that we change the analysis of Part~\ref{item:vanish} in Lemma~\ref{lem:regularity} to improve the final term in our bound. First, we observe that truncating the noise $\xi_i$ and recentering (the first part of the proof of Theorem~\ref{thm:nonsos}) can only make the $L_q$ norm of $|\xi_i|$  larger by a factor of $2$ (see the proof of Lemma 2.6.8 in \cite{vershynin2018high}); in what follows, we let $\xi_i$ denote the possibly truncated and recentered noise and use this fact.
%, so without loss of generality we consider the case where truncation, performed in the reduction to Theorem~\ref{thm:nonsos-bdd}, does not affect the noise.
Now we consider the application of Theorem~\ref{thm:nonsos-bdd} in the proof of Theorem~\ref{thm:nonsos} and show how in Part 2 of Lema~\ref{lem:regularity} we can replace the infinity norm of the noise by the smaller quantity $\sigma_q$. Specifically this occurs in Part 2 of Lemma~\ref{lem:regularity}.

For Part~\ref{item:vanish} (a), we replace Lemma~\ref{lem:sphere-max} by the following argument based on Chebyshev's inequality. Let $\xi = (\xi_1,\ldots,\xi_n)$ be the vector of (truncated) noise and observe that $\sqrt{\E{\xi_i^2}} \le \E{|\xi_i|^q}^{1/q} = O(\sigma_q)$ by Jensen to see
\[ \Pr{\|P_V \xi\| \ge s} \le \frac{\E{\|P_V \xi\|^2}}{s^2} \le \frac{\langle P_V P_V^T, \sigma_q^2 I \rangle}{s^2} \le \frac{2d \sigma_q^2}{s^2}.  \]

Similarly for Part~\ref{item:vanish} (b), we use Chebyshev's inequality and the fact that
\[ \E{\|\sum_i \xi_i x_i\|^2} = \sum_i \E{\xi_i^2} \|x_i\|^2 \le 2n \sigma_q^2. \]
to get that $\|\frac{1}{n} \sum_i \xi_i x_i\| = O(\sigma_q/\sqrt{\delta n})$ with probability at least $1 - \delta$.

Taking the union bound and using these estimates in the analysis, otherwise unchanged from the proof of Theorem~\ref{thm:nonsos}, gives the result.
\end{proof}
Given this result, we run the algorithm multiple times, and take the geometric median, as described in {\sc SCRAM-GM}. We recall the key guarantee for geometric median from \cite{minsker2015geometric} in its contrapositive form, which informally says that if a $1 - \alpha > 1/2$ proportion of points cluster near each other, then the geometric median will successfully return a point close to this cluster.
\begin{lemma}[Lemma 2.1 (a) of \cite{minsker2015geometric}]\label{lem:geometric-median}
Suppose $x_1,\ldots,x_n$ are points in a $d$-dimensional Euclidean space with norm $\|\cdot\|$. Suppose $z \in \mathbb{R}^d, r > 0, \alpha \in (0,1/2)$, let $C_{\alpha} \triangleq (1 - \alpha)\sqrt{\frac{1}{1 - 2\alpha}}$,
and let
\[ y = \arg\min_y \sum_{i = 1}^n \|y - x_i\|, \]
be the geometric median. If
\[ \#\{i : \|x_j - z\| > r\} \le \alpha n \]
then $\|y - z\| \le C_{\alpha} r$.
\end{lemma}
\begin{comment}
Compared to \cite{minsker2015geometric}, a slight difference is that our geometric median is taken with respect to different norms for each bucket, so we need the following Lemma to analyze it:
\begin{lemma}[c.f. Lemma 2.1 of \cite{minsker2015geometric}]
\todo{fixme: this lemma seems false. probably not too hard to fix, todo.}
Suppose $x_1,\ldots,x_n$ are points in $\mathbb{R}^d$, for every $i$ $\|\cdot\|_i$ is a norm on $\mathbb{R}^d$, and suppose there exists $z \in \mathbb{R}^, r \ge 0, \delta \ge 0$ such that $1-\delta$ proportion of the $x_i$ satisfy $\|z - x_i\|_i \le r$. Then if
\[ y = \arg\min_y \sum_{i = 1}^n \|y - w_i\|_i, \]
we have $\|y - x_i\|_i = O(r)$ for $1 - O(\delta)$ proportion of $i$.
\end{lemma}
\begin{proof}
    We prove the result by contradiction. Suppose $\|y - x_i\|_i \ge Cr$ for $1 - c\delta$ proportion of $i$. Let
    \[ S = \{ i : \|z - x_i\|_i \le r \} \]
    so $|S| \ge (1 - \delta)n$ by assumption. It follows by reverse triangle inequality that
    \[ \sum_{i \in S} \|y - x_i\|_i \ge \sum_{i \in S} (\|y - z\|_i - r) \]
\end{proof}
\end{comment}
\begin{algorithm2e}[t]
\DontPrintSemicolon
\caption{{\sc SCRAM-GM}($x_t,y_t,\delta,\bareta,\alpha$)}
\label{alg:scram-gm}
	\KwIn{Input data $(x_t,y_t)_{t = 1}^n$.}
	\KwOut{Predictor $\wh{w}$.}
	    Shuffle the data and split into two equal sized groups $C_1,C_2$ and split $C_1$ into $k \triangleq \Theta(\log(1/\delta))$ equal-size buckets $B_1,\ldots,B_k$.\;
	    Run \textsc{SCRAM} with parameters $\bareta,\alpha$ on each bucket to get predictors $w_1,\ldots,w_k$.\;
	    Return the geometric median
	    \[ \wh w = \arg\min_y \sum_{i = 1}^k \|y - w_i\|_{\Sig_{C_2}} \]
	    where $\Sig_{C_2} \triangleq \frac{1}{|C_2|}\sum_{t \in C_2} x_t x_t^T$.
\end{algorithm2e}
\begin{theorem}\label{thm:geometricmedian}
Suppose that $\eta < 1/3$ is an upper bound on the contamination level, define $\beta$ as in \eqref{eqn:rho-defn}, and suppose for some $q \in [2,\infty],\sigma_q \ge 0$ and all $t$ that
\begin{equation}\label{eqn:q-noise3}
    \|\xi_t\|_{q} \triangleq \E{|\xi|^q}^{1/q} \le \sigma_q.
\end{equation}
Then provided $\eta = 0$ or %\todo{check and update this from the first theorem.}
\[ \eta \cdot n \gtrsim \log(\min(n,d)), \]
we can take
$\alpha = \Theta\left(\sqrt{\frac{\eta \log(d)\log(1/\delta)}{n}}\right)$ and $\overline{\eta} = \eta + \Theta(\eta \sqrt{\beta})$ such that the output $w$
of \textsc{SCRAM-GM}  satisfies
\begin{align*} 
\beta^{1 + 1/q} \norm{w^* - w}_{\Sig^*}
    &\lesssim \eta^{1 - 1/q} \sigma_q + \epsilon + \eta^{1/8} R^{1/2} (\epsilon +  \eta^{1/2 - 1/q} \sigma_q)^{1/2}\sqrt[8]{\frac{\log(\min(n,d))\log(1/\delta)}{n}} \\
    &\quad+  \eta^{1/4} R\sqrt[4]{\frac{\log(\min(n,d))\log(1/\delta)}{n}} 
     + \min \left\{\sigma\sqrt{\frac{d\log(1/\delta)}{n}}, (R\sigma)^{1/2}\sqrt[4]{\frac{\log(1/\delta)}{n}}\right\} \\
     &\quad+ R\sqrt{\frac{\log^3(n)\log(1/\delta)}{n}}
\end{align*}
with probability at least $1 - \delta$. 
\end{theorem}
\begin{proof}
%\todo{clarify $u^*$ vs $w^*$ here. (it never matters due to $\epsilon$ distances, but should be written more carefully.)}
    %By Lemma~\ref{lem:generalization} and the union bound, we have the bound
    %\[ \|w - u^*\|_{\Sig^*} \lesssim \|w - u^*\|_{\Sigma_i} + R\sqrt{\frac{\log^3(n) + \log(K/\delta)}{n}} \]
    %uniformly over all $\|w\| \le R$ and $i \in K$ with probability at least $1- \delta/4$. 
    Combining Lemma~\ref{lem:nonsos-lowprob} and Lemma~\ref{lem:generalization} gives
    \[ \|w_i - w^*\|_{\Sig^*} \le r \triangleq C(r_0 + R\sqrt{\frac{\log^3(n)\log(1/\delta)}{n}}) \]
    with probability at least $0.98$,
    where $r_0$ is the right hand side of \eqref{eqn:lowprob-guarantee} plus $\epsilon$ (to replace $u^*$ by $w^*$) and $C$ is an absolute constant. 
    Hence by applying Lemma~\ref{lem:nonsos-lowprob}, independence, and Hoeffding's inequality we see that
    \[ \#\{i : \beta^{1 + 1/q} \|w_i - w^*\|_{\Sig^*} \le r\} \ge 0.97 k \]
    with probability at least $1 - \delta$ where $r$ is the right hand side of \eqref{eqn:lowprob-guarantee}, including the constant factor. We condition on this event in what follows.

    Note that by Bernstein's inequality, for any particular $w$
    \[ \left|\|w - w^*\|_{\Sig_{C_2}}^2 - \|w - w^*\|_{\Sig^*}^2\right| \lesssim \|w - w^*\|_{\Sig^*}R\sqrt{\frac{\log(1/\delta)}{n}} + \frac{R^2 \log(1/\delta)}{n} \]
    with probability $1- \delta$,
    where we used that $\E{\langle w - w^*, X \rangle^4} \le 4R^2 \E{\langle w- w^*, X \rangle^2}$ to upper bound the variance term. Note that $R\sqrt{\log(1/\delta)/n} = O(r)$.
    Hence union bounding over $w_1,\ldots,w_k$ we find with probability at least $1 - \delta$
    \[ \#\{i : \beta^{1 + 1/q}\|w_i - w^*\|_{\Sigma_{C_2}} = O(r)\} \ge 0.97k \]
    which by Lemma~\ref{lem:geometric-median} gives the result in the norm $\|\cdot\|_{\Sigma_{C_2}}$ and combined with Lemma~\ref{lem:generalization} gives the desired result in $\|\cdot\|_{\Sig^*}$.
    %\todo{elaborate details.}%\todo{elaborate, also we need to do matrix concentration again. may be cleaner to do it in the stochastic setting so we don't have to do matrix concentration so many times.}
\end{proof}
%\todo{finish writing this section.}
\iffalse
\subsection{Details of oblivious model}
\begin{remark}\label{rmk:alternative-setup}

\end{remark}
\fi

%% file: sos_strikes_back.tex
\section{Optimal Breakdown Point via Sum of Squares Programming}\label{sec:sos}
\label{sec:sos_strikes_back}

As previously explained, the breakdown point for the estimator \textsc{SCRAM} is at $\eta = 1/3$, because when $\eta \ge 1/3$ the landscape of its objective exhibits bad local minima. Remarkably, if we instead use the natural degree-4 Sum of Squares relaxation of our original combinatorial optimization problem, it maintains the same statistical guarantees as \textsc{SCRAM} (including the optimal $\eta$ dependence) while also managing to escape the bad local minima of the nonconvex problem and achieve optimal breakdown point $\eta = 1/2$. 

As we will see in the analysis, the fundamental fact we use which is true for the SoS relaxation (Program~\ref{program:sos}) and not true for an arbitrary stationary point or local minima of Program~\ref{program:biconvex} is that the SoS relaxation always computes a lower bound on the original (unrelaxed) problem \eqref{eq:opt}, allowing us to compare objective values with the ground truth pair $(a^*,u^*)$.

In Section~\ref{subsec:sos_prelims}, we provide some preliminaries on sum-of-squares, and in Section~\ref{subsec:sos} we provide the main guarantees for our sum-of-squares-based algorithm.

\subsection{Preliminaries: Sum of Squares and Semidefinite Programming} 
\label{subsec:sos_prelims} 
\textbf{Pseudoexpectations:} The sum of squares SDP hierarchy is a series of increasingly tight SDP relaxations for solving polynomial systems $\calP \defeq \{p_i(x) \geq 0\}_{i=1}^N$.  Although it is in general NP-hard to solve polynomial systems, the level-$\ell$ SoS SDP attempts to approximately solve $\calP$ with increasing accuracy as $\ell$ increases by adding more constraints to the SDP.  This improvement in approximation naturally comes at the expense of increasing runtime and space.  

In particular, one can think of the SoS SDP as outputting a "distribution" $\mu$ over solutions to $\calP$.  However, there are two important caveats.  Firstly, one can only access the degree-$\ell$ moments of the "distribution" and secondly there may be no true distribution with the corresponding degree $\ell$ moments.  Thus we refer to $\mu$ as a pseudodistribution.

\begin{definition}
A degree $\ell$ pseudoexpectation $\pE : \R[x]_{\leq \ell} \to \R$ {\it satisfying $\calP$} is a linear functional over polynomials of degree at most $\ell$ satisfying 
\begin{enumerate}
\item(Normalization) $\psE*{1} = 1$, 
\item(Constraints of $\calP$) $\psE*{p(x) a^2(x)} \geq 0$ for all $p \in \calP$ and polynomials $a$ with $\deg(a^2 \cdot p) \le \ell$,
\item(Non-negativity on square polynomials)$\psE*{q(x)^2} \ge 0$ whenever $\deg(q^2) \le \ell$.
\end{enumerate}
\end{definition}

For any fixed $\ell \in \mathbb{N}$, given a polynomial system,one can efficiently compute a degree $\ell$ pseudo-expectation in polynomial time. 
\begin{fact} (\cite{Nesterov00}, \cite{Parrilo00}, \cite{Lasserre01}, \cite{Shor87}). For any $n$, $\ell \in \Z^+$, let $\pE_{\zeta}$ be degree $\ell$ pseudoexpectation satisfying a polynomial system $\calP$. Then the following set has a $n^{O(\ell)}$-time weak
separation oracle (in the sense of \cite{GLS1981}):
 \begin{align*}
 & \{\pE_\zeta(1, x_1, x_2, . . . , x_n)^{\otimes \ell}| \text{ degree } \ell \text{ pseudoexpectations } \pE_{\zeta} \text{ satisfying }\calP \}
\end{align*}

Using this separation oracle, the ellipsoid algorithm finds a degree $\ell$ pseudoexpectation in time $n^{O(\ell)}$, which we call the degree $\ell$ sum-of-squares algorithm. 
\end{fact}

To reason about the properties of pseudo-expectations, we turn to the dual object of sum-of-squares proofs.
%We will now describe the sum-of-squares SDP hierarchy, starting with sum-of-squares proofs.

\paragraph{Sum-of-Squares Proofs}
For any nonnegative polynomial $p(x): \R^d \rightarrow \R$, one could hope to prove its nonnegativity by writing $p(x)$ as a sum of squares of polynomials $p(x) = \sum_{i=1}^m q_i(x)^2$ for a collection of polynomials $\{q_i(x)\}_{i=1}^m$.  Unfortunately, there exist nonnegative polynomials with no sum of squares proof even for $d = 2$.  Nevertheless, there is a generous class of nonnegative polynomials that admit a proof of positivity via a proof in the form of a sum of squares.  The key insight of the sum of squares algorithm, is that these sum of squares proofs of nonnegativity can be found efficiently provided the degree of the proof is not too large. 

\begin{definition} (Sum of Squares Proof)
Let $\mathcal{A}$ be a collection of polynomial inequalities $\{p_i(x) \geq 0\}_{i=1}^m$.  A sum of squares proof that a polynomial $q(x) \geq 0$ for any $x$ satisfying the inequalities in $\mathcal{A}$ takes on the form 

\[
    \left(1+ \sum_{k \in [m']} b_k^2(x)\right) \cdot q(x) = \sum_{j\in [m'']} s_j^2(x) + \sum_{i \in [m]} a_i^2(x) \cdot p_i(x) 
\]
where $\{s_j(x)\}_{j \in [m'']},\{a_i(x)\}_{i \in [m]}, \{b_k(x)\}_{i \in [m']}$ are real polynomials.  If such an expression were true, then $q(x) \geq 0$ for any $x$ satisfying $\mathcal{A}$.  We call these identities sum of squares proofs, and the degree of the proof is the largest degree of the involved polynomials $\max \{\deg(s_j^2), \deg(a_i^2 p_i)\}_{i,j}$.  Naturally, one can capture polynomial equalities in $\mathcal{A}$ with pairs of inequalities.   We denote a degree $\ell$ sum of squares proof of the positivity of $q(x)$ from $\calA$ as $\calA \sststile{\ell}{x} \{q(x) \geq 0\}$ where the superscript over the turnstile denote the formal variable over which the proof is conducted.  This is often unambiguous and we drop the superscript unless otherwise specified.    
\end{definition}

Sum of squares proofs can also be strung together and composed according to the following convenient rules.  
\begin{fact}
For polynomial systems $\calA$ and $\calB$, if $\calA \sststile{d}{x} \{p(x) \geq 0\}$ and $\calB \sststile{d'}{x} \{q(x)\geq 0\}$ then $\calA \cup \calB \sststile{\max(d,d')}{x}\{p(x) + q(x) \geq 0\}$.  Also $\calA \cup \calB \sststile{dd'}{x} \{p(x)q(x) \geq 0\}$ 
\end{fact}

Sum of squares proofs yield a framework to reason about the properties of pseudo-expectations, that are returned by the SoS SDP hierarchy.  
%In particular, the following fact holds.
%
\begin{fact} (Informal Soundness)
If $\calA \sststile{r}{x} \{q(x) \ge 0\}$ and $\pE$ is a degree-$\ell$ pseudoexpectation operator for the polynomial system defined by $\calA$, then $\pE[q(x)]  \ge 0$.
\end{fact}

The following fact about pseudoexpectations will be particularly useful:
\begin{lemma}\label{lem:passnorm}
	For any psd matrix $\Sig$ which induces a norm $\norm{\cdot}_{\Sig}$, any vector $w^*$, and any degree-2 pseudoexpectation $\psE{\cdot}$ over $\R^d$-valued variable $w$, we have that \begin{equation}
		\norm{\psE{w} - w^*}_{\Sig}^2 \le \psE{\norm{w - w^*}_{\Sig}^2}. \label{eq:soscs}
	\end{equation}
\end{lemma}

\begin{proof}
	By the dual definition of $L_2$ norm, the left-hand side of \eqref{eq:soscs} can be written as \begin{equation}
	    \sup_{v\in\S^{d-1}}\iprod{\Sig v,\psE{w} - w^*}^2.
	\end{equation} For any $v\in\S^{d-1}$, \begin{equation}
		\iprod*{\Sig v,\psE{w} - w^*}^2 = \left(\psE{\iprod{\Sig v,w - w^*}}\right)^2 \le \psE{\iprod{\Sig v,w-w^*}^2} \le \psE{\norm{w - w^*}^2_{\Sig}},
	\end{equation} where the first inequality follows by the pseudoexpectation version of SoS Cauchy-Schwarz (see e.g. Lemma A.5 of \cite{barak2014rounding}). Therefore, taking the maximum over all $v \in S^{d - 1}$ proves the inequality.
\end{proof}

% We defer the use of other canonical sum of squares inequalities to Appendix \ref{app:sos-toolkit}.
\input{sos_toolkit}

\input{sos_newest}

%% file: sos_toolkit.tex
\paragraph{Useful SoS Inequalities} 
Here we present some useful inequalities captured by the sum of squares proof system.

\begin{fact} (Cauchy Schwarz)
Let $x_1,..,x_n,y_1,...,y_n$ be indeterminates, than 

$$\sststile{4}{} \big(\sum_{i \leq n}x_iy_i\big)^2 \leq \big(\sum_{i \leq n}x_i^2\big)\big(\sum_{i \leq n}y_i^2\big). $$
\end{fact}

\begin{fact} (Triangle Inequality) 
Let $x,y$ be $n$-length vectors of indeterminates, then 
$$\sststile{2}{} \norm{x + y}^2 \leq 2\norm{x}^2 + 2\norm{y}^2. $$
\end{fact}

\begin{fact}\label{fact:pseudocauchy}(Pseudoexpectation Cauchy Schwarz). 
Let $f(x)$ and $g(x)$ be degree at most $\ell \leq \frac{D}{2}$ polynomial in indeterminate $x$, then $$\pE[f(x)g(x)]^2 \leq \pE[f(x)^2]\pE[g(x)^2].$$  
\end{fact}

\begin{fact} (Spectral Bounds)
Let $A \in \R^{d \times d}$ be a positive semidefinite matrix with $\lambda_{max}$ and $\lambda_{min}$ being the  largest and smallest eigenvalues of $A$ respectively. Let $\pE$ be a pseudoexpectation with degree greater than or equal to $2$ over indeterminates $v = (v_1,...,v_d)$.  Then we have 
$$\sststile{2}{} \iprod{A,vv^T} \leq \lambda_{max} \norm{v}^2 $$
and 
$$\sststile{2}{} \iprod{A,vv^T} \geq \lambda_{min} \norm{v}^2. $$
\end{fact}

%% file: sos_newest.tex
%!TEX root = ./main.tex

\subsection{SoS Algorithm and Analysis}
\label{subsec:sos}

In this section we state the main guarantee for our algorithm when $\eta$ is large, as well as the result of combining this guarantee with the ones in Section~\ref{sec:nosos} to obtain a guarantee for the full range of possible $\eta$.

As in Section~\ref{sec:nosos}, the constants in our result must deteriorate slightly as we approach the (optimal) breakdown point $\eta = 1/2$, so we introduce a parameter $\rho$ which tracks the distance to $1/2$; as long as we are strictly bounded away from this point, $\rho$ is upper bounded by a constant and can be ignored.

We first state a result for \emph{bounded} noise.

\begin{theorem}\label{thm:sos}
	Suppose that the contamination rate is $\eta \in (0.3, 1/2)$, define $0 < \rho < 1$ by $\eta = \frac{1}{2 + 2\rho^2}$, and suppose
	%for any $\rho \in (0,1/2)$ and\footnote{We made no effort to optimize this constant; it is likely that with more work, we can improve this to $\eta < 1/2$.}
	\begin{equation}
		n \gtrsim \log(\min(n,d)/\delta).
		%\td{\Omega}\left(\frac{d}{(c^2k\eta)^{k}}\log\left(\sigma/\delta\eta\right)\right).
		\label{eq:nbound_sos}
	\end{equation}
	If the noise $\brc{\xi_t}$ satisfies $\abs{\xi_t} \le \sigma$ for all $t$ with probability 1, then there is a $\poly(n,d)$ algorithm which takes as input $(x_1,y_1),...,(x_n,y_n)$ and, with probability at least $1 - \delta$, outputs a vector $\td{w}$ which satisfies
% 	For and any $w'$ for which $\norm{w^* - w'}^2 \le O(\sigma^2 \eta^{1 - 2/k})$, we have that    
    \begin{equation}
		% \norm{\td{w} - w^*}^2_{\Sig_n} \le \frac{1}{\rho^4} \cdot O\left(k \sigma^2 \eta^{1 - 2/k} + \epsilon^2 + \sigma (R + \sigma)\sqrt{\frac{d}{n}\log(2/\delta)} + \rho^2 \alpha R^2\right)\label{eq:soswant}
		\rho^2\norm{\td{w} - w^*}_{\Sig_n} \lesssim \sigma + \epsilon + \rho R\sqrt[4]{\frac{\log(\min(n,d)/\delta)}{n}} + \min\brc*{\sigma \sqrt{\frac{d + \log(1/\delta)}{n}},(R\sigma)^{1/2}\rho\cdot\sqrt[4]{\frac{\log(1/\delta)}{n}}}
	\end{equation} for oblivious covariates and \begin{equation}
		\rho^2\norm{\td{w} - w^*}_{\Sig_n} \lesssim \sigma + \epsilon + \rho R\sqrt[4]{\frac{\log(\min(n,d)/\delta)}{n}} + (R\sigma)^{1/2}\rho\cdot\sqrt[4]{\frac{\log(1/\delta)}{n}}
	\end{equation} for the more general case of adaptive covariates.
	 % in the adaptive case.
	% In particular, by taking 
	% $\eta \triangleq \eta + \Theta([d\log(1/\delta)/n]^{1/k})$ 
	% %$\eta\triangleq \eta + \left(\frac{d/n}{2k^2 - 4k}\right)^{k/(2k-4)}$ 
	% and $\alpha \triangleq \Theta(\sqrt{\log(d/\delta)/n})$ , we get that
	% \begin{multline}
	% 	\norm{\td{w} - w^*}^2_{\Sig_n} \le \frac{1}{\rho^4} \cdot O\left(k\sigma^2\eta^{1-2/k} + \epsilon^2 + k\sigma^2[d\log(1/\delta)/n]^{1/k - 2/k^2} + \right. \\
	% 	\left. R\sigma\sqrt{(d/kn)\cdot \log(2/\delta)} + \rho^2 R^2 \sqrt{\log(d/\delta)/n} \right)\label{eq:soswant_opt}
	% \end{multline}
\end{theorem}

By a simple truncation argument, we can also obtain versions of this result for weakly $L_q$ and subgaussian noise. For brevity, we only state the latter:

\begin{theorem}\label{thm:sos_subgaussian}
	Let $\eta,\rho,n$ satisfy the hypotheses of Theorem~\ref{thm:sos}.
	If the noise $\brc{\xi_t}$ is $\sigma^2$-subgaussian, then there is a $\poly(n,d)$ algorithm which takes as input $(x_1,y_1),...,(x_n,y_n)$ and outputs a vector $\td{w}$ which satisfies
    \begin{equation}
        \rho^2\norm{\td{w} - w^*}_{\Sig_n} \lesssim \sigma\sqrt{\log(1/\rho)} + \epsilon + \rho R\sqrt[4]{\frac{\log(\min(n,d)/\delta)}{n}} + \sigma \sqrt{\frac{d + \log(1/\delta)}{n}}
	\end{equation} for oblivious covariates and \begin{equation}
		\rho^2\norm{\td{w} - w^*}_{\Sig_n} \lesssim \sigma\sqrt{\log(1/\rho)} + \epsilon + \rho R\sqrt[4]{\frac{\log(\min(n,d)/\delta)}{n}} + (R\sigma)^{1/2}\rho\cdot\sqrt[4]{\frac{\log(1/\delta)}{n}}
	\end{equation} for the more general case of adaptive covariates.
\end{theorem}

\begin{proof}
    For any $\delta'$, we know that each of the $\xi_t$ satisfy $\abs{\xi_t} \lesssim \sigma\sqrt{\log(1/\delta')}$ individually with probability $1 - \delta'$. If we treat indices $t$ for which this does not hold as corruptions and take  $\delta'$ to be $1/4 - \eta/2$, then we can take the the corruption level in Theorem~\ref{thm:sos} to be $1/4 + \eta/2$ and the bound on the noise to be $\sigma\sqrt{\log\left(\frac{4}{1 - 2\eta}\right)}$. Note that for $\eta\in(1/4,1/4)$, the $\rho$ corresponding to the new corruption level $1/4 + \eta/2$ is within a constant factor of $\rho$. Also note that $\sqrt{\log(1/\delta')} = O(\log(1/\rho))$. The result then follows by Theorem~\ref{thm:sos}.
\end{proof}

We now state the full guarantee obtained by combining the above with the results of Section~\ref{sec:nosos}. For brevity, we will only state the subgaussian case:

\begin{theorem}\label{thm:ultimate_regression}
    Let $0 \le \eta < 1/2$, and define $\rho > 0$ by $\eta = \frac{1}{2 + 2\rho^2}$, and suppose $n$ satisfies $n\gtrsim\log(\min(n,d)/\delta)$. If the noise $\brc{\xi_t}$ is $\sigma^2$-subgaussian, then there is a $\poly(n,d)$ algorithm which takes as input $(x_1,y_1),\ldots,(x_n,y_n)$ and, with probability at least $1 - \delta$, outputs a vector $w$ which satisfies \begin{align}
        \min(1,\rho^2)\norm{u^* - w}_{\Sig_n} &\lesssim c_{\delta,\eta,n} \eta \sigma + \eta^{1/2}\rho^2\epsilon + \eta^{1/8} R^{1/2} (\sqrt{c_{\delta,\eta,n} \eta}\sigma  + \epsilon)^{1/2}\sqrt[8]{\frac{\log(\min(n,d)/\delta)}{n}} \\
        &\quad+  \eta^{1/4} R\sqrt[4]{\frac{\log(\min(n,d)/\delta)}{n}} 
         + \min \left\{\sigma\sqrt{\frac{d + \log(2/\delta)}{n}}, (R\sigma)^{1/2}\sqrt[4]{\frac{\log(2/\delta)}{n}}\right\}
    \end{align} for oblivious covariates, where $c_{\delta,\eta,n}$ is defined in \eqref{eq:cdef}. In the more general case of adaptive covariates, $w$ satisfies
    \begin{align} 
        \min(1,\rho^2)\cdot \norm{u^* - w}_{\Sig_n}
        &\lesssim c_{\delta,\eta,n}\eta \sigma + \eta^{1/2}\rho^2\epsilon + \eta^{1/8} R^{1/2} (\sqrt{c_{\delta,\eta,n} \eta}\sigma  + \epsilon)^{1/2}\sqrt[8]{\frac{\log(\min(n,d)/\delta)}{n}} \\
        &\quad+  \eta^{1/4} R\sqrt[4]{\frac{\log(\min(n,d)/\delta)}{n}} 
         +  (R\sigma)^{1/2}\sqrt[4]{\frac{\log(2/\delta)}{n}}, \label{eq:ultimate_subgauss_adaptive}
    \end{align} i.e. the same bound except the last term was changed. Recall that $u^*$ here is the best norm-$R$ linear predictor of the uncorrupted and unnoised data, that is, \begin{equation}
	u^* \triangleq \arg\min_{u : \|u\| \le R} \frac{1}{n}\sum_t (y^*_t - \iprod{u,x_t})^2. \label{eq:ustardef2}
\end{equation}
\end{theorem}

\begin{proof}
    If $0 \le \eta < 0.3$, apply Theorem~\ref{thm:nonsos-subgaussian}, noting that the parameter $\beta$ in that theorem is an absolute constant for this range of $\eta$. Otherwise, apply Theorem~\ref{thm:sos_subgaussian}, noting that $\eta = \Theta(1)$ in this case, and that $\norm{u^* - w}_{\Sig_n}$ and $\norm{w^* - w}_{\Sig_n}$ differ by $O(\epsilon)$.
\end{proof}

\subsubsection{Sum-of-Squares Program and Feasibility}

\begin{algorithm2e}
\DontPrintSemicolon
\caption{\textsc{SoSRegression}($D$)}
\label{alg:sos}
	\KwIn{Dataset $D = \brc{(x_1,y_1),\ldots,(x_n,y_n)}$}
	\KwOut{Vector $\td{w}$ for which $\norm{\td{w} - w^*}_{\Sig_n}$ is small (see Theorem~\ref{thm:sos})}
	    Let $\psE{\cdot}$ be the pseudoexpectation optimizing Program~\ref{eq:opt}.\;
	    \Return{$\psE{w}$}.\;
\end{algorithm2e}

\noindent We will condition on the events of Lemma~\ref{lem:regularity}. Now consider the following set of polynomial constraints.

\begin{program}\label{program:sos}
%Let $c,k$ be the parameters from Assumption~\ref{assume:hypernoise}, 
Let $\alpha>0$ be a parameters to be tuned later. The program variables are $\{a_t\}_{t\in[n]}$ and $w$, and the constraints are
\begin{enumerate}
	\item \label{constraint:normbound}(Norm bound) $\sum^d_{i=1} w_i^2 \le R^2$.
	\item \label{constraint:boolean}(Booleanity) $a_t^2 = a_t$ for all $t\in[n]$.
	\item \label{constraint:fraction} (Large fraction of inliers) $\frac{1}{n}\sum^n_{t=1} a_t \ge 1 - \eta - \alpha$.
	%\item \label{constraint:hyper}(Hypercontractive Residuals) For all even integers $2\le \ell \le k$, \begin{equation}\frac{1}{n}\sum^n_{t=1} a_t (y_t - \iprod{w,x_t})^{\ell} \le (2c\sigma\ell^{1/2})^{\ell}\end{equation}
	\item \label{constraint:sub}(Outliers sub-sample the empirical covariance\footnote{One can use matrix inequalities in SoS: see e.g. Section 7.1 in \cite{hopkins2018mixture}.}) \begin{equation}
		\frac{1}{n}\sum^n_{t=1} (1 - a_t)x_t x_t^{\top} \preceq \eta\Sig_n + \alpha\cdot \Id.
	\end{equation}
	% \item (First-order stationarity) This is \eqref{eq:firstorder} with $\epsilon_{grad} = 0$, explicitly for all $\|v\| \le R$ we have
	% \label{constraint:gradient-zero}
	% \[ \frac{1}{n}\sum_t a_t(y_t - \iprod{w,x_t})\iprod{x_t, v - w} = 0.\]
	% This constraint is only needed for the AltMin-type analysis, i.e. for $\eta$ small. 
%	\todo{make small modification to argument below to handle this constraint.} %done
\end{enumerate}
The program objective is to minimize 
\[ \min \psE*{\sum_{t = 1}^n a_t (y_t - \iprod{w, x_t})^2} \]
over degree-4 SoS-pseudoexpectations satisfying the above constraints.
\end{program}

We first show that conditioned on the events of Lemma~\ref{lem:regularity} holding, there always exists a feasible solution to the above polynomial system.

\begin{lemma}[Satisfiability]\label{lem:sat}
	For any $\delta > 0$, if $n$ satisfies the bound in \eqref{eq:nbound_sos}, then for any sequence of $x_1,...,x_n$ chosen during the process in Definition~\ref{defn:huberreg}, we have that with probability at least $1 - \delta$ over the randomness of the $\Ber(\eta)$ coins generating $a^*_1,...,a^*_n$ and over the randomness of $\xi_1,...,\xi_n$, the choice of $a_t = a^*_t$ and
	\begin{equation}\label{eqn:v-as-optimum}
	v = \arg\min_{\|v\| \le R} \sum_{t = 1}^n a_t (y_t - \langle v, x_t \rangle)^2 
	\end{equation}
	is a feasible solution to Program~\ref{program:sos}.
	%, and in particular, the objective value of Program~\ref{program:sos} is at most
	As a consequence, for any $\|v\| \le R$ the objective value of Program~\ref{program:sos} is at most
	\[ \frac{1}{n} \sum_{t = 1}^n a^*_t(y_t - \langle v, x_t \rangle)^2. \]
\end{lemma}

\begin{proof}
	%Take $a_t = a^*_t$ and $w = w^*$. 
	Clearly Constraints~\ref{constraint:normbound} and \ref{constraint:boolean} are satisfied. Part~\ref{item:manygood} of Lemma~\ref{lem:regularity} implies that Constraint~\ref{constraint:fraction} is satisfied with probability $1 - \delta/3$. Part~\ref{item:sigconc} of Lemma~\ref{lem:regularity} implies that Constraint~\ref{constraint:sub} is satisfied with probability at least $1 - \delta/3$. Finally, the first-order stationarity condition is satisfied because $v$ is the optimizer of \eqref{eqn:v-as-optimum}. 
	
	To get the consequence, we use that such an upper bound holds with $v$ the minimizer of \eqref{eqn:v-as-optimum} by feasibility of $(a^*,v)$, and then use the fact that it is the minimizer to extend to conclusion to all (not necessarily first-order stationary) $v$. 
\end{proof}

\subsubsection{Bounding Clean Square Loss}

We now proceed to the sum-of-squares proof that the constraints of Program~\ref{program:sos} imply a bound on the clean square loss achieved by $w$, under the degree-4 SoS proof system.

Let $v^*$ be defined as
\begin{equation}\label{eqn:v*}
v^* \triangleq \arg\min_{v : \|v\| \le R} \frac{1}{n} \sum_{t = 1}^n a^*_t (y^*_t - \langle v, x_t \rangle)^2.
\end{equation}
The following Lemma is needed only for the misspecified setting: if $\epsilon = 0$ we will trivially have $v^* = w^*$. In the misspecified setting $v^*$ will naturally appear in the analysis, instead of $w^*$, because it gives the optimal bounded norm linear function approximating the true regression function $x_t \mapsto \langle w^*, x_t \rangle + \epsilon_t$. We define $\Sig'_n \triangleq \frac{1}{n}\sum a^*_t\cdot x_tx_t^{\top}$.
\begin{lemma}\label{lem:misspecification}
For $v^*$ as defined above, we have $\|v^* - w^*\|_{\Sig'_n}^2 = O(\epsilon^2)$ and also, if we define 
\[ \epsilon'_t \triangleq y^*_t - \langle v^*, x_t \rangle, \]
then for all $w$ with $\|w\| \le R$ we have:
\begin{equation}\label{eqn:first-order}
\sum_{t = 1}^n a^*_t \epsilon'_t \langle w - v^*, x_t \rangle \le 0 
\end{equation}
\end{lemma}
\begin{proof}
Since $\nabla_v (y^*_t - \langle v^*, x_t \rangle)^2 = -2(y^*_t - \langle v^*, x_t \rangle)x_t$, we see that the first order optimality condition for \eqref{eqn:v*} implies for any $w$ with $\|w\| \le R$ we have
\[ \frac{-2}{n} \sum_{t = 1}^n a^*_t \epsilon'_t \langle w - v^*, x_t \rangle  \ge 0 \]
which gives \eqref{eqn:first-order}. 

It remains to upper bound $\|v^* - w^*\|_{\Sigma'}^2$. By writing it out, we see
\[ \|v^* - w^*\|_{\Sigma^*_t}^2 = \frac{1}{n} \sum_{t = 1}^n a^*_t \langle v^* - w^*, x_t \rangle^2 = \frac{1}{n} \sum_{t = 1}^n a^*_t (y^*_t - \epsilon_t - \langle v^*, x_t \rangle)^2 \le \frac{2}{n} \sum_{t = 1}^n a^*_t (\epsilon_t^2 + (\epsilon'_t)^2) \le 2 \epsilon^2 \]
where in the second-to-last step we used $(a + b)^2 \le 2a^2 + 2b^2$ and in the last step we used that $v^*$ minimizes \eqref{eqn:v*}.
\end{proof}

We can now prove Theorem~\ref{thm:sos}.

\begin{proof}[Proof of Theorem~\ref{thm:sos}]
% 	Because $\norm{x_t} \le 1$ for every $t$, $\Sig_n \preceq \Id$, so $\norm{w^* - w'}_{\Sig_n} \le \norm{w^* - w'} \le O(\sigma\eta^{1/2 - 1/k})$ and it therefore suffices to show \eqref{eq:soswant} when $w' = w^*$.
    Let $\psE{\cdot}$ be the pseudo-expectation optimizing the objective in Program~\ref{program:sos}, and define $\td{w}\triangleq \psE{w}$.
	By part~\ref{item:sigconc} of Lemma~\ref{lem:regularity} and Constraint~\ref{constraint:normbound}, we have that
	\begin{equation}
	    (1 - \eta)\norm{\td{w} - w^*}^2_{\Sig_n} \le \norm{\td{w} - w^*}^2_{\Sig'_n} + \alpha\norm{\td{w} - w^*}^2 \le \psE{\norm{w - v^*}^2_{\Sig'_n}} + \alpha R^2 + 2\epsilon^2,
	\end{equation}
	where $\Sig'_n \triangleq \frac{1}{n}\sum a^*_t\cdot x_tx_t^{\top}$ and $\norm{\cdot}_{\Sig'_n}$ is the induced norm, and in the last step we used the first part of Lemma~\ref{lem:misspecification}, Lemma~\ref{lem:passnorm}, and Constraint~\ref{constraint:normbound}.
	
	We can further bound \begin{align}
		\MoveEqLeft \psE{\norm{w - v^*}^2_{\Sig'_n}} \\
		&= \frac{1}{n}\sum^n_{t=1} a^*_t \psE{\iprod{w - v^*,x_t}^2} \\
		&= \frac{1}{n}\sum^n_{t=1} a^*_t \psE*{(y_t - \iprod{w,x_t}) - (y_t - \iprod{v^*,x_t})^2}\\
		&= \frac{1}{n}\sum^n_{t=1} a^*_t\brk*{\psE*{(y_t - \iprod{w,x_t})^2} - (y_t - \iprod{v^*,x_t})^2} + \frac{2}{n}\sum^n_{t=1}a^*_t (y_t - \iprod{v^*,x_t})\cdot \iprod*{\psE{w} - v^*,x_t} \\
		&= \underbrace{\frac{1}{n}\sum^n_{t=1} a^*_t\brk*{\psE*{(y_t - \iprod{w,x_t})^2} - (y_t - \iprod{v^*,x_t})^2}}_{\circled{1}} + \underbrace{\iprod*{\psE{w} - v^*,\frac{2}{n}\sum^n_{t=1} a^*_t (\xi_t + \epsilon'_t) x_t}}_{\circled{2}}
		% &= \underbrace{\psE*{\frac{1}{n}\sum^n_{t=1}a^*_t(y_t - \iprod{w,x_t})^2}}_{\circled{1}} - \underbrace{\frac{1}{n}\sum^n_{t=1}a^*_t \xi_t^2}_{\circled{2}} - \underbrace{\iprod*{\psE{w} - w^*, \frac{2}{n}\sum^n_{t=1}a^*_t\xi_t x_t}}_{\circled{3}}.
	\end{align}
	where in the fourth step we used the identity $(a - b)^2 = a^2 - b^2 - 2b(a-b)$ and $\epsilon'_t := y_t - \xi_t - \langle v^*, x_t \rangle$ as defined in Lemma~\ref{lem:misspecification}.

    % (a-b)^2 = a^2 - b^2 - 2b(a - b)

%	By the second part of Lemma~\ref{lem:regularity}, we can lower bound $\circled{2}$ by $\sigma^2 - O(k\sigma^2\eta^{1-2/k})$.

    Because of Lemma~\ref{lem:misspecification} and $\|\psE{w}\|^2 \le R^2$ from Constraint~\ref{constraint:normbound} we know that
    \[ \langle \psE{w} - v^*, \frac{2}{n} \sum_{t = 1}^n a^*_t \epsilon'_t x_t \rangle \le 0 \]
    so we can drop this term from \circled{2}.
	Then by part~\ref{item:vanish} of Lemma~\ref{lem:regularity}, together with Cauchy-Schwarz,
	\begin{equation}
		% \circled{2} \le \frac{2R}{n}\norm*{\sum^n_{t=1}a^*_t\xi_t x_t} \le O\left(\sigma R\sqrt{\frac{d}{n}\log(n/\delta)}\right),
		\circled{2} \le 2\sigma\left(\lambda\norm*{\psE{w} - v^*}_{\Sig'_n} + \lambda'\norm*{\psE{w} - v^*}\right) \le O\left(\norm*{\psE{w} - v^*}_{\Sig'_n}\sigma\lambda + R\sigma\lambda'\right).
	\end{equation}

	It remains to upper bound $\circled{1}$, and this is the bulk of the analysis. Concretely, we need to show that the constraints of the program SoS-imply an upper bound on the quantity $\frac{1}{n}\sum^n_{t=1}a^*_t(y_t - \iprod{w,x_t})^2 - \frac{1}{n}\sum^n_{t=1}a^*_t (y_t - \iprod{w^*,x_t}^2)$ of $c\|w - v^*\|_{\Sig'_n}^2 + O(\cdot)$ with $c \in [0,1)$, so that we can solve for an upper bound on $\|w - v^*\|_{\Sig'_n}^2$. We do so in Lemma~\ref{lem:sos} below and get $c = \frac{(1 + \rho^2)\overline{\eta}}{1 - \overline{\eta}}$. Choosing $\rho$ to be the solution to $\overline{\eta} = \frac{1}{2 + 2\rho^2}$ and observing that
	\[ \frac{1}{1 - c} = \frac{1 - \overline{\eta}}{1 - (2 + \rho^2)\overline{\eta}} = \frac{1 + 2\rho^2}{\rho^2} = 2 + 1/\rho^2, \]
	we get that \begin{equation}
		\norm*{\psE{w} - v^*}_{\Sig'_n}^2 \le O(1/\rho^2)\cdot \left(\norm*{\psE{w} - v^*}_{\Sig'_n}\sigma\lambda + \calE\right)
	\end{equation} for $\calE \triangleq R\sigma\lambda' + \frac{\sigma^2 +\epsilon^2}{\rho^2} + \alpha R^2$. We do case analysis based on which of the two terms on the right-hand side dominates:
	\begin{enumerate}
		\item If the former dominates, then the bound simplifies to \begin{equation}
			\norm*{\psE{w} - v^*}_{\Sig'_n} \lesssim \sigma\lambda/\rho^2 %\le O\left(\frac{\sigma}{\rho^2}\cdot \sqrt{\frac{d + \log(1/\delta)}{n}}\right)
		\end{equation}
		\item Otherwise, if $\calE$ dominates, then after taking a square root, the bound can be rewritten as \begin{equation}
			\norm*{\psE{w} - v^*}_{\Sig'_n} \lesssim \rho^{-1}\cdot\left( (R\sigma\lambda')^{1/2} + \frac{\sigma + \epsilon}{\rho} + \alpha^{1/2} R \right) %\le \frac{\epsilon + \sigma}{\rho^2} + \frac{\alpha^{1/2} R}{\rho} + \frac{R^{1/2}\sigma^{1/2}}{\rho}\cdot \sqrt{\frac{\log(1/\delta)}{n}}
		\end{equation}
	\end{enumerate}
	In either case, we conclude that \begin{equation}
		\norm*{\psE{w} - v^*}_{\Sig'_n} \lesssim \sigma\lambda/\rho^2 + \rho^{-1}\cdot\left( (R\sigma\lambda')^{1/2} + \frac{\sigma + \epsilon}{\rho} + \alpha^{1/2} R \right).
	\end{equation}
	If the covariates are adaptively chosen, we get \begin{equation}
		\norm*{\psE{w} - v^*}_{\Sig'_n} \lesssim \frac{R^{1/2}\sigma^{1/2}}{\rho}\cdot \sqrt[4]{\frac{\log(1/\delta)}{n}} + \frac{\sigma + \epsilon}{\rho^2} + \frac{\alpha^{1/2} R}{\rho}.
	\end{equation}
	If the covariates are obliviously chosen, then we could also obtain \begin{equation}
		\norm*{\psE{w} - v^*}_{\Sig'_n} \lesssim \frac{\sigma}{\rho^2}\cdot \sqrt{\frac{d + \log(1/\delta)}{n}} + \frac{\sigma + \epsilon}{\rho^2} + \frac{\alpha^{1/2} R}{\rho}.
	\end{equation}
	Plugging in $\alpha = \Theta\left(\sqrt{\eta\log(\min(n,d)/\delta)}{n}\right)$ as in Section~\ref{sec:nosos} completes the proof.
\begin{comment}%obsolete
	Finally, to get \eqref{eq:soswant_opt}, we use the fact that for constants $A,B>0$, if we choose $\eta = \Max{\eta}{\left(\frac{B}{(2 - 4/k)A}\right)^{k/(2k - 4)}}$ then \begin{equation}
		A\cdot \eta^{2-4/k} \le \Max{(A\cdot \eta^{2-4/k})}{\frac{Bk}{2k - 4}} \le \Max{(A\cdot \eta^{2-4/k})}{B}
	\end{equation}
	\begin{equation}
		 B\cdot\log(1/\eta) \le \frac{Bk}{2k - 4}\cdot \log\left(\frac{2A(k-2)}{Bk}\right) \le B\log(2A/B).
	\end{equation}
	By taking $A = k^2$, $B = d/n$, and $\eta = \Max{\eta}{\eta^*}$, we can bound the right-hand side of \eqref{eq:soswant} (rescaled by $\rho^4$) by \begin{align}
		\rho^4 \norm{\td{w} - w'}^2_{\Sig_n} &\le O\left((\sigma^2\cdot A\cdot \eta^{2 - 4/k})^{1/2} + (\sigma R B\cdot\log(1/\eta))^{1/2} + \epsilon^2 + \rho^2 \alpha R^2 + \sigma R\sqrt{(d/n)\cdot \log(n/\delta)} \right) \\
		&\le O\left(\sqrt{\sigma^4 A \cdot \eta^{2-4/k} + \sigma^2 R^2 B\cdot\log(1/\eta)} + \epsilon^2 + \rho^2 \alpha R^2 + \sigma R(\sqrt{(d/n)\cdot \log(n/\delta)})\right) \\
		&\le O\left(\Max{(\sigma^2\sqrt{A}\cdot \eta^{1-2/k})}{(\sigma^2 \sqrt{B})} + \sigma R\sqrt{B\cdot \log(2A/B)} + \epsilon^2 + \rho^2 \alpha R^2 + \sigma R(\sqrt{(d/n)\cdot \log(n/\delta)})\right) \\
		&\le O\left(\Max{(k\sigma^2\eta^{1-2/k})}{(\sigma^2 \sqrt{d/n})} + \sigma R\sqrt{(d/n)\cdot \log(k^2n^2/d\delta)}+ \epsilon^2 + \rho^2 \alpha R^2\right) \\
		&\le O\left(k\sigma^2\eta^{1-2/k} + \sigma(\sigma+R)\sqrt{(d/n)\cdot \log(kn/d\delta)}+ \epsilon^2 + \rho^2 \alpha R^2\right)
	\end{align} as claimed.
\end{comment}
\end{proof}

\begin{lemma}\label{lem:sos}
	Conditioned on the four parts of Lemma~\ref{lem:regularity} holding,
	%there is a degree-$O(k)$ SoS proof of the inequality 
	we have for any $\rho \in (0,1]$ that
	\begin{multline}
		% \psE*{\frac{1}{n}\sum^n_{t=1}a^*_t(y_t - \iprod{w,x_t})^2} \le \frac{1}{n} \sum_{t = 1}^n a^*_t (y_t - \langle v^*, x_t \rangle)^2 + \frac{(1 + 2\rho^2)\eta}{1 - \eta} \|v^* - w\|_{\Sig'_n}^2 + \\
		% O\left(\frac{k}{\rho^2} \sigma^2 \eta^{1 - 2/k} + \frac{1}{\rho^2} \epsilon^2 + \alpha R^2\right)\label{eq:sos_goal}
		\psE*{\frac{1}{n}\sum^n_{t=1}a^*_t(y_t - \iprod{w,x_t})^2} \le \frac{1}{n} \sum_{t = 1}^n a^*_t (y_t - \langle v^*, x_t \rangle)^2 + \frac{(1 + 2\rho^2)\eta}{1 - \eta} \|v^* - w\|_{\Sig'_n}^2 + \\
		O\left(\frac{\sigma^2 +\epsilon^2}{\rho^2} + \alpha R^2\right)\label{eq:sos_goal}
	\end{multline} 
	as long as $\psE*{\cdot}$ is a SoS degree-4 pseudoexpectation satisfying the constraints of the program.
	%using the constraints of the program.
\end{lemma}

\begin{proof}
	Let $\circled{*}$ denote the quantity inside the pseudoexpectation on the left-hand side of \eqref{eq:sos_goal}.  Then in the SoS degree-4 proof system we can show the following bound
	\begin{align}
		\circled{*} &= \frac{1}{n}\sum^n_{t=1} a^*_ta_t(y_t - \langle w,x_t\rangle)^2 + \frac{1}{n}\sum^n_{t=1} a^*_t(1-a_t)(y_t - \langle w,x_t\rangle)^2 \\
		&\leq \frac{1}{n}\sum^n_{t=1} a_t(y_t - \langle w,x_t\rangle)^2 + \frac{1}{n}\sum^n_{t=1} a^*_t(1-a_t)(y_t - \langle w,x_t\rangle)^2 \\
		%&\leq \sigma^2 + \frac{1}{n}\sum^n_{t=1} a^*_t(1-a_t)(y_t - \langle w,x_t\rangle)^2 \\
		&= \frac{1}{n}\sum^n_{t=1} a_t(y_t - \langle w,x_t\rangle)^2 + \frac{1}{n}\sum^n_{t=1} a^*_t(1-a_t)(y_t-\epsilon'_t -  \langle v^*,x_t\rangle +  \langle v^* - w,x_t\rangle + \epsilon'_t)^2 \\
		&\le \frac{1}{n}\sum^n_{t=1} a_t(y_t - \langle w,x_t\rangle)^2 +
		  \frac{2 + 1/\rho^2}{n}\sum^n_{t=1} a^*_t(1-a_t)(y_t - \epsilon'_t - \langle v^*,x_t\rangle)^2 \\
		& \qquad +  \frac{1 + 2\rho^2}{n}\sum^n_{t=1} a^*_t(1-a_t)\langle v^* - w,x_t\rangle^2 + \frac{2 + 1/\rho^2}{n} \sum^n_{t=1}a^*_t(1 - a_t)(\epsilon'_t)^2 \\
		&\le \frac{1}{n}\sum^n_{t=1} a_t(y_t - \langle w,x_t\rangle)^2 + \frac{2 + 1/\rho^2}{n}\sum^n_{t=1} a^*_t(1-a_t)\xi^2_t + \\
		&\ \qquad \qquad \qquad\qquad \qquad \qquad\qquad \qquad \qquad \frac{1 + 2\rho^2}{n}\sum^n_{t=1} a^*_t(1-a_t)\langle v^* - w,x_t\rangle^2 + (2 + 1/\rho^2)\epsilon^2 
	\end{align}
	where in the second step we use Constraint~\ref{constraint:boolean} to get $a^*_ta_t\le a_t$,
	%, in the third step we use Constraint~\ref{constraint:hyper} for $\ell = 2$, 
	in the fourth step we use the SOS Cauchy-Schwartz inequality to show $(a + b + c)^2 = (\rho a/\rho + b + \rho c/\rho)^2 \le (1 + 2\rho^2) (a^2/\rho^2 + b^2 + c^2/\rho^2)$, and in the fifth step we used that $\sum_{t = 1}^n a^*_t (\epsilon'_t)^2 \le \sum_{t = 1}^n a^*_t \epsilon_t^2 \le \epsilon^2$ by construction (see \eqref{eqn:v*}). %, in the sixth step we use Holder's, and in the last step we use Constraint~\ref{constraint:hyper} for $\ell = k$.
	
	Therefore, we can upper bound $\psE*{\circled{*}}$ by
	\begin{multline}\label{eqn:three-sums}
		\underbrace{\psE*{\frac{1}{n}\sum^n_{t=1} a_t(y_t - \langle w,x_t\rangle)^2}}_{\circled{I}} + \underbrace{\frac{2 + 1/\rho^2}{n}\sum^n_{t=1} a^*_t(1- \psE{a_t})\xi^2_t}_{\circled{II}} + \\ \underbrace{\psE*{\frac{1 + 2\rho^2}{n}\sum^n_{t=1} a^*_t(1-a_t)\langle v^* - w,x_t\rangle^2}}_{\circled{III}} + (2 + 1/\rho^2)\epsilon^2. 
	\end{multline}
	From the last part of Lemma~\ref{lem:sat},
	we know $\circled{I} \le \frac{1}{n} \sum_{t = 1}^n a^*_t(y_t - \langle v^*, x_t \rangle)^2$.
	% By (scalar) H\"older's,\footnote{We are using degree-$k$ H\"older's, but outside of the pseudoexpectation, so $\psE{\cdot}$ does not need to be degree-$k$ (indeed, this allows us to obtain guarantees even when the noise is only $k = (2+\varepsilon)$-hypercontractive).}
	% \begin{align}
	% \circled{II} &\leq (2 + 1/\rho^2)\left(\frac{1}{n}\sum^n_{t=1} (1-\psE{a_t})^{k/(k - 2)}\right)^{1 - 2/k} \left(\frac{1}{n} \sum^n_{t=1} a^*_t|\xi_t|^k\right)^{2/k} \\
	% 		&\leq (2 + 1/\rho^2)\left(\frac{1}{n}\sum^n_{t=1} (1-\psE{a_t})\right)^{1 - 2/k} \left(\frac{1}{n} \sum^n_{t=1} a^*_t|\xi_t|^k\right)^{2/k} \\
	% 	&=  O((k/\rho^2) \cdot \eta^{1 - 2/k} \sigma^2)
	% \end{align}
	% where we used in the second inequality that the vector $v := (1 - \psE*{a_t})_{t = 1}^n$ has entries in $[0,1]$ (from Constraint~\ref{constraint:boolean}), in the third inequality that $\|v\|_1 \le 2\eta n$ (from Constraint~\ref{constraint:fraction}), and also in the third inequality we used part~\ref{item:hyper} of Lemma~\ref{lem:regularity}.
	And as we are in the bounded noise setting, we can upper bound $\circled{II}$ by $(2+1/\rho^2)\cdot\sigma^2(\eta + \alpha) \le O(\sigma^2/\rho^2)$, where in the last inequality we used that $\eta$ is upper and lower bounded by absolute constants by assumption.

	Finally, to bound $\circled{III}$, we can finally apply Constraint~\ref{constraint:sub}. We get that
	\begin{align}
		\frac{1 + 2\rho^2}{n}\sum^n_{t=1} a^*_t(1-a_t)\langle v^* - w,x_t\rangle^2 
		&\le \frac{1 + 2\rho^2}{n}\sum^n_{t=1} (1-a_t)\langle v^* - w,x_t\rangle^2 \\
		&\le \frac{(1 + 2\rho^2)\eta}{n}\sum^n_{t=1} \langle v^* - w,x_t\rangle^2 + 3\alpha\norm{v^* - w}^2_2 \\
		&\le  \frac{(1 + 2\rho^2)\eta}{(1 - \eta)n}\sum^n_{t=1} a^*_t\langle v^* - w,x_t\rangle^2 + \frac{3\eta \alpha R^2}{1 - \eta} + 3\alpha R^2 \\
		&= \frac{(1 + 2\rho^2)\eta}{(1 - \eta)} \|v^* - w\|_{\Sig'_n}^2 + O(\alpha R^2),
		% OLD: UNECESSARY STEPS
		%&=  \frac{3\eta}{(1 - 2\eta)n}\sum^n_{t=1} a^*_t(\langle w^*,x_t\rangle - y_t + y_t - \langle w,x_t\rangle)^2 + \frac{3\alpha(1-\eta) R^2}{1 - 2\eta} \\
		%&\leq  \frac{6\eta}{(1 - 2\eta)n}\sum^n_{t=1} a^*_t(\langle w^*,x_t\rangle - y_t)^2 + \frac{6\eta}{(1 - 2\eta)}\cdot\circled{*} + \frac{3\alpha(1-\eta) R^2}{1 - 2\eta} \\
		%&\leq  \frac{6\eta\sigma^2}{1 - 2\eta} + \frac{6\eta}{(1 - 2\eta)}\cdot\circled{*} + \frac{3\alpha(1-\eta) R^2}{1 - 2\eta},
	\end{align}
	where the second step follows by Constraint~\ref{constraint:sub} and $\rho \le 1$, the third step follows by part~\ref{item:sigconc} of Lemma~\ref{lem:regularity} which we are conditioning on in this section, and the fourth step uses the definition of $\Sig'_n$ together with the assumption that $\eta$ is at least some absolute constant.
	%UNECESSARY
	%fifth fourth step follows by the basic inequality $(a + b)^2 \le 2a^2 + 2b^2$ again, and the final step follows by part~\ref{item:hyper} of Lemma~\ref{lem:regularity} for $\ell = 2$. The desired bound then follows by taking pseudoexpectations, rearranging, and using the assumption that $\eta\le 1/7$. %$\circled{*}$ then follows from rearranging.
\end{proof}

%% file: new_ellipsoid.tex
%!TEX root=./main.tex

\section{Online Regression}
\label{sec:online}

\subsection{Cutting Plane Algorithm}
In this section we leverage the guarantees of Section~\ref{sec:nosos} to design an efficient algorithm for Huber-contaminated online regression. For brevity, in this section we restrict our attention to the case of sub-Gaussian noise, though our techniques extend easily to handle $k$-hypercontractive noise. 

The basic trick we use is to combine the offline regression oracle with a cutting plane method, so that we can keep efficiently cutting down the space of linear predictors until we find one near $w^*$. Essentially, the algorithm collects a large batch of samples, compares it's current performance on this batch to the optimal robust regression result in hindsight (estimated by {\sc SCRAM}), and if it finds its performance is poor it cuts out a large set of possible predictors and updates to use a new predictor.

The algorithm, which we will refer to as {\sc AMCutter}, can be based upon any central cutting-plane optimization method like ellipsoid or Vaidya's algorithm; here we use Vaidya's algorithm since it is oracle-efficient. More specifically, we recall the following guarantee for Vaidya's algorithm:

\begin{theorem}[\cite{vaidya1989new}, see e.g. Section 2.3 of \cite{bubeck2014convex}]\label{thm:vaidya}
Suppose that $\mathcal{K}$ is an (unknown) convex body in $\mathbb{R}^d$ which contains a Euclidean ball of radius $r > 0$ and is contained in a Euclidean ball centered at the origin of radius $R > 0$. There exists an algorithm which, given access to a separation oracle for $\mathcal{K}$, finds a point $x \in \mathcal{K}$, runs in time $poly(\log(R/r),d)$, and makes $O(d\log(Rd/r))$ calls to the separation oracle. 
\end{theorem}

Now we describe the algorithm. $C_0$ and $N_0$ are constants to be determined later.
%, $K_1$ is something like $\log T$. 
{\sc SeparationOracle} (see Algorithm~\ref{alg:seporacle}) implements the separation oracle (which is also where most of the interaction with Nature occurs). Here the input $w$ lies in $\mathcal{W} = \{w : \|w\| \le R\}$ and Nature's inputs are $x_t$ with $\|x_t\| \le 1$. Finally, we note that if {\sc SeparationOracle} gets to the final round $T$ of the online regression problem, then it may not return to Vaidya's algorithm (so step 2 of {\sc AMCutter} is never reached), but as we will see, even if this happens the algorithm still achieves the correct regret bound.

\begin{algorithm2e}[t]
\DontPrintSemicolon
\caption{{\sc SeparationOracle}($w,x_t,C_0,D$)}
\label{alg:seporacle}
	\KwIn{Vector $w\in\calW$}
	\KwOut{Separating hyperplane between $w$ and the target region $\brc{w': \norm{w' - w^*} \le r}$, if $w$ lies outside}
	    $D\gets\emptyset$.\;
	    \For{each new point $x_t$ input by Nature}{
	        Predict $\wh{y}_t = \iprod{w,x_t}$ and observe $y_t$.\;
	        Append $(x_t,y_t)$ to $D$.\;
	        $v_t\gets${ \sc SCRAM}($D$).\;
	        $\Sig_t \gets \frac{1}{|D|}\sum_{(x_t,y_t)\in D}x_tx_t^{\top}$. Define $\varphi_t(u)\triangleq \norm{u - v_t}^2_{\Sig_t}$.\;
	        \If{$|D| \ge N_0$ and $\varphi_t(w) \ge C_0$}{
	            \tcp{intersect current feasible region with $\{ u : \langle u - w, \nabla \varphi_t(w) \rangle < 0 \}$}
	            \Return{separating hyperplane given by $\nabla \varphi_t(w)$} .\;
	        }
	    }
\end{algorithm2e}

% Algorithm~{\sc SeparationOracle}($w$):
% \begin{enumerate}
%     \item Set $k = 0$ and $D = \{\}$.
%     \item For each new point $x_t$ input by Nature:
%     \begin{enumerate}
%         \item Predict $\hat{y}_t = w \cdot x_t$ and observe $y_t$.
%         \item Add $(x_t,y_t)$ to $D$. Let $v_t$ be the output of the offline SOS regression on $D$. Let $\Sig_t$ be the empirical covariance matrix for $D$. Define $\varphi_t(u) \triangleq \|u - v_t\|_{\Sig_t}^2$ and note that this is a convex function.
%         \item If $|D| \ge N_0$ and $\varphi_t(w) \ge C_0$: return the separating hyperplane $\nabla \varphi_t$, i.e. the feasible region is intersected with
%         \[ \{ u : \langle u - w, \nabla \varphi_t(w) \rangle < 0 \}. \]
%     \end{enumerate}
% \end{enumerate}

\begin{algorithm2e}
\DontPrintSemicolon
\caption{{\sc AMCutter}($r,R,N_0,C_0,T$)}
\label{alg:sosandcut}
	\KwIn{Radius $r$ of target ball around $w^*$, parameter $R$ from Assumption~\ref{assume:scaling}, parameters $N_0,C_0$ to be tuned, number of rounds $T$}
	\KwOut{Sequence of predictions $\wh{y}_1,\ldots,\wh{y}_T$}
	    Let $w$ be the output of running Vaidya's algorithm \cite{vaidya1989new} with {\sc SeparationOracle} defined above and parameters $r,R$, and let $\wh{y}_1,...,\wh{y}_{t_1}$ be the predictions made in the course of running {\sc SeparationOracle}.\;
	    \For{$t_1+1\le t \le T$}{
	        Given new point $x_t$ input by Nature, predict $\wh{y}_t = \iprod{w,x_t}$.
	    }
        \Return{$\wh{y}_1,\ldots,\wh{y}_T$}.\;
\end{algorithm2e}
% \noindent
% Algorithm~{\sc AMCutter}($r,R,N_0,C_0$):
% \begin{enumerate}
%     \item Run Vaidya's algorithm \cite{vaidya1989new} with {\sc SeparationOracle} specified above and parameters $r,R > 0$. Let $w$ be the output of the algorithm.
%     \item For each new point $x_t$ input by Nature: predict $\hat{y}_t = w \cdot x_t$.
% \end{enumerate}

As far as the choice of constants, based on \eqref{eq:ultimate_subgauss_adaptive} and Theorem~\ref{thm:ultimate_regression} we will leave $N_0$ to be optimized later
%\begin{equation}\label{eqn:n0-lb}
%N_0 \ge \td{\Omega}\left(d\log(T/\delta)/\bareta^k\right)
%\end{equation}
%(with the exact choice of $N_0$ still to be optimized later),
%where $\delta$ is the desired probability of failure for the entire algorithm, 
and take
\begin{multline}\label{eqn:c0}
C_0 \triangleq 4Rr +  \max(1,1/\rho^4)\cdot O\left(c_{\delta/T,\eta,N_0}^2 \eta^2\sigma^2 + \rho^4\epsilon^2 + \eta^{1/4}R(\sqrt{c_{\delta/T,\eta,N_0}\eta}\sigma + \epsilon)\sqrt[4]{\frac{\log(T/\delta)}{N_0}} \right. \\
\left. + \eta^{1/2}R^2\sqrt{\frac{\log(T/\delta)}{N_0}} + R\sigma \sqrt{\frac{\log(T/\delta)}{N_0}}\right),
\end{multline}
where $\delta > 0$ is the desired overall probability of success. With this choice of parameters we can guarantee with probability at least $1 - \delta$:
\begin{enumerate}
    \item At every step where $|D| \ge N_0$ in {\sc SeparationOracle}, the guarantee \eqref{eq:ultimate_subgauss_adaptive} is satisfied by the vector $v_t$ output by {\sc SCRAM}, by applying Theorem~\ref{thm:ultimate_regression} and the union bound over all rounds. In particular, by triangle inequality, we have $\norm{w^* - v_t}^2_{\Sig_n} \le C_0 - 4Rr$
        % \norm{w^* - v_t}_{\Sig_n} \lesssim \rho^{-1}\left( c_{\delta,\eta,n}\eta\sigma + \rho\epsilon + \eta^{1/8}R^{1/2}(\sqrt{c_{\delta,\eta,n}\eta}\sigma)^{1/2}\sqrt[8]{\frac{\log(\min(n,d)/\delta)}{n}} \right. \\
        % \left. \eta^{1/4}R\sqrt[4]{\frac{\log(\min(n,d)/\delta)}{n}} + (R\sigma)^{1/2}\sqrt[4]{\frac{\log(1/\delta)}{n}} \right).
    \item 
    %there is a ball of radius $r \triangleq \Theta(\frac{k \sigma^2 \eta^{1 - 2/k}}{R})$ 
    If $w$ lies outside the ball of radius $r$ around $w^*$, the result of {\sc SeparationOracle} is a valid separating hyperplane between $w$ and the ball.
    By convexity of $\varphi$, to see that the ball of radius $r$ around $w^*$ is never cut, we just need to show that all $w'$ with $\|w' - w^*\| \le r$ satisfy $\varphi_t(w') \le C_0$. For $w^*$ we have the stronger guarantee $\varphi_t(w^*) \lesssim C_0 - 4Rr$, just from the guarantee of step 1. For other $w'$ in the ball of radius $r$, we deduce the claim by triangle inequality from the guarantee for $w^*$, using that
    \[ \varphi_t(w') - \varphi_t(w^*) \le \langle \nabla \varphi_t(w'), w' - w^* \rangle = 2 \langle \Sig_t (w' - v_t), w' - w^* \rangle \le 4R\|w' - w^*\| \le 4Rr \]
    where the first inequality is by convexity, 
    and the second inequality uses that $\|\hat{\Sigma}_t\| \le 1$ and that the diameter of $\mathcal{W}$ is at most $2R$.
\end{enumerate}
Recall that the separation oracle can only be called $I = O(d \log(R/r))$ many times, since this is the oracle complexity guarantee from Theorem~\ref{thm:vaidya}: after this many rounds the algorithm is guaranteed to return or query a point in the ball of radius $r$ around $w^*$. Let $D_i$ be the collected dataset $D$ built during the $i$-th invocation of the oracle. Since we know by the triangle inequality and AM-GM that
\[ \|w - w^*\|_{\Sig_t}^2 \le 2\|w - v_t\|_{\Sig_t}^2 + 2\|v_t - w^*\|_{\Sig_t}^2 \]
it follows that after $|D_i|$ gets to size $N_0$ and up to the step before returning a hyperplane, we are guaranteed that $\|w - w^*\|_{\Sig_t}^2 \le 4C_0$. For all of the steps before $|D_i|$ gets to size $N_0$, the error incurred per step is trivially upper bounded by $4R^2$. It follows that the regret incurred per call of the separation is upper bounded by $\max \{ 4N_0 R^2, 4|D_i|C_0 + 4R^2 \}$. Hence, the total regret incurred in step 1 of {\sc AMCutter} is upper bounded by
\begin{equation}\label{eqn:sosandcut-reg1}
\sum_{i = 1}^I (4N_0 R^2 + 4|D_i|C_0) \le 4 N_0 I R^2 + 4 C_0 T = O\left(N_0 dR^2 \log(R/r) + C_0 T\right) 
\end{equation}
using that the total number of oracle calls is $I = O(d\log(R/r))$, and $\sum_i |D_i| \le T$. If $t_1$ is the time step at which the algorithm enters step 2, then the total regret in 
step 2 of {\sc AMCutter} is upper bounded by
\begin{equation}\label{eqn:sosandcut-reg2}
\sum_{t = t_1}^T (\langle w^*, x_t \rangle + \epsilon_t - \langle w, x_t \rangle)^2 \le \sum_{t = t_1}^T (r + |\epsilon_t|)^2 \le 2T(r^2 + \epsilon^2) 
\end{equation}
where in the last step we used the basic inequality $(a + b)^2 \le 2a^2 + 2b^2$.
In particular, the leading term in the regret is $O(k\sigma^2 \eta^{2 - 2/k}T)$ as expected. We formalize this in the following Theorem.

%MINOR TODO: if $\sigma = 0$ or $\eta = 0$ should redefine $r$ to be positive obviously.
\begin{theorem}\label{thm:sosandcut}
For the Huber-Contaminated Online Regression problem with $\eta \le \overline{\eta} < 1/2$ and $\overline{\eta} = \frac{1}{2 + 2\rho^2}$,
Algorithm~{\sc AMCutter} with parameters $R$ and $r \triangleq 1/T$ satisfies the following regret guarantee:
% \begin{multline}
% \sum_{t = 1}^T (y^*_t - \hat{y}_t)^2 \lesssim \frac{1}{\rho^2} \cdot \Big[\left(k\sigma^2 \eta^{2 - 2/k} + \eta\epsilon^2\right) T + (k\sigma^2 + \epsilon^2)\left(d\log(2R/\delta)T^{1/(k+1)} + \log(2/\delta)^{1/k}\cdot T^{k/(k+1)}\right) \\
% + R^2\sqrt{\rho k}\log(d/\delta)^{1/4}\cdot T^{(3k+4)/(4k+4)}\Big].\label{eq:sosandcut}
% \end{multline}
\begin{multline}
    \sum_{t = 1}^T (y^*_t - \hat{y}_t)^2 \lesssim \left(\eta^2\log(1/\eta)\sigma^2\rho^{-4} + \epsilon^2\right) T + \eta^{1/4} R\rho^{-4}\left(\eta^{1/2}\sqrt[4]{\log(1/\eta)}\cdot \sigma + \epsilon\right)\sqrt[4]{\log T}\cdot d^{1/6}T^{5/6} \\
    + \left(\eta^{1/2}R^2 + R\sigma \right)\cdot \rho^{-4} d^{1/3}T^{2/3} \sqrt{\log T} + d^{1/3}R^2\log(RT)T^{2/3}
    \label{eq:sosandcut}
\end{multline}
with probability $1 - 1/\poly(T)$ over the randomness of the coin flips. In particular, for sufficiently large $T$, this quantity is dominated by $(\eta^2\log(1/\eta)\sigma^2\rho^{-4} + \epsilon^2)T$.%, and on the other hand, when $\eta = \epsilon = 0$, this quantity is $\left(R^2\log(RT) + R\rho^{-2}\sigma \sqrt{\log(1/\delta)}\right)\cdot d^{1/3}T^{2/3}$.
\end{theorem}
\begin{proof}
From the above \eqref{eqn:sosandcut-reg1} and \eqref{eqn:sosandcut-reg2}, we see that the total regret is upper bounded by
\begin{equation}
    O\left(N_0 d R^2 \log(R/r) + C_0 T\right) + 2T(r^2 + \epsilon^2). \label{eq:regret_total}
\end{equation}
so by taking $N_0 = d^{-2/3} T^{2/3}$ and $r = 1/T$, we get the claimed regret bound upon noting that $c_{1/10T,\eta,d^{-2/3}T^{2/3}} = O(\sqrt{\log(1/\eta)})$.
% and recalling how $N_0$ appears in $C_0$ in the terms 
% %$O(\sigma(\sigma+R)\sqrt{(d/N_0)\cdot \log(kN_0/d\delta)} + R^2 \sqrt{\log(d/\delta)/N_0})$ 
% \begin{equation}O\left((k\sigma^2 + \epsilon^2)\cdot \left(d\log(2R/\delta)/N_0 + [\log(2/\delta)/N_0]^{1/k}\right) + R^2(k^2\log(d/\delta)\rho^2/N_0)^{1/4}\right)\end{equation}
% from \eqref{eqn:c0},
% % d N_0 = (d/N_0)^{1/k - 2/k^2} T => N_0^{1 + 1/k - 2/k^2} = d^{1/k - 2/k^2 - 1} T
% % => N_0 = d^{(1/k - 2/k^2 - 1)/(1 + 1/k - 2/k^2)} T^{1/(1 + 1/k - 2/k^2)}
% we see that by taking
% $N_0 = T^{k/(k+1)}$ and $r \triangleq 1/T$ in \eqref{eq:regret_total}, we get the claimed regret bound.
\end{proof}
%\TODO{be a little more careful in the last step, and check the bounds used before.}
\subsection{Gradient Descent Algorithm}
For the high-dimensional setting, cutting planes don't work because their guarantees are dimension-dependent. Fortunately, we can fix this by using gradient descent instead. 
\begin{algorithm2e}
\DontPrintSemicolon
\caption{{\sc AM-GD}($R,N_0,C_1,\gamma,T$)}
\label{alg:sosgd}
	\KwIn{Parameter $R$ from Assumption~\ref{assume:scaling}, number of rounds $T$, parameters $r,N_0,C_1,\gamma$ to be tuned}
	\KwOut{Sequence of predictions $\wh{y}_1,\ldots,\wh{y}_T$ (via interaction with Nature)}
	    Let $w_1 = 0$.\;
	    \While{there are more inputs}{
	    Let $g_s$ be the output of {\sc SeparationOracle} run with parameters $r\triangleq0,R,C_1$ and input $w_{s}$\;
	    Let $w_{s + 1} = w_{s} - \frac{\gamma}{\sqrt{T}} g_s$.\;
	    Set $s \gets s + 1$.
	    }
\end{algorithm2e}
We recall the following guarantee for online gradient descent from \cite{zinkevich2003online}.
\begin{theorem}[\cite{zinkevich2003online,hazan2019introduction}]\label{thm:gd}
Suppose that $f_1,\ldots,f_T$ is a sequence of convex functions such that $\|\nabla f_t(w)\| \le G$ for any $w$ with $\|w\| \le R$. Let $w_1 = 0$ and suppose that 
\[ w_{t + 1} \triangleq \Pi_R\left(w_{t} - \frac{2R}{G\sqrt{T}} \nabla f_t(w_t)\right) \]
where $\Pi_R(x) \triangleq \frac{x}{\max(R,\|x\|)}$ is the projection onto the Euclidean ball of norm $R$. Then for any $w^*$ with $\|w^*\| \le R$,
\[ \sum_{t = 1}^T f_t(w_t) - \sum_{t = 1}^T f_t(w^*) \le \sum_{t = 1}^T \langle \nabla f_t(w_t), w_t - w^* \rangle \le 3 RG \sqrt{T}. \]
\end{theorem}
We now discuss parameter selection: we define
\begin{multline}
    C_0 \triangleq
        \max(1,1/\rho^4) \cdot O\left(c^2_{\delta/T,\eta,N_0}\eta^2 \sigma^2 + \eta\rho^4\epsilon^2 + \eta^{1/4} R (\sqrt{c_{\delta/T,\eta,N_0} \eta}\sigma  + \epsilon)\sqrt[4]{\frac{\log(N_0 T/\delta)}{N_0}} \right. \\
        \left. +  \eta^{1/2} R^2\sqrt{\frac{\log(N_0 T/\delta)}{N_0}} 
         +  R\sigma\sqrt{\frac{\log(2T/\delta)}{N_0}}\right)
\end{multline}
where $\delta > 0$ is the overall acceptable probability of failure, 
based upon the right-hand side of \eqref{eq:ultimate_subgauss_adaptive} and take $C_1 \triangleq 2C_0$.
\begin{theorem}\label{thm:sos-gd}
For the Huber-Contaminated Online Regression problem with $\eta \le \overline{\eta} < 1/2$ and $\overline{\eta} = \frac{1}{2 + 2\rho^2}$,
Algorithm~{\sc AM-GD} with parameters $R$ and $\gamma = \Theta(1)$ satisfies the following regret guarantee:
\begin{multline}
\sum_{t = 1}^T (y^*_t - \hat{y}_t)^2 \lesssim 
\left(\eta^2\log(1/\eta)\sigma^2\rho^{-4} + \epsilon^2\right) T + \eta^{1/4} R\rho^{-4}\left(\eta^{1/2}\sqrt[4]{\log(1/\eta)}\cdot \sigma + \epsilon\right)\sqrt[4]{\log T}\cdot T^{9/10} \\
    + \left(\eta^{1/2}R^2\rho^{-4}\sqrt{\log T} + R\rho^{-4}\sigma \sqrt{\log T} + R^2/\eta\right)\cdot T^{4/5}
\label{eq:sos-gd}
\end{multline}
%\begin{align} 
%\sum_{t = 1}^T (y^*_t - \hat{y}_t)^2 \lesssim (k\sigma^2 \eta^{1 - 2/k} + \epsilon^2) T &+ \left(d^{2/3} (\sigma^2 + R^2) (\log(RT) + \sqrt{\log(kTd/\delta)})\right)T^{2/3} \\ 
%&+ \frac{d^2 R^2 \log(RT)}{\overline{\eta}^k}\log(T/\delta) 
%\end{align}
with probability $1 - 1/\poly(T)$ over the randomness of the coin flips. In particular, for sufficiently large $T$, this quantity is dominated by $(\eta\epsilon^2 + k\sigma^2 \eta^{2 - 2/k})T$.
%$\alpha'_k \triangleq \frac{3k - 6}{k^2 + 3k - 6}$ and $\beta'_k\triangleq \frac{k^2}{k^2 + 3k - 6}$.
\end{theorem}

Note that in \eqref{eq:sos-gd} there is a term $R^2 T^{4/5}/\eta$ which increases as $\eta\to 0$. As discussed previously, for very small contamination rate $\eta$ one can simply apply the above Theorem with slightly larger $\eta$ to get meaningful bounds.

\begin{proof}
As in the proof of Theorem~\ref{thm:sosandcut}, we first bound the regret incurred in a single call of \textsc{SeparationOracle} by $4N_0R^2 + 8|D_i|C_0$ where $D_i$ is the dataset $D$ collected in call $i$. It follows then that if $V$ is the total number of calls made to \textsc{SeparationOracle} then the total clean regret is upper bounded by $O(N_0 R^2 V + T C_0)$ where we used that $\sum_i |D_i| \le T$. On the other hand, we know from Theorem~\ref{thm:gd} that if we define $\varphi_i$ to be the function whose gradient is returned at the end of Algorithm~\textsc{SeparationOracle}, then
\[ C_0 V = (C_1 - C_0)V \le \sum_{s = 1}^{V} (\varphi_i(w_s) - \varphi_i(w^*)) \le 6R^2 \sqrt{V} \]
since $\|\nabla \varphi_i(w')\| \le \|\Sig_t(w' - v_t)\| \le 2R$ and using the corresponding choice of $\gamma$. Therefore $V = O(R^4/C_0^2)$. Hence the clean regret is upper bounded by $O(N_0 R^6/C_0^2 + T C_0)$. 

Finally, it remains to choose $N_0$. At this point the optimal choice for $N_0$ is given by equalizing $N_0$ and the terms involving $N_0$ but not $\eta$ in $C_0^3 T/R^6$. Since the leading order term in $C_0$ of this kind is of order $N_0^{-1/2}$ we can roughly minimize by taking
% N_0 = T/N_0^{\gamma_k} => N_0^{1 + \gamma_k} = T
$N_0 = T^{2/5}$. In this case, \begin{equation}
    \frac{N_0 R^6}{C_0^2} \lesssim \frac{N_0 R^6}{\max(1,1/\rho^4)\cdot \eta R^4 \log(T) /N_0} \le \max(1,1/\rho^4)\cdot(R^2/\eta)\cdot T^{4/5},
\end{equation} so the claimed bound follows.
\end{proof}

%% file: apply.tex
%!TEX root = ./main.tex

\section{Putting Everything Together}
\label{sec:apply}

In this section we record consequences of applying our results on Huber-contaminated online regression to the reduction of \cite{foster2020beyond} (see Appendix~\ref{app:oracle}).

The first consequence is the following pseudo-regret/regret bound for Huber-contaminated contextual bandits in the finite-dimensional case. 

\begin{theorem}[Main, formal version of Theorem~\ref{thm:main_bandits_informal}]\label{thm:main_cb_formal}
	For the Huber-Contaminated Contextual Bandits problem with contamination rate $0 \le \eta < 1/2$ and corresponding parameter $\rho$ given by $\eta = \frac{1}{2 + 2\rho^2}$, $\sigma^2$-subgaussian noise $\brc{\xi_t}$, misspecification rate $\epsilon$, range parameter $R$, noise parameter $\sigma$, action space of size $K$, and $d$-dimensional contexts, then there is a $\poly(n,d)$-time algorithm which achieves clean pseudo-regret $\psRegHCB(T)$ at most \begin{multline}
% 		O(C_{\eta}\sqrt{K})\left(
% 			(k^{1/2}\sigma\eta^{1 - 1/k} + \epsilon)T \right.\\
% 			\left. + (\sigma \sqrt{k} + \epsilon)\left(\sqrt{d\log(2R/\delta)}\cdot T^{(k+2)/(2k+2)} + \log(2/\delta)^{1/2k}\cdot T^{(2k+1)/(2k+2)}\right)\right.\\
% 			\left. + R\cdot \rho^{1/4}k^{1/4}\log(d/\delta)^{1/8}\cdot T^{(7k+8)/(8k+8)}
% 		\right),
        O(\sqrt{K}) \left((\eta\sqrt{\log(1/\eta)}\sigma\rho^{-2} + \epsilon)T + \eta^{1/8}R^{1/2}\rho^{-2}\left(\eta^{1/4}\sqrt[8]{\log(1/\eta)}\cdot \sigma^{1/2} + \epsilon^{1/2}\right)\sqrt[8]{\log T}\cdot d^{1/12}T^{11/12} \right. \\
        \left. + \left(\eta^{1/4}R + R^{1/2}\sigma^{1/2}\right)\cdot \rho^{-2}d^{1/6}T^{5/6}\sqrt{\log T} + d^{1/6}R\sqrt{\log(RT)}T^{5/6} \right).
	\end{multline} In particular, for sufficiently large $T$, this quantity is dominated by $\left(\eta\sqrt{\log(1/\eta)}\sigma\rho^{-2} + \epsilon\right)\sqrt{K}T$.

	In the special case where $\epsilon = 0$, there is a $\poly(n,d)$-time algorithm which achieves clean regret $\RegHCB(T)$ at most \begin{multline}
        O(\sqrt{K}) \left(\eta\sqrt{\log(1/\eta)}\sigma\rho^{-2}T + \eta^{1/8}R^{1/2}\rho^{-2}\left(\eta^{1/4}\sqrt[8]{\log(1/\eta)}\cdot \sigma^{1/2} \right)\sqrt[8]{\log T}\cdot d^{1/12}T^{11/12} \right. \\
        \left. + \left(\eta^{1/4}R + R^{1/2}\sigma^{1/2}\right)\cdot \rho^{-2}d^{1/6}T^{5/6}\sqrt{\log T} + d^{1/6}R\sqrt{\log(RT)}T^{5/6} \right).
	\end{multline} with probability $1 - 1/\poly(T)$. For sufficiently large $T$, this is dominated by $\eta\sqrt{\log(1/\eta)}\sigma\rho^{-2}\sqrt{K}T$.
\end{theorem}

\begin{proof}
	For the first part of the theorem, we can apply Theorem~\ref{thm:sosandcut} with failure probability $T^{-1/3}$ to get that the clean square loss regret incurred by {\sc AMCutter} is given by \eqref{eq:sosandcut} with probability at least $1 - T$ and is otherwise upper bounded by $R^2T$. So the expectation of this quantity is at most the quantity in \eqref{eq:sosandcut} plus $R^2 T^{1/3}$, which is dominated by the $d^{1/3}R^2\log(RT)T^{2/3}$ term in \eqref{eq:sosandcut}. The result then follows from applying the clean pseudo-regret bound of Theorem~\ref{thm:fosterrakhlin} and using the elementary fact that for positive numbers $\brc{a_i}_{i\in[s]}$, $\left(\sum^s_{i=1}a_i\right)^{1/2} \le \sum^s_{i=1}\sqrt{a_i}$.

	For the second part of the theorem, we can directly apply the high-probability guarantee Theorem~\ref{thm:sosandcut} together with the high-probability guarantee of Theorem~\ref{thm:fosterrakhlin2} and a union bound.
\end{proof}

\begin{theorem}[High-dimensional variant of Theorem~\ref{thm:main_cb_formal}]\label{thm:main_cb_formal-hd}
	Let $\eta,\rho,\epsilon,R,\sigma,K$ be the same as in Theorem~\ref{thm:main_cb_formal}, but now we make no assumptions on the dimension of the context space $\calX$. There exists an algorithm which runs in polynomial time and achieves clean pseudo-regret $\psRegHCB(T)$ at most \begin{multline}
    % 	O(C_{\eta}\sqrt{K})\cdot\rho^{-1}\left(\left(\sigma\sqrt{k} \eta^{1 - 1/k} + \epsilon\right)T \right. \\
    % 	+ (\sigma\sqrt{k} + \epsilon)\left([\log(2/\delta)]^{1/2k}\cdot T^{(2k+5)/(2k+6)} + \sqrt{\log(2R/\delta)}\cdot T^{(3k+12)/(4k+12)}\right) \\
    %      \left.  + R\cdot \rho^{1/4}k^{1/4}\log(T/\delta)^{1/8}\cdot T^{(7k+24)/(8k+24)} + R\sqrt{\log(T/\delta)}\cdot T^{(3k+12)/(4k+12)}\right)
        O(\sqrt{K})\cdot \left( 
            \left(\eta\sqrt{\log(1/\eta)}\sigma\rho^{-2} + \epsilon\right) T + \eta^{1/8} R^{1/2}\rho^{-2}\left(\eta^{1/4}\sqrt[8]{\log(1/\eta)}\cdot \sigma^{1/2} + \epsilon^{1/2}\right)\sqrt[8]{\log T}\cdot T^{19/20} \right.\\
            \left. + \left(\eta^{1/4}R\rho^{-2}\sqrt[4]{\log T} + R^{1/2}\rho^{-2}\sigma^{1/2} \sqrt[4]{\log T} + R/\sqrt{\eta}\right)\cdot T^{9/10}
        \right).
        \label{eq:hdformal_bound}
	\end{multline}
	In particular, for sufficiently large $T$, this quantity is dominated by $\left(\eta\sqrt{\log(1/\eta)}\sigma\rho^{-2} + \epsilon\right)\sqrt{K}T$. When $\epsilon = 0$, we can similarly achieve a bound on the clean regret $\RegHCB(T)$ with high probability.
\end{theorem}
\begin{proof}
    The proof is identical to Theorem~\ref{thm:main_cb_formal}, except that we replaced the use of Theorem~\ref{thm:sosandcut} by Theorem~\ref{thm:sos-gd} and {\sc AMCutter} by {\sc AM-GD}.
\end{proof}
\begin{comment}
\begin{proof}
	Let $P:\calX\to\R^m$ be the JL projection, and let $\brc{w^*_a}$ be the vectors for which $f(z,a)\triangleq \iprod{w^*_a, z}$ satisfies \eqref{eq:realizable} for quantities $\brc{\epsilon_t(z,a)}_{z,t,a}$. By \eqref{eqn:jl-misspec}, we can equivalently view the given instance of Huber-contaminated contextual bandits with contexts $\brc{z_t}$ of unbounded dimension as an instance of Huber-contaminated contextual bandits with contexts $\brc{Pz_t}$ of dimension $m$, where the loss functions are realized by $Pw^*$ up to some additional misspecification coming from the loss incurred by dimensionality reduction. In this case, $\E[\ell^*_t]{\ell^*_t(a)|z_t = z} = \iprod{Pw^*_a,Pz_t} + \epsilon'_t(z,a)$ for quantities $\epsilon'_t(z,a) \triangleq \epsilon_t(z,a) - \iprod{w^*_a,z} + \iprod{Pw^*_a,Pz_t}$, and by \eqref{eqn:jl-misspec}, $|\epsilon'_t(z,a)| \le \epsilon'\triangleq \epsilon + O(R\sqrt{\log(T/\delta)/m})$.

	The result then follows from Theorem~\ref{thm:fosterrakhlin} and the fact that the extra $O(R\sqrt{\log(T/\delta)K/m}T)$ term in the pseudo-regret bound in Theorem~\ref{thm:fosterrakhlin} coming from the misspecification incurred by dimensionality reduction can be absorbed into the sublinear terms of \eqref{eq:hdformal_bound}.
\end{proof}
\end{comment}

%% file: lowerbound.tex
\section{Lower Bound Against Convex Surrogates}
\label{sec:lowerbound}

We exhibit an $\Omega(\eta^3\sigma R)$ lower bound against regression using convex losses.  This lower bound captures natural approaches like Huber regression, $L_1$/LAD regression, and OLS. By rescaling, we can assume $\sigma = 1$ without loss of generality, which we do in the statement of the result below; also, just for this example we scale (without loss of generality) so that $\|w^*\| \le 1$ and $\|x_t\| \le R$, because this makes the equations slightly cleaner.

\begin{theorem}\label{thm:lowerbound}
For any convex loss $h(\cdot)$, there exists a distribution over covariates $x \sim \calD_x$ with support in $[-R,R]$ and true regressor $\ell \in[-1,1]$ such that the following is true.  Let $y \sim \ell \cdot x + \zeta$ with noise $\zeta \sim \N(0,1)$, and let $\mathcal{C}$ denote the joint distribution over $(x,y)$.  Furthermore, let $\widehat y$ denote the Huber contaminated labels drawn $y \sim (1-\eta)(\ell \cdot x + \zeta) + \eta \mathcal{Q}$ where $\mathcal{Q}$ is an arbitrary distribution with support in $[-R,R]$ for $R \geq \frac{1}{\eta}$ and $\eta \in [0,\frac{1}{2})$.  Let $\mathcal H$ be the joint distribution of the contaminated data $(x,\widehat y)$. For any $b \in [0,1]$, let $w := \argmin_{\ell \in [-b,b]} \mathbb{E}_{(x,\widehat{y}) \sim \mathcal{H}}[h(y - \ell \cdot x)]$ be the minimizer of the loss on contaminated data. 
Then the clean square loss of $w$ is lower bounded as $\mathbb{E}_{(x,y) \sim \mathcal{C}}[(y - w \cdot x)^2] \geq \min\left(\frac{\eta^3R}{40}, \frac{(1 - b)^2R^2}{2}\right)$.
\end{theorem}

\begin{proof}
First, we consider the case where the constraint parameter $b$ is less than $1$. In this case, we can just consider a simple clean example, e.g. the covariate distribution $x = 0$ with probability $1/2$ and $x = R$ with probability $1/2$, and take $\ell = 1$. If $b < 1$ then the best predictor in $[-b,b]$ makes squared loss at least $(1 - b)^2R^2/2$, which proves the second lower bound.

We now consider the more interesting case where $b = 1$.
Our hard instance is constructed as follows.  Let $\calD_x \defeq m_1\delta(1) + (1-m_1)\delta(-R)$ where $\delta(\cdot)$ is the dirac delta and $m_1 = 1-\frac{\eta}{10R}$.  %\TODO{fixme: needed to rule out algorithm that always outputs 0}
Let the true regressor $\ell = 0$ so that the uncorrupted $y \sim \N(0,1)$ for all $x \in [-R,R]$.  Let the corrupted labels be $\widehat y$ defined as follows

\[ \widehat{y} = \begin{cases} 
      (1-\eta)\N(0,1) + \eta \delta(R+1) & x = 1 \\
      \N(0,1) & x = -R \\
   \end{cases}
\]
Let $h'(\cdot)$ be the right derivative of $h(\cdot)$, which is well defined because every convex function on an open convex domain is semi-differentiable.  Let $g(v) \defeq -\mathbb{E}_{y \sim \N(0,1)}[h'(y - v)]$.  By convexity of $h(\cdot)$ we have the right derivative evaluated at $w$ is greater than or equal to zero.  
\begin{multline*}
\lim_{\epsilon \rightarrow 0}\frac{ \mathbb{E}_{(x,y) \sim \mathcal{H}} [h(y-(v + \epsilon) \cdot x)] - \mathbb{E}_{(x,y) \sim \mathcal{H}} [h(y-v \cdot x)] }{\epsilon}\Big|_{v = w}\\ = (1-\eta)m_1 \cdot g(w) - h'(R + 1 - w)\eta \cdot m_1 + (1-m_1)R g(-Rw) \geq 0     
\end{multline*}
Rearranging we obtain 
\begin{equation}
g(w) \geq \frac{h'(R + 1 - w)\eta\cdot m_1 -(1-m_1)R g(-Rw)}{(1-\eta)m_1}  \label{eq:rearrange}
\end{equation}
Let $g^{-1}(\cdot)$ denote the left inverse of $g(\cdot)$.   Note that $h(\cdot)$ is convex implies  $-h'(\cdot)$ is monotonically decreasing implies $g(\cdot)$ is monotonically increasing implies $g^{-1}(\cdot)$ is monotonically increasing. Thus, applying $g^{-1}(\cdot)$ to both sides of \eqref{eq:rearrange} we obtain 
\begin{equation}
w \geq g^{-1}\big(\frac{h'(R + 1 - w)\eta\cdot m_1 -(1-m_1)R g(-Rw)}{(1-\eta)m_1}\big) \label{eq:ellstar}    
\end{equation}
To lower bound $w$ it suffices to lower bound the argument of $g^{-1}(\cdot)$. We obtain, 

$$\frac{h'(R + 1 - w)\eta\cdot m_1 -(1-m_1)R\cdot g(-Rw)}{(1-\eta)m_1} \geq \frac{h'(R)\eta\cdot m_1 + h'(R) R(1-m_1)}{(1-\eta)m_1}$$ 
Where we lower bounded the first term in the numerator using the fact that $h'(\cdot)$ is monotonically increasing and $w \in [-1,1]$ to conclude $h'(R + 1-w) \geq h'(R)$.  We lower bounded the second term in the numerator using the fact that $g(\cdot)$ is monotonically increasing and that $h'(R) \geq \max_{[-R,R]} |h'(x)|$ (monotonicity of $h'(\cdot)$) to conclude $g(-Rw)\geq g(-R) \geq -h'(R)$. Further lower bounding, we obtain

$$= \frac{h'(R)(\eta m_1 - (1-m_1)R)}{(1-\eta)m_1} = \frac{h'(R)(\eta (1-\frac{\eta}{10R}) - \frac{\eta}{10})}{(1-\eta)m_1} \geq \frac{h'(R)\eta}{2(1-\eta)m_1}\geq \frac{h'(R)\eta}{2}$$
Where in the first inequality we use that $R \geq \frac{1}{\eta}$. Substituting this lower bound into \eqref{eq:ellstar} we obtain $w \geq g^{-1}\big( \frac{h'(R)\eta}{2}\big)$. Once again using the fact that $h'(R) \geq \max_{[-R,R]} |h'(x)|$ we observe that 
$$g(\rho) - g(g^{-1}(0)) \leq \frac{(\rho - g^{-1}(0)) h'(R)}{\sqrt{2\pi}}$$
for any $\rho \geq g^{-1}(0)$.  This follows by the definition of $g(\cdot)$ and the fact that the mode of the standard gaussian is $\frac{1}{\sqrt{2\pi}}$.  Setting $\rho = g^{-1}(\frac{h'(R)\eta}{2})$ we obtain
$$\frac{h'(R)\eta}{2} = g(g^{-1}(\frac{h'(R)\eta}{2})) - g(g^{-1}(0)) \leq \frac{(g^{-1}(\frac{h'(R)\eta}{2}) - g^{-1}(0)) h'(R)}{\sqrt{2\pi}}$$ 
which implies 
\begin{equation}
    w \geq g^{-1}(\frac{h'(R)\eta}{2}) \geq \eta + g^{-1}(0)
    \label{eq:casebranch}
\end{equation}
We then have two possibilities.  \\
\textbf{Case 1:} Either $g^{-1}(0) \geq \frac{-\eta}{2}$ in which case the loss is lower bounded by
\begin{multline}
 \mathbb{E}_{(x,y) \sim \mathcal{C}}[(y - w \cdot x)^2] \geq \mathbb{E}_{(x,y) \sim \mathcal{C}}[(y - w \cdot x)^2| x = -R]\mathbb{P}_{\calD_x}(x = -R) = (1-m_1)R^2(w)^2 \\
 \geq (1-m_1)R^2 (\eta + g^{-1}(0))^2  \geq \frac{\eta^3R}{40}     
\end{multline}
Where in the first inequality we use the law of total expectation, and in the second inequality we used \eqref{eq:casebranch} and $g^{-1}(0) \geq \frac{-\eta}{2}$.  This is the desired lower bound. \\
\textbf{Case 2:} In the other case we have $g^{-1}(0) \leq \frac{-\eta}{2}$. Then we flip the sign of the corruptions placed by the adversary.  Let the corrupted distribution be
\[ \widehat{y} = \begin{cases} 
      (1-\eta)\N(0,1) + \eta \delta(-R-1) & x = 1 \\
      \N(0,1) & x = -R \\
   \end{cases}
\]
Then working through the same calculations flipping signs at the right places we obtain  \\$w \leq g^{-1}\big(-\frac{h'(R)\eta}{2}\big)$.
Once again, using that 
$$g(\rho) - g(g^{-1}(0)) \geq \frac{(\rho - g^{-1}(0)) h'(R)}{\sqrt{2\pi}}$$ 
for any $\rho \leq g^{-1}(0)$, and setting $\rho = g^{-1}\big(-\frac{h'(R)\eta}{2}\big)$ we obtain 
$$-\frac{h'(R)\eta}{2} = g(g^{-1}(-\frac{h'(R)\eta}{2})) - g(g^{-1}(0)) \geq \frac{(g^{-1}\big(-\frac{h'(R)\eta}{2}\big) - g^{-1}(0)) h'(R)}{\sqrt{2\pi}}$$
Rearranging we obtain 
$$w \leq g^{-1}\big(-\frac{h'(R)\eta}{2}\big) \leq g^{-1}(0) - \eta \leq \frac{-3\eta}{2}$$
Where the last inequality follows by $g^{-1}(0) \leq \frac{-\eta}{2}$.   
The loss is then lower bounded by
$$ \mathbb{E}_{(x,y) \sim \mathcal{C}}[(y - w \cdot x)^2] \geq \mathbb{E}_{(x,y) \sim \mathcal{C}}[(y - w \cdot x)^2| x = -R]\mathbb{P}_{\calD_x}(x = -R) \geq (1-m_1)R^2(w)^2 \geq \frac{9\eta^3R}{40} $$ 
where in the last inequality we use $w \leq \frac{-3\eta}{2}$ .  This is our desired lower bound.

\end{proof}

%% file: reduction.tex
%!TEX root = ./main.tex
\section{Reduction from Contextual Bandits to Online Regression}
\label{app:oracle}

In this section we verify that the reduction given in \cite{foster2020beyond}, specifically the proof of Theorem 5 in their paper, also applies to our Huber-contaminated setting as well. Formally, we show the following:

\begin{theorem}[Bandits to Regression Reduction]\label{thm:fosterrakhlin}
 	Given any oracle $\calO$ for Huber-contaminated online regression achieving clean square loss regret $\RegHSQ(T)$ in the sense of Definition~\ref{defn:huberreg}, we can produce a learner for Huber-contaminated contextual bandits in the sense of Definition~\ref{defn:huber_bandits} that achieves clean pseudo-regret $O\left(\sqrt{KT\cdot\RegHSQ(T)} + \epsilon\sqrt{K}T\right)$.
\end{theorem} 

We will use the {\sc SquareCB} algorithm from \cite{foster2020beyond}, which draws upon ideas from \cite{Abe1999AssociativeRL}, and which we repeat here for completeness:

\begin{algorithm2e}
\DontPrintSemicolon
\caption{\textsc{SquareCB}($A,\gamma,\mu$)}
\label{alg:squarecb}
	\KwIn{Online regression oracle $\calO$, learning rate $\gamma>0$, exploration parameter $\mu > 0$}
	\KwOut{Sequence of actions, in the setting of Definition~\ref{defn:huber_bandits}}
	\For{$t\in[T]$}{
		Get context $z_t$ from Nature.\;
		For every $a\in\calA$, use regression oracle $\calO$ to compute prediction $\wh{y}_{t,a} \triangleq \wh{y}_t(z_t,a)$.\;\label{step:yhat}
		Define $b_t\triangleq \arg\min_{a\in\calA}\wh{y}_{t,a}$.\;
		For $a\neq b_t$, define $p_{t,a} = \frac{1}{\mu + \gamma(\wh{y}_{t,a} - \wh{y}_{t,b_t})}$ and let $p_{t,b_t} = 1 - \sum_{a\neq b_t}p_{t,a}$. The numbers $\brc{p_{t,a}}_a$ define a distribution $p_t$ over actions.\;\label{step:pta}
		Sample $a_t$ from $p_t$ and observe loss $\ell$, and update $\calO$ with example $((x_t,a_t),\ell)$.
	}
\end{algorithm2e}

\begin{proof}[Proof of Theorem~\ref{thm:fosterrakhlin}]
	Fix any policy $\pi: \calX\to\calA$ and consider the learner given by {\sc SquareCB} (Algorithm~\ref{alg:squarecb}) above for a regression oracle $\calO$ achieving square loss $\RegHSQ(T)$, which is some random variable depending on the interactions with Nature. Recall that for this choice of learner, $\RegHCB(T)$ is the supremum of
	\begin{equation}\E*{\sum_{t=1}^T (\ell^*_t(a_t) - \ell^*_t(\pi(z_t)))}\label{eq:argument}\end{equation}
	over all such $\pi$. Define the filtration \begin{equation}
	\frakF_{t-1} \triangleq \sigma((z_1,a_1,\ell^*_1(a_1),\ell_1(a_1),\gamma_1),\ldots,(z_{t-1},a_{t-1},\ell^*_{t-1}(a_{t-1}),\ell_{t-1}(a_{t-1}),\gamma_{t-1}),(z_t,\gamma_t)).
	\end{equation}
	We can write the sum of conditional expectations of immediate regrets incurred by $\pi$ as \begin{align}
		\sumt \condE{(\ell^*_t(a_t) - \ell^*_t(\pi(z_t))) | \frakF_{t-1}} 
		&\le \sumt \condE{(f(z_t,a_t) - f(z_t,\pi(z_t))) | \frakF_{t-1}} + 2\epsilon T\\
		&\le \sumt \condE{(f(z_t,a_t) - f(z_t,\pi_f(z_t))) | \frakF_{t-1}} + 2\epsilon T\\
		&=   \sumt \suma p_{t,a}(f(z_t,a) - f(z_t,\pi_f(z_t))) + 2\epsilon T.\label{eq:prekeyfrlem}
	\end{align}
	where recall from Definition~\ref{defn:huber_bandits} that $\pi_f(z)\triangleq \arg\max_a f(z,a)$, and $p_{t,a}$ is defined in Step~\ref{step:pta} of {\sc SquareCB}

	The following lemma is a key ingredient in the reduction of \cite{foster2020beyond}:

	\begin{lemma}[Lemma 3, \cite{foster2020beyond}] \label{lem:reduction} 
	For any collection of numbers $\brc{\wh{y}_{a}}_{a\in\calA}\in \brk{-R,R}^K$, let $p$ be the corresponding probability distribution computed in Step~\ref{step:pta}. For any collection of numbers $\brc{f_a}_{a\in\calA}\in\brc{-R,R}^K$, if we define $a^* \triangleq \arg\max_a f_a$, we have that \begin{equation}
		\suma p_a \brk*{(f_a - f_{a^*}) - \frac{\gamma}{4} (\wh{y}_a - f_a)^2} \leq \frac{2K}{\gamma}
	\end{equation}
	\end{lemma}

	Applying Lemma~\ref{lem:reduction}, we can upper bound \eqref{eq:prekeyfrlem} by \begin{equation}
		\frac{\gamma}{4}\sumt \condE{(\wh{y}_t(z_t,a_t) - f(z_t,a_t))^2 | \frakF_{t-1}} + \frac{2KT}{\gamma} + 2\epsilon T.
	\end{equation}
	By this and law of total expectation, the pseudo-regret incurred by policy $\pi$ can be upper bounded by \begin{equation}
		\frac{\gamma}{4}\E*{(\wh{y}_t(z_t,a_t) - f(z_t,a_t))^2} + \frac{2KT}{\gamma} + 2\epsilon T. \label{eq:prederror}
	\end{equation}
	To bound the prediction error in \eqref{eq:prederror}, using the identity $b^2 \le (a+b)^2 - 2ab$, we can upper bound $(\wh{y}_t(z_t,a_t) - f(z_t,a_t))^2$ by \begin{equation}
		(\wh{y}_t(z_t,a_t) - \ell^*_t(a_t))^2 - 2(f(z_t,a_t) - \ell^*_t(a_t))(\wh{y}_t(z_t,a_t) - f(z_t,a_t)).\label{eq:triv}
	\end{equation}
	Recall from \eqref{eq:realizable} that the misspecification adversary is oblivious, that is, conditioned on $\frakF_{t-1}$, $f(z_t,a_t) - \ell^*_t(a_t)$ is equal to $-\epsilon_t(z_t,a_t)$. Putting this and \eqref{eq:triv} together and applying law of total expectation, we can bound the expectation of the prediction error in \eqref{eq:prederror} by
	\begin{align}
		\MoveEqLeft \E*{(\wh{y}_t(z_t,a_t) - f(z_t,a_t))^2} \\
		&\le \E{\RegHSQ(T)} + 2\E*{\sumt \condE{\epsilon_t(z_t,a_t)(\wh{y}_t(z_t,a_t) - f(z_t,a_t))|\frakF_{t-1}}} \\
		&\le \E{\RegHSQ(T)} + 2\E*{\sumt \epsilon^2_t(z_t,a_t) + \frac{1}{4}\sumt \condE{(\wh{y}_t(z_t,a_t) - f(z_t,a_t))^2 | \frakF_{t-1}}} \\
		&\le \E{\RegHSQ(T)} + 2\epsilon^2 T + \frac{1}{2}\sumt \E{(\wh{y}_t(z_t,a_t) - f(z_t,a_t))^2},
	\end{align}
	which upon rearranging gives \begin{equation}
		\E*{(\wh{y}_t(z_t,a_t) - f(z_t,a_t))^2} \le 2\E{\RegHSQ(T)} + 4\epsilon^2 T.
	\end{equation}
	Substituting this into \eqref{eq:prederror}, and taking $\gamma = 2\sqrt{KT/(\E{\RegHSQ(T)} + 2\epsilon^2T)}$ and $\mu = K$, we conclude that the pseudo-regret incurred by $\pi$ is upper bounded by \begin{equation}
		\frac{\gamma}{2}(\E{\RegHSQ(T)} + 2\epsilon^2 T) + \frac{2KT}{\gamma} + 2\epsilon T \le 2\sqrt{KT\cdot \E{\RegHSQ(T)}} + 5\epsilon\sqrt{K}T
	\end{equation} as desired.
\end{proof}

In the special case where $\epsilon = 0$, \cite{foster2020beyond} also gives a \emph{high-probability} bound on the \emph{regret} (see their Theorem 1). By adapting their argument, we can show an analogous statement in this setting:

\begin{theorem}[Bandits to Regression Reduction]\label{thm:fosterrakhlin2}
 	Fix any $\delta > 0$. Given any oracle $\calO$ for Huber-contaminated online regression achieving clean square loss regret $\RegHSQ(T)$ in the sense of Definition~\ref{defn:huberreg} with $\epsilon = 0$, we can produce a learner for Huber-contaminated contextual bandits in the sense of Definition~\ref{defn:huber_bandits} that with probability at least $1 - \delta$ achieves achieves clean regret at most $4\sqrt{KT\cdot\RegHSQ(T)} + 8\sqrt{KT\log(2/\delta)}$.
\end{theorem}

%% file: azuma.tex
\section{Proof of Theorem~\ref{thm:azuma-vector}}
In this section we give a self-contained proof of Theorem~\ref{thm:azuma-vector}, largely following the proof of Equation 5.18 in \cite{kallenberg1991some}. 

First, we recall the statement.
Suppose that $X_1,\ldots,X_n$ are random vectors in $\mathbb{R}^d$ with $\|X_t\| \le 1$ for all $t$, and $\xi_1,\ldots,\xi_n$ are random variables such that almost surely, the law of $\xi_t$ conditional on $X_1,\ldots,X_{t},\xi_1,\ldots,\xi_{t - 1}$ is mean-zero and $\sigma^2$-subgaussian. Then
    \begin{equation}
		\Pr*{\left\|\frac{1}{n} \sum^n_{i=1} \xi_i X_i\right\| \ge s} \le 2\exp\left(\frac{-ns^2}{2\pi \sigma^2}\right).
	\end{equation}
	
\begin{proof}[Proof of Theorem~\ref{thm:azuma-vector}]
Without loss of generality, we rescale so that $\sigma = 1$.
The key observation is that for any $a \in \mathbb{R}^d$ and $\lambda \in \mathbb{R}$,
\begin{equation}\label{eqn:azuma-a}
    F_a \triangleq \E{e^{\lambda \sum_i \xi_i \langle X_i, a \rangle - \lambda^2 \sum_i \langle X_i, a \rangle^2/2}} \le 1.
\end{equation}
The proof of \eqref{eqn:azuma-a} follows by an inductive argument. Let $\mathcal{F}_t$ be the filtration generated by $X_1,\ldots,X_t,\xi_1,\ldots,\xi_{t - 1}$. Then the first step of the induction is to observe
\begin{align*} 
\E{e^{\lambda \sum_{i = 1}^n \xi_i \langle X_i, a \rangle - \lambda^2 \sum_{i = 1}^n \langle X_i, a \rangle^2/2} \mid \mathcal{F}_{n}} 
&= e^{\lambda \sum_{i = 1}^{n - 1} \xi_i \langle X_i, a \rangle - \lambda^2 \sum_{i = 1}^{n - 1} \langle X_i, a \rangle^2/2} \E{e^{\lambda \xi_n \langle X_n, a \rangle - \lambda^2 \langle X_n, a \rangle^2/2} \mid \mathcal{F}_n} \\
&\le e^{\lambda \sum_{i = 1}^{n - 1} \xi_i \langle X_i, a \rangle - \lambda^2 \sum_{i = 1}^{n - 1} \langle X_i, a \rangle^2/2} 
\end{align*}
by the conditional subgaussian assumption on $\xi_n$. Iterating this argument shows \eqref{eqn:azuma-a}.

From here the argument follows \cite{kallenberg1991some}. We let $Z \sim N(0,I_{d \times d})$ be a Gaussian vector independent of everything else, and letting $\gamma = \lambda\sqrt{\pi/2}$ we have
\begin{align*} 
\E{e^{\lambda \norm*{\sum_{i = 1}^n \xi_i X_i}}} 
&\le \E{e^{\gamma \E[Z]{\|\langle Z, \sum_{i = 1}^n \xi_i X_i \rangle|} + [\gamma^2/2](n - \E[Z]{\sum_i \langle X_i, Z \rangle^2}))}} \\
&\le e^{n\gamma^2/2} \E{e^{\gamma \|\langle Z, \sum_{i = 1}^n \xi_i X_i \rangle| - \sum_i \langle X_i, Z \rangle^2})} 
\end{align*}
where in the first inequality we used $\E{|\langle Z, u \rangle|} = \sqrt{2/\pi} \|u\|$ and $\E[Z]{\sum_i \langle X_i, Z \rangle^2} = \sum_i \|X_i\|^2 \le n$ almost surely, and the second step is Jensen's inequality. Using the inequality $e^{|x|} \le e^x + e^{-x}$ gives
\[ \E{e^{\gamma \|\langle Z, \sum_{i = 1}^n \xi_i X_i \rangle| - (\gamma^2/2)\sum_i \langle X_i, Z \rangle^2})} \le \E[Z]{F_Z + F_{-Z}} \le 2 \]
by \eqref{eqn:azuma-a}. This shows
$e^{\lambda \|\sum_{i = 1}^n \xi_i X_i\|} \le 2 e^{n\lambda^2\pi/2}$
hence
\[ \Pr{e^{\lambda \|\sum_{i = 1}^n \xi_i X_i\|} \ge e^{\lambda s}} \le 2e^{n\lambda^2\pi/2 - \lambda s}\]
and taking $\lambda = s/n\pi$ makes the rhs $e^{-s^2/2n\pi}$ which is equivalent to the result.
\end{proof}

%% file: main.bbl
\newcommand{\etalchar}[1]{$^{#1}$}
\begin{thebibliography}{DKK{\etalchar{+}}19b}

\bibitem[ABM19]{altschuler2019best}
Jason Altschuler, Victor-Emmanuel Brunel, and Alan Malek.
\newblock Best arm identification for contaminated bandits.
\newblock {\em J. Mach. Learn. Res.}, 20(91):1--39, 2019.

\bibitem[AC16]{auer2016algorithm}
Peter Auer and Chao-Kai Chiang.
\newblock An algorithm with nearly optimal pseudo-regret for both stochastic
  and adversarial bandits.
\newblock In {\em Conference on Learning Theory}, pages 116--120, 2016.

\bibitem[ACBFS02]{auer2002nonstochastic}
Peter Auer, Nicolo Cesa-Bianchi, Yoav Freund, and Robert~E Schapire.
\newblock The nonstochastic multiarmed bandit problem.
\newblock {\em SIAM journal on computing}, 32(1):48--77, 2002.

\bibitem[AGKS21]{awasthi2021online}
Pranjal Awasthi, Sreenivas Gollapudi, Kostas Kollias, and Apaar Sadhwani.
\newblock Online learning under adversarial corruptions, 2021.

\bibitem[AL99]{Abe1999AssociativeRL}
N.~Abe and Philip~M. Long.
\newblock Associative reinforcement learning using linear probabilistic
  concepts.
\newblock In {\em ICML}, 1999.

\bibitem[AW01]{azoury2001relative}
Katy~S Azoury and Manfred~K Warmuth.
\newblock Relative loss bounds for on-line density estimation with the
  exponential family of distributions.
\newblock {\em Machine Learning}, 43(3):211--246, 2001.

\bibitem[BBM{\etalchar{+}}05]{bartlett2005local}
Peter~L Bartlett, Olivier Bousquet, Shahar Mendelson, et~al.
\newblock Local rademacher complexities.
\newblock {\em The Annals of Statistics}, 33(4):1497--1537, 2005.

\bibitem[BCBL13]{bubeck2013bandits}
S{\'e}bastien Bubeck, Nicolo Cesa-Bianchi, and G{\'a}bor Lugosi.
\newblock Bandits with heavy tail.
\newblock {\em IEEE Transactions on Information Theory}, 59(11):7711--7717,
  2013.

\bibitem[Ber06]{bernholt2006robust}
Thorsten Bernholt.
\newblock Robust estimators are hard to compute.
\newblock Technical report, Technical Report, 2006.

\bibitem[BJK78]{bassett1978asymptotic}
Gilbert Bassett~Jr and Roger Koenker.
\newblock Asymptotic theory of least absolute error regression.
\newblock {\em Journal of the American Statistical Association},
  73(363):618--622, 1978.

\bibitem[BJK15]{bhatia2015robust}
Kush Bhatia, Prateek Jain, and Purushottam Kar.
\newblock Robust regression via hard thresholding.
\newblock In {\em Advances in Neural Information Processing Systems}, pages
  721--729, 2015.

\bibitem[BJKK17]{bhatia2017consistent}
Kush Bhatia, Prateek Jain, Parameswaran Kamalaruban, and Purushottam Kar.
\newblock Consistent robust regression.
\newblock In {\em Advances in Neural Information Processing Systems}, pages
  2110--2119, 2017.

\bibitem[BK20]{bakshi2020outlier}
Ainesh Bakshi and Pravesh Kothari.
\newblock Outlier-robust clustering of non-spherical mixtures.
\newblock {\em arXiv preprint arXiv:2005.02970}, 2020.

\bibitem[BKS14]{barak2014rounding}
Boaz Barak, Jonathan~A Kelner, and David Steurer.
\newblock Rounding sum-of-squares relaxations.
\newblock In {\em Proceedings of the forty-sixth annual ACM symposium on Theory
  of computing}, pages 31--40, 2014.

\bibitem[Bos57]{boscovich1757litteraria}
Roger~Joseph Boscovich.
\newblock De litteraria expeditione per pontificiam ditionem, et synopsis
  amplioris operis, ac habentur plura ejus ex exemplaria etiam sensorum
  impessa.
\newblock {\em Bononiensi Scientiarum et Artum Instuto Atque Academia
  Commentarii}, 4:353--396, 1757.

\bibitem[BP20]{bakshi2020robust}
Ainesh Bakshi and Adarsh Prasad.
\newblock Robust linear regression: Optimal rates in polynomial time.
\newblock {\em arXiv preprint arXiv:2007.01394}, 2020.

\bibitem[BR19]{bouneffouf2019survey}
Djallel Bouneffouf and Irina Rish.
\newblock A survey on practical applications of multi-armed and contextual
  bandits.
\newblock {\em arXiv preprint arXiv:1904.10040}, 2019.

\bibitem[BS12]{bubeck2012best}
S{\'e}bastien Bubeck and Aleksandrs Slivkins.
\newblock The best of both worlds: Stochastic and adversarial bandits.
\newblock In {\em Conference on Learning Theory}, pages 42--1, 2012.

\bibitem[Bub14]{bubeck2014convex}
S{\'e}bastien Bubeck.
\newblock Convex optimization: Algorithms and complexity.
\newblock {\em arXiv preprint arXiv:1405.4980}, 2014.

\bibitem[CAT{\etalchar{+}}20]{cherapanamjeri2020optimal}
Yeshwanth Cherapanamjeri, Efe Aras, Nilesh Tripuraneni, Michael~I Jordan,
  Nicolas Flammarion, and Peter~L Bartlett.
\newblock Optimal robust linear regression in nearly linear time.
\newblock {\em arXiv preprint arXiv:2007.08137}, 2020.

\bibitem[CBL06]{cesa2006prediction}
Nicolo Cesa-Bianchi and G{\'a}bor Lugosi.
\newblock {\em Prediction, learning, and games}.
\newblock Cambridge university press, 2006.

\bibitem[Chi20]{chinot2020erm}
Geoffrey Chinot.
\newblock Erm and rerm are optimal estimators for regression problems when
  malicious outliers corrupt the labels, 2020.

\bibitem[CKMY20]{chen2020classification}
Sitan Chen, Frederic Koehler, Ankur Moitra, and Morris Yau.
\newblock Classification under misspecification: Halfspaces, generalized linear
  models, and connections to evolvability.
\newblock {\em arXiv preprint arXiv:2006.04787}, 2020.

\bibitem[CSV17]{charikar2017learning}
Moses Charikar, Jacob Steinhardt, and Gregory Valiant.
\newblock Learning from untrusted data.
\newblock In {\em Proceedings of the 49th Annual ACM SIGACT Symposium on Theory
  of Computing}, pages 47--60. ACM, 2017.

\bibitem[DGT19]{diakonikolas2019distribution}
Ilias Diakonikolas, Themis Gouleakis, and Christos Tzamos.
\newblock Distribution-independent pac learning of halfspaces with massart
  noise.
\newblock In {\em Advances in Neural Information Processing Systems}, pages
  4749--4760, 2019.

\bibitem[DHKK20]{diakonikolas2020robustly}
Ilias Diakonikolas, Samuel~B Hopkins, Daniel Kane, and Sushrut Karmalkar.
\newblock Robustly learning any clusterable mixture of gaussians.
\newblock {\em arXiv preprint arXiv:2005.06417}, 2020.

\bibitem[DK19]{diakonikolas2019recent}
Ilias Diakonikolas and Daniel~M Kane.
\newblock Recent advances in algorithmic high-dimensional robust statistics.
\newblock {\em arXiv preprint arXiv:1911.05911}, 2019.

\bibitem[DKK{\etalchar{+}}17]{diakonikolas2017being}
Ilias Diakonikolas, Gautam Kamath, Daniel~M Kane, Jerry Li, Ankur Moitra, and
  Alistair Stewart.
\newblock Being robust (in high dimensions) can be practical.
\newblock In {\em Proceedings of the 34th International Conference on Machine
  Learning-Volume 70}, pages 999--1008. JMLR. org, 2017.

\bibitem[DKK{\etalchar{+}}18]{diakonikolas2018robustly}
Ilias Diakonikolas, Gautam Kamath, Daniel~M Kane, Jerry Li, Ankur Moitra, and
  Alistair Stewart.
\newblock Robustly learning a gaussian: Getting optimal error, efficiently.
\newblock In {\em Proceedings of the Twenty-Ninth Annual ACM-SIAM Symposium on
  Discrete Algorithms}, pages 2683--2702. SIAM, 2018.

\bibitem[DKK{\etalchar{+}}19a]{diakonikolas2019robust}
Ilias Diakonikolas, Gautam Kamath, Daniel Kane, Jerry Li, Ankur Moitra, and
  Alistair Stewart.
\newblock Robust estimators in high-dimensions without the computational
  intractability.
\newblock {\em SIAM Journal on Computing}, 48(2):742--864, 2019.

\bibitem[DKK{\etalchar{+}}19b]{diakonikolas2019sever}
Ilias Diakonikolas, Gautam Kamath, Daniel Kane, Jerry Li, Jacob Steinhardt, and
  Alistair Stewart.
\newblock Sever: A robust meta-algorithm for stochastic optimization.
\newblock In {\em International Conference on Machine Learning}, pages
  1596--1606, 2019.

\bibitem[DKK{\etalchar{+}}20]{diakonikolas2020polynomial}
Ilias Diakonikolas, Daniel~M. Kane, Vasilis Kontonis, Christos Tzamos, and
  Nikos Zarifis.
\newblock A polynomial time algorithm for learning halfspaces with tsybakov
  noise.
\newblock {\em arXiv preprint arXiv:2010.01705}, 2020.

\bibitem[DKS19]{diakonikolas2019efficient}
Ilias Diakonikolas, Weihao Kong, and Alistair Stewart.
\newblock Efficient algorithms and lower bounds for robust linear regression.
\newblock In {\em Proceedings of the Thirtieth Annual ACM-SIAM Symposium on
  Discrete Algorithms}, pages 2745--2754. SIAM, 2019.

\bibitem[DKTZ20]{diakonikolas2020learning}
Ilias Diakonikolas, Vasilis Kontonis, Christos Tzamos, and Nikos Zarifis.
\newblock Learning halfspaces with massart noise under structured
  distributions.
\newblock {\em arXiv preprint arXiv:2002.05632}, 2020.

\bibitem[dNS20]{dorsi2020regress}
Tommaso d'Orsi, Gleb Novikov, and David Steurer.
\newblock Regress consistently when oblivious outliers overwhelm, 2020.

\bibitem[DT19]{dalalyan2019outlier}
Arnak Dalalyan and Philip Thompson.
\newblock Outlier-robust estimation of a sparse linear model using l1-penalized
  huber's m-estimator.
\newblock In {\em Advances in Neural Information Processing Systems}, pages
  13188--13198, 2019.

\bibitem[Dur19]{durrett2019probability}
Rick Durrett.
\newblock {\em Probability: theory and examples}, volume~49.
\newblock Cambridge university press, 2019.

\bibitem[FN71]{fuk1971probability}
D~Kh Fuk and Sergey~V Nagaev.
\newblock Probability inequalities for sums of independent random variables.
\newblock {\em Theory of Probability \& Its Applications}, 16(4):643--660,
  1971.

\bibitem[FR20]{foster2020beyond}
Dylan~J Foster and Alexander Rakhlin.
\newblock Beyond ucb: Optimal and efficient contextual bandits with regression
  oracles.
\newblock {\em arXiv preprint arXiv:2002.04926}, 2020.

\bibitem[GKT19]{gupta2019better}
Anupam Gupta, Tomer Koren, and Kunal Talwar.
\newblock Better algorithms for stochastic bandits with adversarial
  corruptions.
\newblock In {\em Conference on Learning Theory}, pages 1562--1578, 2019.

\bibitem[GLS81]{GLS1981}
M.~Gr{\"o}tschel, L.~Lov{\'a}sz, and A.~Schrijver.
\newblock The ellipsoid method and its consequences in combinatorial
  optimization.
\newblock {\em Combinatorica}, 1(2):169--197, Jun 1981.

\bibitem[Haz19]{hazan2019introduction}
Elad Hazan.
\newblock Introduction to online convex optimization.
\newblock {\em arXiv preprint arXiv:1909.05207}, 2019.

\bibitem[HKZ{\etalchar{+}}12]{hsu2012tail}
Daniel Hsu, Sham Kakade, Tong Zhang, et~al.
\newblock Tail inequalities for sums of random matrices that depend on the
  intrinsic dimension.
\newblock {\em Electronic Communications in Probability}, 17, 2012.

\bibitem[HL18]{hopkins2018mixture}
Samuel~B Hopkins and Jerry Li.
\newblock Mixture models, robustness, and sum of squares proofs.
\newblock In {\em Proceedings of the 50th Annual ACM SIGACT Symposium on Theory
  of Computing}, pages 1021--1034. ACM, 2018.

\bibitem[HL19]{hopkins2019hard}
Samuel~B Hopkins and Jerry Li.
\newblock How hard is robust mean estimation?
\newblock In {\em Conference on Learning Theory}, pages 1649--1682, 2019.

\bibitem[HM13]{hardt2013algorithms}
Moritz Hardt and Ankur Moitra.
\newblock Algorithms and hardness for robust subspace recovery.
\newblock In {\em Conference on Learning Theory}, pages 354--375, 2013.

\bibitem[HS16]{hsu2016loss}
Daniel Hsu and Sivan Sabato.
\newblock Loss minimization and parameter estimation with heavy tails.
\newblock {\em The Journal of Machine Learning Research}, 17(1):543--582, 2016.

\bibitem[Hub64]{huber1964robust}
Peter~J Huber.
\newblock Robust estimation of a location parameter.
\newblock {\em The Annals of Mathematical Statistics}, pages 73--101, 1964.

\bibitem[Hub73]{huber1973robust}
Peter~J Huber.
\newblock Robust regression: Asymptotics, conjectures and monte carlo.
\newblock {\em The Annals of Statistics}, pages 799--821, 1973.

\bibitem[Kan20]{kane2020robust}
Daniel~M. Kane.
\newblock Robust learning of mixtures of gaussians.
\newblock {\em arXiv preprint arXiv:2007.05912}, 2020.

\bibitem[Kee10]{keener2010theoretical}
Robert~W Keener.
\newblock {\em Theoretical statistics: Topics for a core course}.
\newblock Springer Science \& Business Media, 2010.

\bibitem[KKM18]{klivans2018efficient}
Adam Klivans, Pravesh~K Kothari, and Raghu Meka.
\newblock Efficient algorithms for outlier-robust regression.
\newblock In {\em Conference On Learning Theory}, pages 1420--1430, 2018.

\bibitem[KM15]{koltchinskii2015bounding}
Vladimir Koltchinskii and Shahar Mendelson.
\newblock Bounding the smallest singular value of a random matrix without
  concentration.
\newblock {\em International Mathematics Research Notices},
  2015(23):12991--13008, 2015.

\bibitem[KP18]{karmalkar2018compressed}
Sushrut Karmalkar and Eric Price.
\newblock Compressed sensing with adversarial sparse noise via l1 regression.
\newblock In {\em 2nd Symposium on Simplicity in Algorithms (SOSA 2019)}.
  Schloss Dagstuhl-Leibniz-Zentrum fuer Informatik, 2018.

\bibitem[KPK19]{kapoor2019corruption}
Sayash Kapoor, Kumar~Kshitij Patel, and Purushottam Kar.
\newblock Corruption-tolerant bandit learning.
\newblock {\em Machine Learning}, 108(4):687--715, 2019.

\bibitem[KS91]{kallenberg1991some}
Olav Kallenberg and Rafal Sztencel.
\newblock Some dimension-free features of vector-valued martingales.
\newblock {\em Probability Theory and Related Fields}, 88(2):215--247, 1991.

\bibitem[KSS18]{kothari2018robust}
Pravesh~K Kothari, Jacob Steinhardt, and David Steurer.
\newblock Robust moment estimation and improved clustering via sum of squares.
\newblock In {\em Proceedings of the 50th Annual ACM SIGACT Symposium on Theory
  of Computing}, pages 1035--1046. ACM, 2018.

\bibitem[L{\etalchar{+}}97]{latala1997estimation}
Rafa{\l} Lata{\l}a et~al.
\newblock Estimation of moments of sums of independent real random variables.
\newblock {\em The Annals of Probability}, 25(3):1502--1513, 1997.

\bibitem[L{\etalchar{+}}17]{loh2017statistical}
Po-Ling Loh et~al.
\newblock Statistical consistency and asymptotic normality for high-dimensional
  robust $ m $-estimators.
\newblock {\em The Annals of Statistics}, 45(2):866--896, 2017.

\bibitem[Las01]{Lasserre01}
Jean~B. Lasserre.
\newblock {\em New Positive Semidefinite Relaxations for Nonconvex Quadratic
  Programs}, pages 319--331.
\newblock Springer US, Boston, MA, 2001.

\bibitem[Li18]{li2018principled}
Jerry~Zheng Li.
\newblock {\em Principled approaches to robust machine learning and beyond}.
\newblock PhD thesis, Massachusetts Institute of Technology, 2018.

\bibitem[LMPL18]{lykouris2018stochastic}
Thodoris Lykouris, Vahab Mirrokni, and Renato Paes~Leme.
\newblock Stochastic bandits robust to adversarial corruptions.
\newblock In {\em Proceedings of the 50th Annual ACM SIGACT Symposium on Theory
  of Computing}, pages 114--122, 2018.

\bibitem[LRV16]{lai2016agnostic}
Kevin~A Lai, Anup~B Rao, and Santosh Vempala.
\newblock Agnostic estimation of mean and covariance.
\newblock In {\em 2016 IEEE 57th Annual Symposium on Foundations of Computer
  Science (FOCS)}, pages 665--674. IEEE, 2016.

\bibitem[M{\etalchar{+}}15]{minsker2015geometric}
Stanislav Minsker et~al.
\newblock Geometric median and robust estimation in banach spaces.
\newblock {\em Bernoulli}, 21(4):2308--2335, 2015.

\bibitem[Min17]{minsker2017some}
Stanislav Minsker.
\newblock On some extensions of bernstein’s inequality for self-adjoint
  operators.
\newblock {\em Statistics \& Probability Letters}, 127:111--119, 2017.

\bibitem[Nes00]{Nesterov00}
Yurii Nesterov.
\newblock {\em Squared Functional Systems and Optimization Problems}, pages
  405--440.
\newblock Springer US, Boston, MA, 2000.

\bibitem[NO20]{neu2020efficient}
Gergely Neu and Julia Olkhovskaya.
\newblock Efficient and robust algorithms for adversarial linear contextual
  bandits.
\newblock {\em arXiv preprint arXiv:2002.00287}, 2020.

\bibitem[Par00]{Parrilo00}
Pablo~A. Parrilo.
\newblock Structured semidefinite programs and semialgebraic geometry methods
  in robustness and optimization.
\newblock Technical report, California Institute of Technology, 2000.

\bibitem[PF20]{pesme2020online}
Scott Pesme and Nicolas Flammarion.
\newblock Online robust regression via sgd on the l1 loss.
\newblock {\em arXiv preprint arXiv:2007.00399}, 2020.

\bibitem[Pin94]{pinelis1994optimum}
Iosif Pinelis.
\newblock Optimum bounds for the distributions of martingales in banach spaces.
\newblock {\em The Annals of Probability}, pages 1679--1706, 1994.

\bibitem[PJL20]{pensia2020robust}
Ankit Pensia, Varun Jog, and Po-Ling Loh.
\newblock Robust regression with covariate filtering: Heavy tails and
  adversarial contamination.
\newblock {\em arXiv preprint arXiv:2009.12976}, 2020.

\bibitem[Pol91]{pollard1991asymptotics}
David Pollard.
\newblock Asymptotics for least absolute deviation regression estimators.
\newblock {\em Econometric Theory}, 7(2):186--199, 1991.

\bibitem[PSB{\etalchar{+}}20]{prasad2020robust}
Adarsh Prasad, Arun~Sai Suggala, Sivaraman Balakrishnan, Pradeep Ravikumar,
  et~al.
\newblock Robust estimation via robust gradient estimation.
\newblock {\em Journal of the Royal Statistical Society Series B},
  82(3):601--627, 2020.

\bibitem[RH17]{rigollethigh}
Philippe Rigollet and Jan-Christian H{\"u}tter.
\newblock High dimensional statistics.
\newblock {\em URL http://www-math. mit. edu/\~{} rigollet/PDFs/RigNotes17.
  pdf}, 2017.

\bibitem[SBRJ19]{suggala2019adaptive}
Arun~Sai Suggala, Kush Bhatia, Pradeep Ravikumar, and Prateek Jain.
\newblock Adaptive hard thresholding for near-optimal consistent robust
  regression.
\newblock In {\em Conference on Learning Theory}, pages 2892--2897, 2019.

\bibitem[SF20]{Sasai2020RobustEW}
Takeyuki Sasai and H.~Fujisawa.
\newblock Robust estimation with lasso when outputs are adversarially
  contaminated.
\newblock {\em ArXiv}, abs/2004.05990, 2020.

\bibitem[Sho87]{Shor87}
N.Z. Shor.
\newblock Quadratic optimization problems.
\newblock {\em Soviet Journal of Computer and Systems Sciences}, 25, 11 1987.

\bibitem[SL17]{seldin2017improved}
Yevgeny Seldin and G{\'a}bor Lugosi.
\newblock An improved parametrization and analysis of the exp3++ algorithm for
  stochastic and adversarial bandits.
\newblock In {\em Conference on Learning Theory}, pages 1743--1759, 2017.

\bibitem[SLX20]{simchi2020bypassing}
David Simchi-Levi and Yunzong Xu.
\newblock Bypassing the monster: A faster and simpler optimal algorithm for
  contextual bandits under realizability.
\newblock {\em Available at SSRN}, 2020.

\bibitem[SS{\etalchar{+}}11]{shalev2011online}
Shai Shalev-Shwartz et~al.
\newblock Online learning and online convex optimization.
\newblock {\em Foundations and trends in Machine Learning}, 4(2):107--194,
  2011.

\bibitem[SS14]{seldin2014one}
Yevgeny Seldin and Aleksandrs Slivkins.
\newblock One practical algorithm for both stochastic and adversarial bandits.
\newblock In {\em Proceedings of the 31st International Conference on
  International Conference on Machine Learning-Volume 32}, pages II--1287,
  2014.

\bibitem[SSBD14]{shalev2014understanding}
Shai Shalev-Shwartz and Shai Ben-David.
\newblock {\em Understanding machine learning: From theory to algorithms}.
\newblock Cambridge university press, 2014.

\bibitem[SST10]{srebro2010optimistic}
Nathan Srebro, Karthik Sridharan, and Ambuj Tewari.
\newblock Optimistic rates for learning with a smooth loss.
\newblock {\em arXiv preprint arXiv:1009.3896}, 2010.

\bibitem[Ste18]{steinhardt2018robust}
Jacob Steinhardt.
\newblock {\em Robust Learning: Information Theory and Algorithms}.
\newblock PhD thesis, Stanford University, 2018.

\bibitem[TK08]{tewarinotes}
Ambuj Tewari and Sham Kakade.
\newblock Lectures notes for cmsc 35900: Learning theory, 2008.

\bibitem[Tro11]{tropp2011user}
Joel~A Tropp.
\newblock User-friendly tail bounds for matrix martingales.
\newblock Technical report, CALIFORNIA INST OF TECH PASADENA, 2011.

\bibitem[Tro12]{tropp2012user}
Joel~A Tropp.
\newblock User-friendly tail bounds for sums of random matrices.
\newblock {\em Foundations of computational mathematics}, 12(4):389--434, 2012.

\bibitem[Tsy08]{tsybakov2008introduction}
Alexandre~B Tsybakov.
\newblock {\em Introduction to nonparametric estimation}.
\newblock Springer Science \& Business Media, 2008.

\bibitem[Tuk60]{tukey1960survey}
John~W Tukey.
\newblock A survey of sampling from contaminated distributions.
\newblock {\em Contributions to probability and statistics}, pages 448--485,
  1960.

\bibitem[Tuk75]{tukey1975mathematics}
John~W Tukey.
\newblock Mathematics and the picturing of data.
\newblock In {\em Proceedings of the International Congress of Mathematicians,
  Vancouver, 1975}, volume~2, pages 523--531, 1975.

\bibitem[Vai89]{vaidya1989new}
Pravin~M Vaidya.
\newblock A new algorithm for minimizing convex functions over convex sets.
\newblock In {\em 30th Annual Symposium on Foundations of Computer Science},
  pages 338--343. IEEE Computer Society, 1989.

\bibitem[Ver18]{vershynin2018high}
Roman Vershynin.
\newblock {\em High-dimensional probability: An introduction with applications
  in data science}, volume~47.
\newblock Cambridge university press, 2018.

\bibitem[Vov01]{vovk2001competitive}
Volodya Vovk.
\newblock Competitive on-line statistics.
\newblock {\em International Statistical Review}, 69(2):213--248, 2001.

\bibitem[YCS14]{yi2014alternating}
Xinyang Yi, Constantine Caramanis, and Sujay Sanghavi.
\newblock Alternating minimization for mixed linear regression.
\newblock In {\em International Conference on Machine Learning}, pages
  613--621. PMLR, 2014.

\bibitem[YJY09]{yang2009online}
Liu Yang, Rong Jin, and Jieping Ye.
\newblock Online learning by ellipsoid method.
\newblock In {\em Proceedings of the 26th Annual International Conference on
  Machine Learning}, pages 1153--1160, 2009.

\bibitem[ZBFL18]{zhou2018new}
Wen-Xin Zhou, Koushiki Bose, Jianqing Fan, and Han Liu.
\newblock A new perspective on robust m-estimation: Finite sample theory and
  applications to dependence-adjusted multiple testing.
\newblock {\em Annals of statistics}, 46(5):1904, 2018.

\bibitem[Zin03]{zinkevich2003online}
Martin Zinkevich.
\newblock Online convex programming and generalized infinitesimal gradient
  ascent.
\newblock In {\em Proceedings of the 20th international conference on machine
  learning (icml-03)}, pages 928--936, 2003.

\bibitem[ZJS20]{zhu2020robust}
Banghua Zhu, Jiantao Jiao, and Jacob Steinhardt.
\newblock Robust estimation via generalized quasi-gradients.
\newblock {\em arXiv preprint arXiv:2005.14073}, 2020.

\bibitem[ZS19]{zimmert2019optimal}
Julian Zimmert and Yevgeny Seldin.
\newblock An optimal algorithm for stochastic and adversarial bandits.
\newblock In {\em The 22nd International Conference on Artificial Intelligence
  and Statistics}, pages 467--475. PMLR, 2019.

\end{thebibliography}
